\def\draft{0}

\documentclass[11pt,letterpaper]{article}
\pdfoutput=1
\usepackage[T1]{fontenc}
\usepackage{tcolorbox}

\AtBeginDocument{%
  \addtolength\abovedisplayskip{0.1\baselineskip}%
  \addtolength\belowdisplayskip{0.1\baselineskip}%
  \addtolength\abovedisplayshortskip{0\baselineskip}%
  \addtolength\belowdisplayshortskip{0\baselineskip}%
}

\usepackage{lmodern} \normalfont %
\usepackage{anyfontsize}
\DeclareFontShape{T1}{lmr}{bx}{sc} { <-> ssub * cmr/bx/sc }{}
\DeclareFontShape{T1}{lmr}{m}{scit}{ <-> ssub * cmr/m/sc }{}
\DeclareFontShape{T1}{lmr}{bx}{scit}{ <-> ssub * cmr/bx/sc }{}

\usepackage{etoolbox}
\usepackage{graphicx, color}

\usepackage{url}            %
\usepackage{booktabs}       %
\usepackage{amsfonts}       %
\usepackage{nicefrac}  
\usepackage{microtype}
\usepackage{graphicx}
\usepackage{booktabs} %
\usepackage{subcaption}
\usepackage{lipsum}
\usepackage{tkz-fct}
\pdfoutput=1
\usepackage{amsmath,amsthm,wrapfig}
	\usepackage[utf8]{inputenc} 
\usepackage{amsmath, amssymb,bm, cases, mathtools, thmtools}
\usepackage[normalem]{ulem}

\usepackage{verbatim}
\usepackage{graphicx}\graphicspath{{figures/}}
\usepackage{multicol}
\usepackage{tabularx}
\usepackage{mathrsfs} 
\usepackage{caption}
\usepackage{algorithm}
\usepackage{algorithmicx}
\usepackage[noend]{algpseudocode}
\usepackage{array,booktabs,arydshln,xcolor}

\usepackage[%
		minnames=1,maxnames=99,maxbibnames=99,minalphanames=1,maxalphanames=5,maxcitenames=3,
		style=alphabetic,
		sorting=anyt,
		sortcites=false, %
		doi=false,url=false,
		uniquename=init,
		giveninits=true,
		hyperref,natbib,
		backend=biber]{biblatex}
\renewbibmacro{in:}{%
	\ifentrytype{article}{}{\printtext{\bibstring{in}\intitlepunct}}}

\usepackage[T1]{fontenc}
\usepackage{dsfont}
\usepackage{hyperref}
\usepackage{listings} %

\usepackage{enumitem}
\usepackage{booktabs}       %
\usepackage{nicefrac}       %

\usepackage{microtype}

\usepackage{lineno}
\usepackage{bbm}

\usepackage{caption}
\usepackage{subcaption}

\usepackage{datetime}

\DeclareMathAlphabet\EuRoman{U}{eur}{m}{n}
\SetMathAlphabet\EuRoman{bold}{U}{eur}{b}{n}

\declaretheorem[style=plain,numberwithin=section,name=Theorem]{theorem}

\declaretheorem[style=plain,sibling=theorem,name=Lemma]{lemma}
\declaretheorem[style=plain,sibling=theorem,name=Proposition]{proposition}

\declaretheorem[style=definition,sibling=theorem,name=Definition]{definition}

\declaretheorem[style=remark,qed=$\triangleleft$,sibling=theorem,name=Remark]{remark}
\numberwithin{theorem}{section}

\usepackage{xparse}
\usepackage{xstring}
\usepackage{xspace}

\ifnum\draft=1
    \newcommand{\mynote}[2]{{\marginpar{\color{#1}\sf \tiny #2}}}
    \newcommand{\myinlinenote}[2]{{\color{#1} \footnotesize \sf [#2]}}
    \newcommand{\mytext}[2]{{\color{#1} #2}} 
    \newcommand{\mystrikeout}[2]{{\color{#1} \sout{{\color{black} #2}}}} %
\else
    \newcommand{\mynote}[2]{}
    \newcommand{\myinlinenote}[2]{}
    \newcommand{\mytext}[2]{{#2}}
    \newcommand{\mystrikeout}[2]{}
\fi

\newcommand{\astext}[1]{\mytext{orange}{#1}}

\newcommand{\mhtext}[1]{\mytext{blue}{#1}}

\newcommand{\Dist}{\mathcal D}

\newcommand\optparen[1]{\ifthenelse{\equal{#1}{}}{}{(#1)}}

\newcommand{\Naturals}{\mathbb{N}}

\newcommand{\Reals}{\mathbb{R}}

\DeclareMathOperator*{\newlim}{\mathrm{lim}\vphantom{\mathrm{infsup}}}
\DeclareMathOperator*{\newmin}{\mathrm{min}\vphantom{\mathrm{infsup}}}
\DeclareMathOperator*{\newmax}{\mathrm{max}\vphantom{\mathrm{infsup}}}

\DeclareMathOperator*{\newsup}{\mathrm{sup}\vphantom{\mathrm{infsup}}}
\renewcommand{\lim}{\newlim}
\renewcommand{\min}{\newmin}
\renewcommand{\max}{\newmax}

\renewcommand{\sup}{\newsup}

\newcommand{\ProbMeasures}[1]{\mathcal{M}_1(#1)}

\renewcommand{\Pr}{\mathrm{Pr}}
\def\EE{\mathbb{E}}

\newcommand{\norm}[1]{\left\Vert #1 \right\Vert}

\newcommand{\iid}{i.i.d.}

\newcommand*{\Scale}[2][4]{\scalebox{#1}{\ensuremath{#2}}}%

\newcommand{\TVname}{\mathrm{TV}}
\newcommand{\TV}[2]{\TVname\textcolor{black}{
    \left(
        \textcolor{blue!60!black}{#1} 
        \textcolor{black}{\,\Scale[1.1]{\Big\|}\,}
        \textcolor{orange!70!black}{#2}
        \right)}}

\newcommand{\TVinline}[2]{\TVname\textcolor{black}{
    \left(
        \textcolor{blue!60!black}{#1} 
        \textcolor{black}{\,{\big\|}\,}
        \textcolor{orange!70!black}{#2}
        \right)}}

\newcommand{\KLname}{\mathrm{KL}}

\newcommand{\KL}[2]{\KLname\left(\textcolor{blue!70!black}{#1} 
   \textcolor{black}{\,\Scale[1.1]{\Big\|}\,}
   \textcolor{orange!80!black}{#2}
    \right)}

\newcommand{\KLinline}[2]{\KLname\left(\textcolor{blue!70!black}{#1} 
   \textcolor{black}{\,{\big\|}\,}
   \textcolor{orange!80!black}{#2}
    \right)}

\newcommand{\Normal}{\mathcal N}
\newcommand{\trace}{\mathrm{tr}}

\newcommand{\Var}{\text{Var}}

\newcommand{\reals}{\mathbb{R}}

\newcommand{\e}{\mathrm{e}}
\newcommand{\id}[1]{\mathbb{I}_{#1}\xspace}

\newcommand{\Alg}{\mathcal{A}}

\newcommand{\unif}[1]{\text{Unif}(#1)}

\newcommand{\lcrx}[4][{-1}]{
	\IfEq{#1}{-1}{\left #2 {{{{#3}}}} \right #4}{
   	\IfEq{#1}{0}{#2 {{{{#3}}}} #4}{
	\IfEq{#1}{1}{\bigl #2 {{{{#3}}}} \bigr #4}{
	\IfEq{#1}{2}{\Bigl #2 {{{{#3}}}} \Bigr #4}{
	\IfEq{#1}{3}{\biggl #2 {{{{#3}}}} \biggr #4}{
	\IfEq{#1}{4}{\Biggl #2 {{{{#3}}}} \Biggr #4}{
    \GenericWarning{"4th argument to lcrx must be -1, 0, 1, 2, 3, or 4"}
    }}}}}}}
\newcommand{\inner}[3][{-1}]{\lcrx[#1] < {{#2},{#3}} >}

\newcommand\eqdist{~\mathrel{\overset{\smash{\makebox[0pt]{\mbox{\normalfont\tiny dist.}}}}{=}}~}
\newcommand{\indep}{\mathrel{\perp\mkern-9mu\perp}}

\newcommand{\indic}[1]{\mathds{1}\left[#1\right]}

\newcommand{\defeq}{\stackrel{\text{\rm def}}{=}}

\newcommand{\dataset}{\mathbf{X}}

\newcommand{\auxsample}{\bm{Y}}
\newcommand{\scorefun}{\psi}
\newcommand{\distfamilyum}{\mathcal{P}^{\text{id.cov}}}
\newcommand{\distfamilyuc}{\mathcal{P}}
\newcommand{\tester}{\psi_m}
\newcommand{\outputmodel}{\hat{\mu}}

\newcommand{\im}[1]{\mathrm{Im}\left(#1\right)}

\newcommand{\chisqdist}[2]{\mathrm{D}_{\chi^2}\left(\textcolor{blue!70!black}{#1} 
   \textcolor{black}{\,\Scale[1.1]{\Big\|}\,}
   \textcolor{orange!80!black}{#2}
    \right)}
\usetikzlibrary{patterns}
\usepackage{makecell}
\usepackage{arydshln}
\usepackage{thmtools}
\usepackage{thm-restate}
\usepackage[capitalize,noabbrev]{cleveref}
\usepackage{hyperref}
\hypersetup{
	linktocpage=true,
	colorlinks=true,				
	linkcolor=[rgb]{.7,0,0},				
	citecolor=magenta,				
	urlcolor=[rgb]{.7,0,.7},
}
\usepackage[tt=false]{libertine}
    \usepackage[libertine]{newtxmath}
    \usepackage[T1]{fontenc}

\title{The Sample Complexity of Membership Inference \\ and Privacy Auditing}
\author{
Mahdi Haghifam\thanks{Toyota Technological Institute at Chicago. This work was done while the author was with the Khoury College of Computer Sciences, Northeastern University, supported by a Khoury Distinguished Postdoctoral Fellowship.}
\and
Adam Smith\thanks{Department of Computer Science, Boston University. Supported in part by NSF award CNS-2232694 and an Apple faculty research award.}
\and
Jonathan Ullman\thanks{Khoury College of Computer Sciences, Northeastern University. Supported by NSF awards CNS-2232692 and CNS-2247484.}
}
\date{}

\usepackage[letterpaper, margin=1in]{geometry}

\renewcommand{\epsilon}{\varepsilon}

\usepackage[colorinlistoftodos]{todonotes}
\def\[#1\]{\begin{equation*}\begin{aligned}#1\end{aligned}\end{equation*}}
\def\*[#1\*]{\begin{align*}#1\end{align*}}

\let\originalleft\left
\let\originalright\right
\renewcommand{\left}{\mathopen{}\mathclose\bgroup\originalleft}
\renewcommand{\right}{\aftergroup\egroup\originalright}

\usepackage{caption}
\usepackage{cleveref}
\crefname{algorithm}{procedure}{procedures}
\Crefname{algorithm}{Procedure}{Procedures}

\renewcommand{\paragraph}[1]{\medskip\noindent{\bf #1}}

\begin{document}
\maketitle
\ifnum\draft=1 \pagenumbering{roman} \fi

\begin{abstract}
A membership-inference attack gets the output of a learning algorithm, and a target individual, and tries to determine whether this individual is a member of the training data or an independent sample from the same distribution.  A successful membership-inference attack typically requires the attacker to have some knowledge about the distribution that the training data was sampled from, and this knowledge is often captured through a set of independent reference samples from that distribution.

In this work we study how much information the attacker needs for 
membership inference by investigating the sample complexity---the minimum number of reference samples required---for a successful attack. We study this question in the fundamental setting of Gaussian mean estimation where the learning algorithm is given $n$ samples from a Gaussian distribution $\mathcal{N}(\mu,\Sigma)$ in $d$ dimensions, and tries to estimate $\hat\mu$ up to some error $\mathbb{E}[\|\hat \mu - \mu\|^2_{\Sigma}]\leq  \rho^2 d$.  Our result shows that for membership inference in this setting, $\Omega(n + n^2 \rho^2)$ samples can be necessary to carry out any attack that competes with a fully informed attacker.

Our result is the first to show that the attacker sometimes needs many
more samples than the training algorithm uses to train the model.  This result has significant implications for practice, as all attacks used in practice have a restricted form that uses
$O(n)$ samples and cannot benefit from $\omega(n)$ samples.  Thus, these attacks may be underestimating the possibility of membership inference, and better attacks may be possible 
when information about the distribution is easy to obtain.
\end{abstract}

\ifnum\draft=1
    \newpage
    \setcounter{tocdepth}{2}
        {\small \tableofcontents}
    \newpage 
    \pagenumbering{arabic}
\fi

\section{Introduction} \label{sec:intro}
The objective of machine learning is to discover statistical patterns in a population using training data sampled from that population.  Intuitively, an ideal learning algorithm would extract only information about the population itself and nothing particular to the specific training data, so it would ensure the privacy of the people who are represented by the training data.  However, real learning algorithms necessarily have some dependence on their training data, and there are many powerful privacy attacks that can use the output of a learning algorithm to make inferences about the training data, including reconstruction attacks \citep{dinur2003revealing}, data extraction attacks \citep{carlini2019secret, carlini2021extracting}, and \emph{membership-inference attacks} \citep{homer2008resolving, sankararaman2009genomic, dwork2015robust, shokri2016membership}, which are the subject of this work.

A privacy attack is most compelling if it can be carried out by a realistic attacker in a realistic environment.  
What makes the attacker and environment realistic is complex, and depends on many features, such as the type of populations, learning objectives, and training algorithms involved.  One of the most important features of an attack is \emph{the attacker's background knowledge about the population.}  
Our goal in this work is to understand how little an attacker needs to know 
about the population to mount a successful %
attack.

\paragraph{Membership-inference attacks.}
To make this question concrete, we study an important class of privacy attacks called \emph{membership-inference attacks (MIA)}.  Informally, a membership-inference attack gets the output of a learning algorithm, and a target individual that is either a sample from the training data itself or an independent sample from the same population as the training data, and tries to distinguish between the two cases.  A successful %
attack can be a serious violation of privacy on its own, since membership in the dataset itself can be highly sensitive (e.g.\ if the dataset represents a study on patients with a rare disease) or it can be a stepping stone to other kinds of privacy violations where the attacker learns some features of the targeted individual.

Membership-inference attacks have been influential in both the theory and practice of privacy for a long time.  They were first demonstrated in practice on genomic datasets by \citet{homer2008resolving}, whose attack was formally studied by \citet{sankararaman2009genomic}.  In parallel, membership-inference attacks were introduced by \citep{bun2014fingerprinting} as a way to prove tight lower bounds on the error of differentially private algorithms .  These two perspectives on membership-inference were unified by  \citet{dwork2015robust}, who showed how the lower bounds for differential privacy yield membership-inference attacks with very strong robustness properties.  
Building on the connection between MIA and lower bounds for differential privacy for the mean estimation task, a long line of work has leveraged MIAs to establish sharp privacy–utility trade-offs for various estimation and learning tasks  (e.g., \citep{bassily2014private,steinke2017tight,kamath2022new,cai2023score,narayanan2024better,PortellaH25}).  
Later on, the influential work of \citet{shokri2016membership} showed how to perform membership-inference on black-box neural networks. Membership-inference attacks were established by \citet{jagielski2020auditing} as a widely adopted method for empirically auditing the privacy of machine learning algorithms in practice 
(see also \citep{steinke2023privacy,MahloujifarMC24-fDP,KeinanSL2025-one-run} for recent developments.). Membership-inference attacks have also found important applications in studying adaptive data analysis \citep{hardt2014preventing,steinke2015interactive} and memorization \citep{attias2024information} in machine learning.

In both the theoretical and empirical research on membership-inference, the attacker needs fairly precise knowledge of the population in order to pull off the most successful attacks.
The primary goal of this work is to understand whether this precise knowledge is \emph{necessary} to implement successful attacks. %
We make this question precise using the language of \emph{sample complexity}---we assume that the attacker is given a set of auxiliary samples from the same population as the training data, and want to know:  
\emph{how many samples are necessary and sufficient to implement a successful membership-inference attack?} 

\paragraph{Known attacks use
$O(n)$ reference examples.}
All known membership-inference attacks can be implemented using sample access to the underlying population; furthermore, they use $O(n)$ such samples, where $n$ is the size of the training data. For example, empirical work on MIA (e.g.~\citep{shokri2016membership, yeom2018privacy, jagielski2020auditing, YeMMBS22, steinke2023privacy, MahloujifarMC24-fDP, ZarifzadehLS24}) generally considers attacks that use the reference examples to train a small set of \emph{shadow} or \emph{reference} models that are distributed identically to the real model. Each such model uses only $n$ reference examples; moreover they are often trained on overlapping samples, which limits the reference data to just $2n$ examples \citep{carlini2022membership}). These attacks' specific form uses the shadow models only to set the tests' threshold, and does not benefit from having $\omega(n)$ samples. 

A few works \cite{CohenG24-influenceMI,suri2024parameters} use a quadratic approximation to the training loss to design a membership test, instead of comparing to reference models. These works do not directly quantify the size of the reference sample they use because they consider a threat model in which the attacker is allowed to query the Hessian of the loss function on the training set. Another work uses quantile regression to learn a good test statistic using reference samples \citep{bertran2023scalable}.
This method can, in principle, be adapted to any available sample size by enlarging the predictor class (denoted $\mathcal{H}$ in their work); however, their experiments only consider simple models trained with roughly $n$ reference examples.

On the theoretical side, \citet{sankararaman2009genomic} analyzed a simple setting in which $O(n)$ reference samples are sufficient to carry out an optimal attack. \citep{dwork2015robust} gave a robust version of that attack that also uses $O(n)$ samples---in fact, when the trained model is highly noisy it can sometimes suffice to have even just one auxiliary sample.

Thus, the current theory and practice suggest that the attacker might need as much data as the real training algorithm, but not substantially more data than that. %
Until this work, we knew neither whether $n$ reference samples are generally necessary, nor whether they always suffice.
Our main result answers this question for a simple, natural setting by giving a lower bound on the sample complexity of membership-inference, showing that $\omega(n)$ samples, and even as many as $\Omega(n^2)$ samples, can be necessary for an optimal attack.  Thus, not only are the linear-sized reference data sets used by existing attacks necessary in natural settings, they may be insufficient.  %

\vspace{3pt}
\begin{tcolorbox}[colback=gray!10, colframe=black, boxrule=0.5pt]
    \centering
    {\bfseries Main Conceptual Result:} \\
    {There are simple, natural settings where any successful membership-inference attack requires $\omega(n)$ samples from the population---that is, many more samples than were used to train the model.}
\end{tcolorbox}
\vspace{3pt}

We formulate and prove our result in the simple-yet-fundamental case of mean estimation, which has long been an important testbed for theoretical work on membership inference~\citep{sankararaman2009genomic,bun2014fingerprinting,dwork2015robust} and the basis of our theoretical understanding for richer settings.  Although the result is theoretical, it has noteworthy consequences for practice.  The first is that the degree of privacy an algorithm ensures in practice, at least with respect to membership-inference attacks, depends heavily on the knowledge of the attacker, which is inherently difficult to reason about or measure.  The second is that the existing methodology for MIA is not necessarily taking full advantage of knowledge the attacker might have, and thus we should be careful about interpreting empirical privacy measurements that are based on these potentially suboptimal attacks.  We elaborate more on this perspective in \Cref{sec:implications}.

\subsection{Problem Setup} \label{sec:setup}
In this section, we introduce the setup of membership-inference attacks on mean-estimation algorithms that we use to prove our main result.  To reduce the notational burden, we will define our model of MIA in a very concrete setting of Gaussian mean estimation so that we can state our main result, but later on we will define a more general setup for studying the sample complexity of MIA.

\paragraph{Gaussian mean estimation.} We consider the fundamental setting of \emph{Gaussian mean estimation}.  Let the population be $\Dist = \Normal(\mu,\Sigma)$ where $\mu \in \Reals^d$ and $\Sigma \in \Reals^{d \times d}_{\succ 0}$ is a positive-definite covariance matrix. Let $S_n =(X_1,\dots,X_n)\sim \Dist^{\otimes n}$ be a training dataset of size $n$.  We consider an algorithm $\Alg$ that releases $\hat \mu = \Alg\left(S_n\right)$, and, for some parameter $\rho > 0$, we say the algorithm is \emph{$\rho$-accurate} if 
$$
\EE\left[\norm{\hat \mu - \mu}^2_{\Sigma}\right] \leq d\rho^2 + \frac{d}{n}.
$$
In this definition, we measure the error in \emph{Mahalanobis norm} $\norm{\hat \mu - \mu}_{\Sigma} = \|\Sigma^{-1/2}\left(\hat \mu - \mu\right)\|_2$.  We consider this the natural way to measure error for Gaussian mean estimation because (1) it tightly captures the error from sampling, and (2) there are many differentially private algorithms \citep{brown2021covariance,brown2023fast,kuditipudi2023pretty} that ensure privacy by adding noise whose covariance is roughly proportional to $\Sigma$.  Note that we define the error to account for both the standard sampling error, which is $d/n$, and any additional error whose magnitude is controlled by $\rho$. For instance, it can be achieved by setting $\hat \mu = \frac{1}{n}\sum_{i=1}^n X_i + \rho Z$ where $Z \sim \Normal(0,\Sigma)$. 

\paragraph{Sample-based MIA.} Next, we formalize the problem of membership-inference attack that will study, which we refer to as \emph{sample-based MIA.}  As in all MIA, the attacker is given the model $\hat\mu$ returned by the algorithm, and a target sample.  The target sample will either be:
\begin{enumerate}
    \item a random member $X_i$ in the dataset for $i$ chosen uniformly in $[n]$ (the ``IN'' case), or
    \item a non-member of the dataset $X_0$ that is drawn independently from $\Dist$, (the ``OUT'' case)
\end{enumerate}
 and the goal is to distinguish these two cases.  To do so , we will also give the attacker some partial knowledge of the distribution, which we represent through a set of $m$ additional, independent samples from $\Dist$.

In this setup, a MIA is a function $\tester: (\Reals^d)^m \times \Reals^d \times \Reals^d \to \{\text{IN},\text{OUT}\}$.  We say $\tester$ is an \emph{$m$-sample-based MIA for Gaussian mean estimators} if,
for every normal distribution $\Normal(\mu,\Sigma)$ and every $\rho$-accurate \emph{noise-addition} \footnote{An algorithm $\Alg:(\Reals^d)^n \to \ProbMeasures{\Reals^d}$ is a noise addition mechanism if for every dataset $S_n = (x_1,\dots,x_n)\in (\Reals^d)^n$, its output can be written as $\Alg(S_n)=\sum_{i=1}^{n}\frac{1}{n}x_i + \xi$ where $\xi$ is a random vector, independent of $S_n$, such that $\EE[\xi]=0$.} algorithm $\Alg$, we have the following: let $(X_0,X_1,\dots,X_n, Y_1,\dots,Y_m)\sim \Dist^{\otimes n+m+1}$ and $\hat \mu = \Alg(S_n)$ where $S_n = (X_1,\dots,X_n)$. Then, it satisfies
\begin{enumerate}
\item (Low FPR) $\Pr\left(\tester\left(\auxsample^m, \hat \mu, X_0\right) = \text{IN}\right) \leq 0.49$, and
\item (High TPR) for every $i \in [n], \Pr\left( \tester\left(\auxsample^m, \hat \mu, X_i\right) = \text{IN}\right) \geq 0.51$.
\end{enumerate}

    This definition only requires the attacker to perform slightly better than a random guess.  Since we are proving lower bounds, this only makes our results stronger. In particular, they apply to the tests with small FPR (say, 1\%) considered in most empirical work on MIA. We also only require success against mechanisms that add independent noise. \astext{Again, this only strengthens the lower bound, and simplifies the presentation.} 

Given this model, we can succinctly state the main technical question of this work:
\begin{tcolorbox}[colback=gray!10, colframe=black, boxrule=0.5pt]
\centering
{What is the minimum number of samples $m$, as a function of $n,d,\rho$,  that suffice for sample-based MIA for Gaussian mean estimation? }
\end{tcolorbox}

\subsection{Our Results} \label{sec:our-results}

Our main result is a strong lower bound on the sample complexity of MIA for Gaussian mean estimation, which answers the main technical question in most natural parameter regimes. Specifically, we show a large gap between the power of fully informed attackers who know the population, and those who only have a limited number of samples.

To put our results in context, we will establish an important baseline of a fully informed attacker (in our formalism, an attacker with $m = \infty$ auxiliary samples).  In this setting \citet{dwork2015robust} showed\footnote{Technically, \citet{dwork2015robust} considered mean estimation with discrete data in $\{0,1\}^d$, but since their methods apply straightforwardly to the case of Gaussian mean estimation with a fully informed attacker, we believe it is appropriate to attribute this statement to their work.}  that there is a successful MIA provided that the dimension of the data is $d \gtrsim d^\star(n,\rho) \triangleq  n + n^2\rho^2$.  Moreover, when $d \lesssim d^\star(n,\rho)$, there is no successful MIA, even if the attacker has full knowledge of the population.  Thus the dimension, and its relationship to the size of the training data, and the accuracy, is the key parameter that determines whether MIA is possible for a fully informed attacker.  Therefore, an important regime to understand is when $d \approx d^\star(n,\rho)$, which is the setting where MIA is possible for an optimal, fully informed attacker, and we want to understand how many samples we need to approximately match the performance of this attacker.

Our main result shows that any attacker that succeeds in this optimal regime where $d \approx d^\star(n,\rho)$, requires a large number of samples $m \gtrsim n + n^2 \rho^2$.  Note that the interesting case is when $\rho \gg n^{-1/2}$ so that $m \gg n$.  Moreover, the attacker still needs approximately this many samples, even if the dimension is arbitrarily large.  That is, we cannot trade samples for a higher dimension.  This result stands in stark contrast with the case that the attacker \emph{knows} the covariance of the population, but not the mean as we discuss in \cref{subsubsec:intro-known-cov}.

\begin{theorem}[Main Result, Informal Version of \cref{thm:lowerbound-cov}] \label{thm:informal-cov}
Fix $n,\rho,d$, and assume  $d \gtrsim d^\star(n,\rho) \triangleq  n + n^2\rho^2$, i.e., the dimension is large enough for %
    a fully-informed attacker to %
    succeed.
Then, every sample-based MIA for the Gaussian mean estimation problem (\Cref{sec:setup}) requires $m \gtrsim n^2 \rho^2 + n$ 
samples.  Moreover, the lower bound holds even when we restrict attention to the specific mean estimator $\hat \mu = \frac{1}{n}\sum_{i=1}^n X_i + \rho Z$ where $Z \sim \Normal(0,\Sigma)$.
\end{theorem}

In particular, in the regime where $\rho =\omega( n^{-1/2})$---that is, when the error of the mean estimator is larger than the sampling error---\Cref{thm:informal-cov} implies that a successful attack requires $m\gtrsim n^2\rho^2 =\omega(n)$. Roughly speaking, any error level $\rho$ less than 1 (which could be achieved by just looking at a single sample) is meaningful. The required size of the reference sample can therefore be as large as $\Omega(n^2)$.

We give an overview of the intuition for and techniques used to prove the main result in \Cref{sec:technical-overview}.

\paragraph{Tightness of \Cref{thm:informal-cov}.}  Our main result gives a complete characterization of the sample complexity of MIA for Gaussian mean estimation in the regime where $d \approx d^\star(n,\rho)$, that is, when the dimension is comparable to the threshold at which MIA becomes possible for an adversary with perfect knowledge of the distribution. In general, however, the straightforward attack, which estimates the covariance matrix to rescale the data and applies \cite{dwork2015robust}, uses $\Theta(d)$ reference samples. When $d$ is much larger than $d^{\star}(n,\rho)$, there is a gap between that bound and the lower bound of \Cref{thm:informal-cov}.

\paragraph{What Does \Cref{thm:informal-cov} Rule Out?} Although our definition of MIA requires the attacker to succeed against an arbitrary noise addition algorithm with a given level of accuracy, the lower bound in \Cref{thm:informal-cov} specifically applies to the algorithm that returns $\hat \mu = \frac{1}{n}\sum_{i=1}^n X_i + \rho Z$ where $Z \sim \Normal(0,\Sigma)$. 
This algorithm adds noise proportional to the population covariance $\Sigma$, which is unknown to the MIA attacker, so the result does not give a single algorithm that is hard to attack with limited samples.  Instead, it rules out the existence of sample-limited attacks by  constructing a \emph{hard distribution over pairs of populations and algorithms}.  As discussed above, these natural algorithms are still vulnerable to MIA if the attacker has access to more samples, so this structure does not trivialize the lower bound.  We believe it is an interesting open question to give a separation based on a single mechanism.

\subsubsection{Implications for Practical MIA and Auditing} \label{sec:implications}
We can now elaborate on the implications of our results for practical membership-inference attacks and privacy auditing.  In general, a membership-inference attack in our setting, with full knowledge of the distribution, can be written as follows: given a target point $X$ and an estimated mean $\hat\mu$, the test \emph{accepts} (predicts \text{IN}) if
\begin{equation} \label{eq:general-mi}
    \psi(X,\hat\mu,\Dist) \geq \tau(X,\Dist),
\end{equation}
where $\Dist$ is the population, $\psi$ is some real-valued \emph{test statistic} that assigns a measure of the likelihood that $X$ was in the training data for $\hat\mu$, and $\tau$ is a real-valued \emph{threshold} that we use to make a binary prediction of IN or OUT.  Given some $\psi$ and $\tau$, we can try to approximate their behavior using auxiliary samples $\auxsample^m \sim \Dist^{\otimes m}$ rather than the full distribution, and our results show that doing so may require a large number of samples.

The prevailing approach to designing sample-based membership-inference attacks is to use a \emph{restricted} form of test that \emph{accepts} (predicts \text{IN}) if 
\begin{equation} \label{eq:restricted-mi}
    \psi(X,\hat\mu) \geq \tau(X,\Dist) \, .
\end{equation}
The difference from \eqref{eq:general-mi} is that $\Dist$ no longer appears on the left-hand side.
That is, these tests compute a \emph{distribution-independent test statistic} and use knowledge of the distribution only to calibrate the appropriate threshold.  
In some settings, such attacks are known to be exactly optimal~\citep{sablayrolles2019white}. It is not difficult to show that if the threshold is calibrated to achieve a specific, constant FPR like 5\%, then we can find such a threshold by training $O(1)$ independent models (i.e.\ recomputing independent estimates of $\mu$ a constant number of times), which requires $O(n)$ samples; in practice, this is essentially how the auxiliary samples are used in calibrating the attack. 

In other words, for any test of the restricted form \eqref{eq:restricted-mi}, having $\Theta(n)$ samples from the distribution is as good as having the distribution itself.  However, since we can show that having $\omega(n)$ samples allows for successful attacks in cases where $O(n)$ samples are insufficient for a successful attack, we know that tests of the general form \eqref{eq:general-mi} can be dramatically more powerful. Therefore, it is possible that existing privacy audits based on membership-inference attacks are underestimating privacy risk, and we should be exploring attacks that exploit knowledge of the distribution in new ways to obtain better attacks. Attacks based on a quadratic approximation~\cite{CohenG24-influenceMI,suri2024parameters} to the loss function, localized near the output parameters, offer a promising approach, even though they do not necessarily perform well on the specific mean estimation problem we consider.\footnote{Specifically, the loss being minimized in mean estimation, namely the average squared  Euclidean distance from the parameter to the data points, has a data-independent Hessian. As a result, the test derived from the approaches in \mhtext{\cite{CohenG24-influenceMI,suri2024parameters}} would not perform well.
}

\subsubsection{MIA for Mean Estimation with Known Covariance.}\label{subsubsec:intro-known-cov}
We complement our results by giving a characterization of membership-inference attacks when the attacker \emph{knows} the covariance of the population, but not the mean.  That is, the population is now assumed to be $\Dist = \Normal(\mu, \Sigma)$ where $\mu$ is \emph{unknown to the attacker} and $\Sigma$ is \emph{known to the attacker}.  In this setting, the threshold where MIA becomes possible for a fully informed attacker remains $d^{\star}(n,\rho) = n + n^2 \rho^2$.  \citet{dwork2015robust} studied sample-based MIA for this setting, showing that the sample complexity can be very small, and in some cases even a single sample suffices.\footnote{Technically, \citet{dwork2015robust} considered mean estimation with discrete data in $\{0,1\}^d$, but since their methods apply straightforwardly to the case of Gaussian mean estimation with known covariance, we believe it is appropriate to attribute the statement to their work.  }  Specifically, they showed that if $d \gtrsim d^{\star}(n,\rho)$, then it suffices for the attacker to have $O(\min\{n, 1/\rho^2\})$ auxiliary samples.  We show that their attack has optimal sample complexity.

\begin{theorem}[Informal Version of \cref{thm:lb-mean-est}]\label{thm:informal-mean}
Fix $n,\rho,d$ and assume that $d \approx d^{\star}(n,\rho) = n + n^2\rho^2.$  Then every sample-based MIA for Gaussian mean estimation \emph{with known covariance} requires $m \gtrsim \min\{n, 1/\rho^2\}$ samples.  More generally, for $n,m,d,\rho$, there exists an $m$-sample-based MIA only if $d \gtrsim n + n^2 \rho^2 + n^2 / m$.  Moreover, the lower bound holds even when $\hat \mu = \frac{1}{n}\sum_{i=1}^n X_i + \rho Z$ where $Z \sim \Normal(0,\Sigma)$.
\end{theorem}

Contrasting this result with \Cref{thm:informal-cov}, we see that the driver of high sample complexity in MIA is estimating the covariance of the population.

\subsection{Technical Overview} \label{sec:technical-overview}

In this section, we give an overview of the technical ideas behind the proofs in the paper. Recall that the class of populations we consider in this paper is $
 \distfamilyuc_{d} = \left\{ \Normal\left(\mu,\Sigma\right) : \mu \in \Reals^d, \Sigma \in \Reals^{d \times d}, \text{$\Sigma$ is Positive Definite (PD)} \right\}$. 
 As mentioned in \cref{sec:our-results}, the mean estimator that witnesses the lower bound in \cref{thm:informal-cov,thm:informal-mean} simply computes the empirical mean and adds independent noise $\rho Z$ for $Z\sim \mathcal{N}(0,\Sigma)$, i.e., $\hat \mu = \frac{1}{n}\sum_{i=1}^n X_i + \rho Z$. 
 
 First, consider an informed-attacker with  knowledge of population. Then, the informed attacker can compute the optimal test based on Neyman-Pearson Lemma which is given by 
 \begin{equation} \label{eq:np-rule}
 \psi_{\text{NP}}\left(X,\hat\mu\right) = \begin{cases}
   \text{IN} & \text{if}~~~ (X-\mu)^\top \Sigma^{-1} (\hat\mu-\mu) \geq c \frac{d}{n}, \\
   \text{OUT} & \text{if}~~~ (X-\mu)^\top \Sigma^{-1} (\hat\mu-\mu) < c \frac{d}{n},
\end{cases}
 \end{equation}
where $c$ is a universal constant. One can show that this test will successfully distinguish the IN and OUT cases (with small constant TPR and FPR) as long as $d \gtrsim d^{\star}(n,\rho) = n+n^2\rho^2$.

\paragraph{Attempt~1: Learning covariance matrix.} Assume without loss of generality that $\mu=0$. An observation about \cref{eq:np-rule} is that the threshold is a population-independent quantity. The  first attempt to design a sample-based MIA is to approximate $\Sigma^{-1}$ using auxiliary samples.  As we show in the proof of \cref{thm:ub-cov}, given a matrix $H \in \Reals^{d \times d}$ such that $\norm{\Sigma^{1/2} H \Sigma^{1/2} - \id{d}}_{\text{op}}\leq \frac{1}{2}$, we can design a test with the score function $(X-\mu)^\top H (\hat\mu-\mu)$. However, learning such an $H$ requires $\Omega(d)$ samples \citealp[Thm.~4.7.1]{vershynin2018high}.

\paragraph{Attempt~2: Learning NP test statistics.}
A tempting idea is based on the observation that the test statistics, i.e., $(X-\mu)^\top \Sigma^{-1} (\hat\mu-\mu)$ is a single scalar. A natural idea is to \emph{directly} approximate the score function without learning the covariance matrix which is a $d \times d$ object (Notice that the attacker has access to both $X$ and $\hat \mu$ and we assumed $\mu=0$). In \citep{haghifam2025sample}, we show that this strategy also requires $\Omega(d)$ samples. More precisely,  consider the following estimation task: fix $y_1,y_2\in \Reals^d$. Then, every algorithm that provides an estimate $\hat m$ of $y_1^\top \Sigma^{-1} y_2$ such that $\frac{1}{2}y_1^\top \Sigma^{-1} y_2 \leq \hat m \leq \frac{3}{2}y_1^\top \Sigma^{-1} y_2$ (with high probability) with only sampling access to $\Normal(0,\Sigma)$ requires at least $\Omega(d)$ samples, and this number of samples also suffices to learn the covariance matrix $\Sigma$ in the operator norm.

The failure of these two natural attempts do not imply our lower bound in \cref{thm:informal-cov}. How can we generalize impossibility result to arbitrary sample-based MIAs?

\subsubsection{Our lower bound for arbitrary sample-based MIAs.} 
Consider \cref{eq:np-rule}. Intuitively, $(X-\mu)^\top \Sigma^{-1} (\hat\mu-\mu)$ measures the correlation between the data point and the estimated mean, within the geometry induced by $\Sigma$. The crucial component  in \cref{eq:np-rule} is $\Sigma^{-1}$. It re-weights the space to account for the data variance, ensuring the test measures meaningful correlation rather than misleading large correlation resulting from high variance directions. Intuitively, it suggests that a sample-based MIA may need to find any directions with high variance in the data to differentiate noise from the signal.  This is exactly the intuition behind our lower bound in \cref{thm:informal-cov}. We consider a specific  family of the Gaussian populations with spiked covariance matrix: 
\begin{align} \label{eq:tech-cov-structure}
&\mathcal{P}_{(d,k)-\text{spiked}} \\
&= \Big\{ \Normal(0,\Sigma): \Sigma = \id{d} + \sigma^2 UU^\top \nonumber\\
&~\text{where}~U \in \Reals^{d\times k}~\text{with orthonormal columns}\Big\} \nonumber
\end{align} 
where we assume that $\sigma^2 \gtrsim d$.  The intuition behind this structure is that when there are $k$ high-variance directions, an attacker with $m < k$ auxiliary samples cannot identify all the high-variance directions. Then, we show that without identifying all the high-variance directions an attacker cannot succeed.

To prove this result formally, we use \cref{lem:fuzzy-lecam}. The Gaussian distribution is characterized by its mean $\mu$ and covariance $\Sigma$. Therefore, we can define a family of hypothesis testing problems parameterized by $\Sigma$ as follows:
\begin{align*}
  H_{(\mu,\Sigma)} = \begin{cases}
  H_{\text{OUT},(\mu,\Sigma)}&: (\auxsample^m,\hat \mu,X_0) \sim P_{\text{OUT},(\mu,\Sigma)},\\
    H_{\text{IN},(\mu,\Sigma)}&: (\auxsample^m,\hat \mu,X_1) \sim P_{\text{IN},(\mu,\Sigma)}.
  \end{cases}
\end{align*}
Here, $P_{\text{OUT},(\mu,\Sigma)}$ and $P_{\text{IN},(\mu,\Sigma)}$ are the joint distribution of $(\auxsample^m,\hat \mu,X_0)$ and $(\auxsample^m,\hat \mu,X_1)$, respectively (see \cref{def:problem-def-uc} for more formal description.). Our goal is to prove a lower bound to capture an attacker that doesn't know $(\mu,\Sigma)$. To do so, we use the following result: for every distribution over $(\mu,\Sigma)$ denoted by $\pi$ and attacker's strategy $\tester: (\Reals^d)^m \times \Reals^d \times \Reals^d \to \{\text{IN},\text{OUT}\}$, we have 
\[
&\sup_{(\mu,\Sigma)} \Big\{P_{\text{OUT},(\mu,\Sigma)}\left(\tester\left(\auxsample^m,\outputmodel,X_0\right)=\text{IN}\right)\\
&+ P_{\text{IN},(\mu,\Sigma)}\left( \tester\left(\auxsample^m,\outputmodel,X_1\right)=\text{OUT}\right) \Big\}\\
&\geq 1 -  \TV{P_{\text{OUT},\pi}}{P_{\text{IN},\pi}},
\]
where $P_{\text{IN},\pi}$ (and similarly $P_{\text{OUT},\pi}$) is a mixture distribution where the sampling procedure is based on first sampling $(\mu,\Sigma)\sim \pi$, and then sampling $(\auxsample^m,\hat \mu,X_1)$ from $P_{\text{IN},(\mu,\Sigma)}$. This implies that if we can find a distribution $\pi$ and show that 
$$
    \TV{P_{\text{OUT},\pi}}{P_{\text{IN},\pi}} \leq 0.02,
$$ 
we can show that for every $m$-sample-based MIA, there exists a population such that the attacker can't satisfy the minimal accuracy conditions. 

For the {\em \bfseries unknown covariance setting} we consider  the distribution $\pi$ to be uniform distribution over all the subspaces of dimension $k$ in $\Reals^d$ (see \cref{eq:tech-cov-structure}).  Then, we need to characterize the total variation (TV) distance between two mixture populations. Characterizing TV (or KL) between mixture populations is challenging and there is no general approach for it. To do so, our proof relies on two key technical steps: The first step is a reduction that shows given $k >m$, we can relate the performance of $m$-sample-based MIA to a sample-based attacker with \emph{zero auxiliary sample in $\Reals^{d-2m}$ with $k-m$ high-variance directions}. This step significantly simplifies the proof. More formally, let $( \auxsample^m,X_0,X_1,\outputmodel)=\textsc{Sample}\left(n,d,k,m,\sigma\right)$ and $(\widetilde{X_0},\widetilde{ X_1},\widetilde{\outputmodel})=\textsc{Sample}\left(n,d-2m,k-m,0,\sigma\right)$ where $\textsc{Sample}$ is the sampling procedure from the mixture distribution (see \cref{def:gen-process-cov} for a formal description.). More precisely, we have
\begin{align}
&\TV{ \auxsample^m,X_0,\outputmodel}{ \auxsample^m ,X_1,\outputmodel} \nonumber\\
&\leq c\sqrt{\frac{m}{n+n^2\rho^2}} + \TV{\widetilde{X_0},\widetilde{\outputmodel}}{\widetilde{ X_1},\widetilde{\outputmodel}}. \label{eq:intro-reduct-zero}
\end{align}
This step implies that as long as $m \lesssim n + n^2 \rho^2$, the first term is small. The second technical step concerns with the zero auxiliary sample case, i.e., the second term in \cref{eq:intro-reduct-zero}. In particular, we show a structural result which implies that \emph{the sufficient statistic} for attackers with zero auxiliary samples has a very simple form. We show that  the sufficient statistics are given by norm of the data point, norm of the output, and the inner product between the data point and the output vector. More precisely, 
\begin{align}
    &\TV{\widetilde{X_0},\widetilde{\outputmodel}}{\widetilde{ X_1},\widetilde{\outputmodel}}=\nonumber\\
    & \TV{\norm{\widetilde{X_0}},\norm{\widetilde{\outputmodel}},\inner{\widetilde{X_0}}{\widetilde{\outputmodel}}}{\norm{\widetilde{ X_1}},\norm{\widetilde{\outputmodel}},\inner{\widetilde{X_1}}{\widetilde{\outputmodel}}}.  \label{eq:intro-stat}
\end{align}
To bound the RHS, we use the chain rule for TV. This step consists of various technical steps. For example, we show a bound on the  expectation of the ratio of Gaussian quadratic forms, i.e., 
$$
    \EE_{Z \sim \Normal(0,\id{d})}\left[\frac{(Z^T A Z)^2}{Z^T A^2 Z}\right]
$$
for a symmetric matrix $A$ as well as the total variation distance between perturbed chi-squared random variables, i.e., the TV distance between $X_1 + Y$ and $X_2 + Y$ where $Y$ follows a chi-squared distribution and $X_1,X_2$ are random variables that are independent of $Y$. 

For the {\em \bfseries known covariance setting}, we define the mixture probability $\pi$ over $\mu$ as $\Normal(0,\id{d})$.  The key insight is that conditioning on auxiliary samples yields tractable posterior distribution, allowing us to compute the total variation distance explicitly. The key step is that under the two hypotheses, the distribution of auxiliary samples are the same. Thus, by the chain rule for KL divergence,
\[
&\KL{\auxsample^m,X_1,\outputmodel}{\auxsample^m,X_0,\outputmodel}\\
&= \EE\left[\KL{X _1,\outputmodel \big|  \auxsample^m }{X _0,\outputmodel \big|  \auxsample^m }\right].
\]
Conditioning on auxiliary samples $\auxsample^m = \mathbf{y}^m$ only affects the distribution of $\mu$. In particular, we show that $X _1,\outputmodel \vert  \auxsample^m$ and $X _0,\outputmodel \vert  \auxsample^m$ both follow Gaussian distribution and this observation lets us to obtain a closed form expression for the KL distance.

\section{Preliminaries}
\subsection{Notation}
For $n \in \Naturals$, we write $[n]=\{1,\dots,n\}$. For a given matrix $A \in \Reals^{n \times m}$, $\im{A}=\{y: \exists x \in \Reals^m ~\text{s.t.}~Ax=y\}$. Also for a set of vectors $(x_1,\dots,x_m)\in (\Reals^d)^m$, we write $\text{span}\left(\{x_1,\dots,x_m\}\right)=\{y\in \Reals^d: \exists (\alpha_1,\dots,\alpha_m)\in \Reals^m, y = \sum_{i=1}^m \alpha_i x_i\}$. For a given matrix $A$, we denote by $\norm{A}_{\text{op}}=\max_{x\in \Reals^m, x \neq 0}\left\{\frac{\norm{Ax}}{\norm{x}}\right\}$ where $\norm{\cdot}$ denotes the standard $\ell_2$ norm, and $\norm{A}_F^2 = \trace\left(A^\top A\right)$ where $\trace(A)$ denotes the trace of matrix $A$.

Let $\mathcal Z$ be a measurable space (equipped with an implicit $\sigma$-algebra), we write $\mathcal M_1(\mathcal Z)$ to denote the set of probability measures over $\mathcal Z$.  For two random variables $X$ and $Y$ defined on the same probability space, we write $X \eqdist Y$ to mean that $X$ and $Y$ are equal in distribution.

For probability measures $P$ and $Q$ defined on the same measurable space, we denote their Total Variation (TV) distance and Kullback-Leibler (KL) divergence as follows:
\begin{itemize}
    \item The TV distance is defined as $\TVinline{P}{Q} = \sup_{A} |P(A) - Q(A)|$, where the supremum is taken over all measurable sets $A$.
    \item The KL divergence is defined as $\KLinline{P}{Q} = \int \log\left(\frac{dP}{dQ}(x)\right) dP(x)$, which is finite only if $P$ is absolutely continuous with respect to $Q$.
\end{itemize}
We also use the following notations: for $X\sim P$ and $Y \sim Q$, we use $\TVinline{X}{Y} = \TVinline{P}{Q}$ and $\KLinline{X}{Y} = \KLinline{P}{Q}$. For a pair of jointly distributed random variables $(X_1,Y_1)$ and $(X_2,Y_2)$, we use the following (non-standard) notation for the conditional TV and KL distance
$$
\EE\left[\TV{Y_1\Big| {X_1} }{Y_2\Big| {X_2}}\right] = \int_{x} \TV{Y_1\Big| {X_1=x}  }{Y_2\Big| {X_2=x}} \text{d}P_{X_1}(x),
$$
$$
\hfill \EE\left[\KL{Y_1\Big| X_1 }{Y_2\Big| X_2}\right] = \int_{x} \KL{Y_1\Big| X_1=x }{Y_2\Big| X_2=x} \text{d}P_{X_1}(x),
$$
where $P_{X_1}$ denotes the marginal distribution of $X_1$.  Notice that the expectation is over the distribution of $X_1$.  We denote by $\chi^2(k)$ the chi-square distribution with $k$ degrees of freedom, i.e., the distribution of the sum of squares of $k$ independent mean-zero and unit-variance Gaussian random variables. 

Next we formally define a family of distributions that is important in this paper.

\begin{definition}\label{def:dist-families}
Fix $d \in \Naturals$. We define 
\[
 \distfamilyuc_{d} = \left\{ \Normal\left(\mu,\Sigma\right) : \mu \in \Reals^d, \Sigma \in \Reals^{d \times d}, \text{$\Sigma$ is \textbf{Positive Definite}} \right\}.
\]
We also define $\distfamilyum_{d} \subset \distfamilyuc_{d}$ as $\distfamilyum_{d} =  \left\{ \Normal\left(\mu,\id{d}\right) : \mu \in \Reals^d \right\}$. In particular, all the distributions in $\distfamilyum_{d}$ have the identity covariance matrix.
\end{definition}

\subsection{Minimal Definition of Membership Inference}

In this section, we present a definition for the membership inference attack. Our definition captures a realistic attack scenario where the attacker has limited knowledge: rather than having access to the true data distribution, the attacker possesses only $m$ auxiliary samples drawn from the same distribution \iid. The goal is to achieve membership inference performance that is slightly better than random guessing.  First, we start by formally defining the measure of the mean estimation error.

\begin{definition}[$\rho$-accurate mean estimator for distribution $\Dist$]
\label{def:accurate-mean-estimator}
Fix $n,d \in \Naturals$. We say $\Alg_n:(\Reals^d)^n \to \ProbMeasures{\Reals^d}$ is a $\rho$-accurate mean estimator for the distribution $\Dist=\Normal(\mu,\Sigma)$, if $\Alg_n$  satisfies
$$\EE_{S_n\sim \Dist^{\otimes n},\hat{\mu}\sim \Alg_n(S_n)}\left[\norm{\Sigma^{-\frac{1}{2}}(\hat \mu - \mu)}_2^2\right]\leq \left(\frac{1}{n}+\rho^2\right)d.$$
\end{definition}

\begin{definition}[Noise-adding mechanism] \label{def:noise-addition-mech}
Fix $n,d \in \Naturals$. We say $\Alg_n:(\Reals^d)^n \to \ProbMeasures{\Reals^d}$ is a \emph{noise addition mechanism} if for every data distribution $\Dist$, its output can be written as follows: let $(X_1,\dots,X_n)\sim \Dist^{\otimes n}$ $\Alg_n(X_1,\dots,X_n)=\frac{1}{n}\sum X_i + \xi$ where $\xi$ is a mean-zero random variable, i.e., \text{noise}, which is independent of $X_1,\dots,X_n$. 
\end{definition}

\begin{definition}[$m$-sample-based MIA] \label{def:problem-def-uc}
Fix $n, m, d \in \mathbb{N}$. Fix a pair $(\Alg_n,\Dist)$ such that $\Dist$ is a data distribution and $\Alg_n$ is a  mean estimator. We say function $\psi_m : (\mathbb{R}^d)^m \times \mathbb{R}^d \times \mathbb{R}^d \to \{\mathrm{IN}, \mathrm{OUT}\}$ is a \emph{$m$-sample-based MIA} for the pair $(\Alg_n,\Dist)$ if, when we generate inputs as follows,

\begin{itemize}
    \item Sample $(X_0, X_1, \ldots, X_n, Y_1, \ldots, Y_m) \sim \mathcal{D}^{\otimes m+n+1}$.
    \item Let $\auxsample^m = (Y_1, \ldots, Y_m)$ denote the auxiliary samples.
    \item Let $\hat{\mu} = \mathcal{A}_n(X_1, \ldots, X_n)$ be the algorithm's output.
\end{itemize}
Then, $\psi_m$ satisfies%
 $$
 \Pr(\psi_m(\auxsample^m, \hat{\mu}, X_0) = \mathrm{OUT}) \geq 0.51, \quad
   \Pr(\psi_m(\auxsample^m, \hat{\mu}, X_1) = \mathrm{IN}) \geq 0.51.
$$
\end{definition}

\begin{remark}
Even though \cref{def:problem-def-uc} considers sample-based MIA for a fixed pair $(\Alg_n,\Dist)$, our main results in later sections concern with the sample-based MIAs that work against many pairs of $(\Alg_n,\Dist)$.
\end{remark}
\begin{remark}
We make a note in passing that even though we present the definition for the mean estimation algorithms, the above definition can be generalized to any learning tasks.
\end{remark}

\newcommand{\algnoisy}{\Alg_{n,\rho}}

\subsection{Preliminaries on Establishing Lower Bound for Hypothesis Testing}\label{lem:fuzzy-lecam}
This paper aims to establish information-theoretic bounds on the number of auxiliary samples required for minimal sample-based MIA (\cref{def:problem-def-uc}). Our main tool is the following lemma.

\begin{lemma} \label{lem:fuzzy-lecam-actual}
Let $V$ and $\mathcal{Z}$ be measurable spaces. Consider two arbitrary sets of probability measures $\{P_{0,\nu}\}_{\nu \in V}$ and $\{P_{1,\nu}\}_{\nu \in V}$ over $\mathcal{Z}$. Consider an arbitrary distribution $\pi$ over $V$. Also, define the mixture probability measures of $P_{0,\pi} \triangleq \EE_{\nu \sim \pi}\left[ P_{0,\nu}\right]$ and $P_{1,\pi} \triangleq \EE_{\nu \sim \pi}\left[ P_{1,\nu}\right]$. Then, consider the family of hypothesis testing problems $\{H_\nu = (H_{0,\nu},H_{1,\nu})\}_{\nu \in V}$ where for every $\nu \in V$, 
\begin{align*}
  H_\nu = \begin{cases}
  H_{0,\nu}&: Z \sim P_{0,\nu},\\
    H_{1,\nu}&: Z \sim P_{1,\nu}.
  \end{cases}
\end{align*}
Then, for every $\pi$ and decision rule (test) $\Psi: \mathcal{Z} \to \{0,1\}$, we have 
\[
\sup_{\nu \in V} \left\{P_{0,\nu}\left(\Psi^{-1}(1)\right) + P_{1,\nu}\left(\Psi^{-1}(0)\right)\right\} \geq 1 -  \TV{P_{0,\pi}}{P_{1,\pi}}.
\]
\end{lemma}
This lemma provides a systematic approach for proving lower bounds. The key steps are:

\begin{enumerate}
\item Specify a family of hypothesis testing problems $\{H_\nu = (H_{0,\nu},H_{1,\nu})\}_{\nu \in V}$ where $\nu$ is a parameter unknown to the tester. In our application of Gaussian mean estimation, the unknown parameters are $(\Dist,\Alg_n)$ where $\Alg_n$ is a $\rho$-accurate mean estimator for $\Dist \in \distfamilyuc_d$ (see \cref{def:accurate-mean-estimator,def:problem-def-uc}.).
\item Specify a prior distribution $\pi$ over the parameters $(\Dist,\Alg_n)$.
\item Provide an upper bound on $\TVinline{P_{0,\pi}}{P_{1,\pi}}$ where $P_{0,\pi}$ and $P_{1,\pi}$ are mixture distributions (See \cref{lem:fuzzy-lecam-actual}.).
\end{enumerate}
The bound then follows from \cref{lem:fuzzy-lecam-actual}, showing that any test's sum of probability errors under the two hypotheses is at least $1 -  \TVinline{P_{0,\pi}}{P_{1,\pi}}$ for some value of $\nu$. In particular, as per \cref{def:problem-def-uc}, if $ \TVinline{P_{0,\pi}}{P_{1,\pi}} \leq 0.02$, then, no sample-based MIA can satisfy the accuracy requirement in \cref{def:problem-def-uc}.

\section{Sample-based MIA with Unknown Covariance}\label{sec:unknown-cov}

In this section, we present our results on  characterization of  the minimum number of samples required for any sample-based MIA to satisfy the minimal accuracy conditions in \cref{def:problem-def-uc}, in the case that both the mean and covariance of the data distribution are unknown to the attacker. Later in \cref{sec:unknown-mean}, we show how the minimum number of auxiliary samples changes if the covariance of the data distribution is known to the attacker.

\paragraph{Lower bound: minimum number of samples required by any attacker.} 
First, we present a lower bound on the minimum number of samples required by any attacker in the unknown covariance case. %
Recall the definitions of accurate mean estimator from \cref{def:accurate-mean-estimator}, $m$-sample-based MIA in \cref{def:problem-def-uc}, the definition of noise addition estimators in \cref{def:noise-addition-mech}, and the class of Gaussian distribution from \cref{def:dist-families}.

\begin{theorem} \label{thm:lowerbound-cov}
Let $c_1$ and $c_2$ be universal constants. 
Let $d,n,m \in \Naturals , \rho \in (0,1)$ be arbitrary constants that satisfy, $d>c_1(n+n^2\rho^2)$, and $ m \leq c_2\left(n + n^2 \rho^2\right)$. Then, there is \underline{no} single  $m$-sample-based MIA for all pairs of $(\Dist,\Alg_n)$ such that $\Dist \in \distfamilyuc_d$, $\Alg_n$ is a noise-adding mechanism which is a $\rho$-accurate mean estimator for $\Dist$.
\end{theorem}

The implication of this result is as follows: every sample-based MIA requires a number of auxiliary samples which is $\Omega\left(n + n^2 \rho^2\right)$. This holds given certain conditions on the dimension.  Note that when $d =  \Omega(n + n^2 \rho^2)$, an informed attacker who knows the data distribution satisfies the accuracy condition (see \cref{lem:full-info-kl}). Also, notice that when the additive error is larger than the sampling error, i.e., $\rho \geq 1/\sqrt{n}$, the number of auxiliary samples required is $\Omega\left(n^2 \rho^2\right)$, i.e., it is quadratic in the number of training samples \mhtext{when $\rho = \Theta(1)$}.

\begin{restatable}{theorem}{ubcov} \label{thm:ub-cov}
Let $c>1$ be a universal constant. Fix $d,n,m \in \Naturals$ and $0<\rho<1$. Then, given $m \geq c d$ and $d \geq c\left( n + n^2 \rho^2\right)$, there exists an $m$-sample-based MIA for all pairs of $(\Dist,\Alg_n)$ such that $\Dist \in \distfamilyuc_d$, $\Alg_n$ is a noise-adding mechanism and a $\rho$-accurate mean estimator for $\Dist$. 
\end{restatable}

The strategy of the attacker is quite simple and inspired by \citep{dwork2015robust}: The attacker divides its auxiliary samples into three parts of size $\Theta(d)$, $\Theta(\min\{n,1/\rho^2\})$, and one. Using $\Theta(d)$ samples it constructs the empirical covariance matrix. Let us denote the inverse of the empirical covariance matrix by $H$. Then,  it uses the next $\Theta(\min\{n,1/\rho^2\})$ samples to get an estimate of the unknown mean. Let us denote it by $\bar \mu_{0}$. Also, denote the last held-out point by $Y_0$. Then, the strategy of the attacker is 
\[
\tester\left(\auxsample^m, \hat \mu, X\right) = 
\begin{cases}
 \text{IN} & \text{if } \inner{H^{1/2}\left(\hat \mu - \bar \mu_0\right)}{H^{1/2}\left(X-Y_0\right)} \geq c\frac{d}{n} \\
\text{OUT} & \text{if } \inner{H^{1/2}\left(\hat \mu - \bar \mu_0\right)}{H^{1/2}\left(X-Y_0\right)} < c\frac{d}{n}
\end{cases}
\]
where $c$ is a universal constant.
\begin{remark}
To compare the lower bound in \cref{thm:lowerbound-cov} and the upper bound in \cref{thm:ub-cov}, consider the case that $d \asymp d^\star(n,\rho) = n + n^2 \rho^2$. Note that $d \asymp d^\star(n,\rho)$ is the threshold under which even an attacker with the full knowledge of the distribution cannot succeed \citep{dwork2015robust}. In this regime, the lower bound in \cref{thm:lowerbound-cov} shows that any sample-based MIA requires at least $m \geq d^\star(n,\rho)$. Also, the upper bound in \cref{thm:ub-cov} shows there exists an attacker which requires $O(d^\star(n,\rho))$ samples. Therefore, in the regime that $d \asymp d^\star(n,\rho)$, our result is tight.
\end{remark}

\subsection{Proof Sketch of \cref{thm:lowerbound-cov}} \label{sec:sketch-cov}
\begin{proof}[Proof Sketch of \cref{thm:lowerbound-cov}]
In this part, we provide an informal overview of the proof, highlighting the key steps.

\paragraph{Setup.}

Since we want to prove a lower bound, we assume that the mean of the unknown data distribution is zero. We consider the following family of  data distributions where the covariance structure is as follows:
Let $\sigma>0$ be a constant and $k,d \in \Naturals$ such that $k\leq d$. Then, define 
\[
\mathcal{P}_{(d,k)-\text{spiked}} = \Big\{ \Normal(0,\Sigma): &\Sigma = \id{d} + \sigma^2 UU^\top~\text{where}~U \in \Reals^{d\times k}~\text{with orthonormal columns} \\
&~\text{\text{and} $UU^\top$ is the projection matrix onto a rank-$k$ subspace} \Big\}.
\]
Intuitively, each distribution in $\mathcal{P}_{(d,k)-\text{spiked}}$ has a large variance on a subspace of dimension $k$ in $\Reals^d$. 

Based on the recipe in \cref{lem:fuzzy-lecam}, 
we need to specify a prior distribution $\pi$ over the parameters $(\mathcal{D}, \mathcal{A}_n)$, where $\mathcal{D} \in \mathcal{P}_{(d,k)\text{-spiked}}$ and $\mathcal{A}_n$ is a $\rho$-accurate estimator for $\mathcal{D}$ (per \cref{def:accurate-mean-estimator}). We construct this prior $\pi$ as follows:
\begin{enumerate}
    \item We first define a prior over the distribution $\mathcal{D}$. We let $\pi_{k,d}$ be the uniform distribution over all $k$-dimensional subspaces, and we sample $UU^T \sim \pi_{k,d}$. This defines the distribution $\mathcal{D} = \Normal(0, \Sigma)$ where $\Sigma = \id{d} + \sigma^2 UU^T$.
    \item For each $\mathcal{D}$ (i.e., each $\Sigma$) sampled this way, we \textit{deterministically couple} it with a single, specific algorithm $\mathcal{A}_n$ from our class:  This algorithm adds noise $\xi = \rho Z$ where $Z \sim \Normal(0, \Sigma)$, thus matching the data's covariance structure.
\end{enumerate}
The sampling process is formally presented in \cref{def:gen-process-cov}.

\begin{algorithm} 
\caption{Data Generation with Unknown Covariance: $\textsc{Sample}\left(n,d,k,m,\sigma\right)$}
\begin{algorithmic}[1]
\Require number of samples:~$n$, total dimension:~$d$,  number of high variance dimensions:~$k$,  the number of auxiliary samples:~$m$, the variance parameter:~$\sigma$.
\State Let $\pi_{k,d}$ be uniform distribution over the subspaces of dimension $k$ in $\Reals^d$
\State Sample $UU^\top$ from $\pi_{k,d}$.
\State Let $\Sigma = \mathbb{I}_d + \sigma^2 UU^\top$.
\State Sample $(X_0, X_1, \ldots, X_n, Z,Y_1, Y_2, \ldots, Y_m) \sim \mathcal{N}(0, \Sigma)^{\otimes(n+m+2)}$.
\State Let $\auxsample^m = (Y_1, Y_2, \ldots, Y_m)$.
\State Let $\outputmodel = \frac{1}{n} \sum_{i=1}^n X_i + \rho Z$.
\State Let $\outputmodel_0 = \outputmodel_1 = \outputmodel$.
\State \textbf{Output} $( \auxsample^m, X_0, \outputmodel_0,X_1, \outputmodel_1)$.
\end{algorithmic}
\label{def:gen-process-cov}
\end{algorithm}

\paragraph{Reduction to zero auxiliary sample.} The first step is a reduction which shows that the performance of an attacker with access to $m$ auxiliary samples can be related to the performance of another sample-based MIA with \emph{zero auxiliary samples} for a \emph{different problem}. Here, we give an overview of the reduction: Assume that the attacker has access to the auxiliary samples $\auxsample^m = (Y_1,\dots,Y_m)$ from an unknown data distribution. Using this $m$ samples, we define a sequence of $m$ random variables $\auxsample^m_{\Vert} =(UU^\top Y_1,\dots, UU^\top Y_m)$. Notice that $UU^\top Y_i$ is the projection of $Y_i$ onto the unknown subspace with a  large variance. By Gaussianity and $m\leq k$, $\auxsample^m_{\Vert} = (UU^\top Y_1,\dots, UU^\top Y_m)$ is $m$ \emph{independent} random vectors lie in the subspace spanned by $UU^\top$ with probability one. 

Assume that we consider an even more powerful attacker who has access to $\left(\auxsample^m,\auxsample^m_{\Vert}\right)$. Intuitively, using this information the attacker can identify $m$ directions in the high-variance subspace. Also, by computing $\{Y_i - UU^\top Y_i: i \in [m]\}$, the attacker can identify $m$ directions in $(UU^\top)^{\perp}$, i.e., $m$ directions that are perpendicular to the high-variance subspace. To gain more intuition, see \cref{fig:reducation} which depicts the orthonormalization of $\Reals^d$ from the attacker's point of view.  Therefore, the attacker knows a subspace of dimension $m$ from $UU^\top$ and a subspace of dimension $m$ from $(UU^\top)^{\perp}$. Interestingly, we show that for this more powerful attacker, the problem of membership inference reduces to the membership inference attack in $\Reals^{d-2m}$ with $k-m$ noisy directions when the attacker has access to \emph{zero} auxiliary samples. Intuitively, see \cref{fig:reducation}. The proof of this reduction is based on showing that the attacker has zero information about the unknown subspaces, and in fact, the noisy $k-m$ directions is distributed uniformly in $d-2m$ available dimensions.   

More precisely, we show the following: let $( \auxsample^m,X_0,X_1,\outputmodel_0,\outputmodel_1)=\textsc{Sample}\left(n,d,k,m,\sigma\right)$ and $(\widetilde{X_0},\widetilde{ X_1},\widetilde{\outputmodel_0},\widetilde{\outputmodel_1})=\textsc{Sample}\left(n,d-2m,k-m,0,\sigma\right)$. Then,

\[
\TV{ \auxsample^m,X_0,\outputmodel_0}{ \auxsample^m ,X_1,\outputmodel_1} \leq c\sqrt{\frac{m}{n+n^2\rho^2}} + \TV{\widetilde{X_0},\widetilde{\outputmodel_0}}{\widetilde{ X_1},\widetilde{\outputmodel_1}}.
\]
This inequality shows that given that $m\lesssim n+n^2\rho^2$, the first term can be smaller than an arbitrary constant.

\begin{figure}[h]
    \centering
    \begin{tikzpicture}[scale=2, xscale = 1.3, yscale = 0.4]
    \draw[thick] (0,0) rectangle (4,4);

    \draw[thick] (1,0) -- (1,4);
    \draw[thick] (2,0) -- (2,4);
    \draw[thick] (3,0) -- (3,4);

    \fill[pattern=horizontal lines] (0,0) rectangle (1,4);      %
    \fill[pattern=crosshatch] (1,0) rectangle (2,4);           %
    \fill[pattern=horizontal lines] (2,0) rectangle (3,4);      %
    \fill[pattern=crosshatch] (3,0) rectangle (4,4);           %

    \node[align=center, fill=white, inner sep=3pt, rounded corners=2pt] at (0.5,2) {$m$\\known\\directions};
    \node[align=center, fill=white, inner sep=3pt, rounded corners=2pt] at (1.5,2) {$k-m$\\unknown\\directions};
    \node[align=center, fill=white, inner sep=3pt, rounded corners=2pt] at (2.5,2) {$m$\\known\\directions};
    \node[align=center, fill=white, inner sep=3pt, rounded corners=2pt] at (3.5,2) {$d-k-m$\\unknown\\directions};

    \draw[thick] (0,-0.2) -- (0,-0.4) -- (2,-0.4) -- (2,-0.2);
    \node[below, align=center] at (1,-0.5) {noisy subspace\\with variance $\sigma^2$};
    
    \draw[thick] (2,-0.2) -- (2,-0.4) -- (4,-0.4) -- (4,-0.2);
    \node[below] at (3,-0.5) {unit variance subspace};

    \node[left] at (-0.3,2) {$\mathbb{R}^d =$};

    \draw[<->, thick] (0,4.3) -- (4,4.3);
    \node[above] at (2,4.3) {$d$};
    
    \draw[<->, thick] (4.3,0) -- (4.3,4);
    \node[right] at (4.3,2) {$d$};
\end{tikzpicture}

\clearpage
    \caption{Orthonormalization of $\Reals^d$ from the point of view of the attacker with $m$ auxiliary samples $(Y_1,\dots,Y_m)$ and $m$ additional auxiliary samples $(UU^\top Y_1,\dots,UU^\top Y_m)$. The matrix on the RHS is $d\times d$ and columns show an orthonormal basis for $\Reals^d$ }
    \label{fig:reducation}
\end{figure}
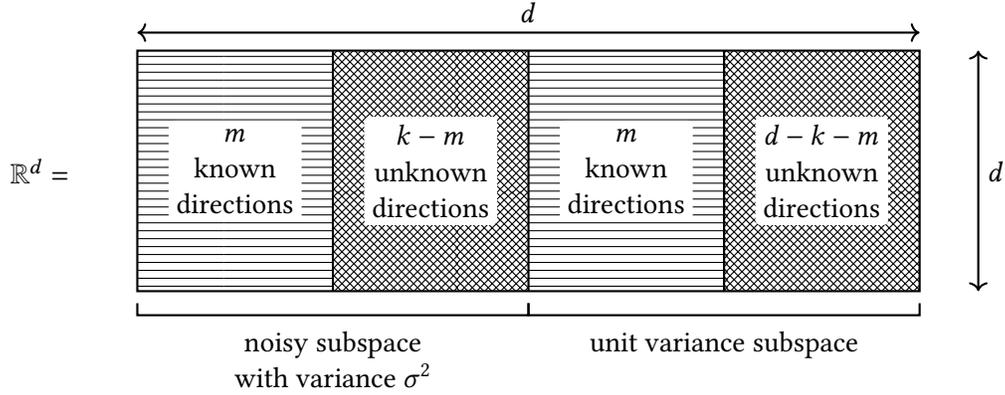

\paragraph{Sufficient statistic for zero auxiliary sample MIA.} The last step is to show that with zero auxiliary samples, it is impossible to perform MIA, i.e., $\TVinline{\widetilde{X_0},\widetilde{\outputmodel_0}}{\widetilde{ X_1},\widetilde{\outputmodel_1}}$ is smaller than a constant. In particular, we need to provide an upper bound on the total variation distance between two mixture distributions. Comparing the total variation distance between two mixture distributions is a challenging task and there are no general techniques for it. 

The crucial technical step in our proof is  characterizing  the \emph{sufficient statistics} for zero auxiliary samples case. In zero auxiliary samples case, we show that the sufficient statistics are given by norm of the data point, norm of the output, and the inner product between the data point and the output vector. More precisely, 
\[
\TV{\widetilde{X_0},\widetilde{\outputmodel_0}}{\widetilde{ X_1},\widetilde{\outputmodel_1}} = \TV{\norm{\widetilde{X_0}},\norm{\widetilde{\outputmodel_0}},\inner{\widetilde{X_0}}{\widetilde{\outputmodel_0}}}{\norm{\widetilde{ X_1}},\norm{\widetilde{\outputmodel_1}},\inner{\widetilde{X_1}}{\widetilde{\outputmodel_1}}}.
\]

\paragraph{Statistical indistinguishability with zero auxiliary samples.}
Characterization of the sufficient statistics  significantly simplifies the proof. Intuitively, it is difficult to distinguish between two cases mainly because the norm and the inner product operators \emph{mix} the components from the noisy and unit-variance directions, therefore, it makes the TV term small. 

To be more precise, let $(\widetilde{X_0},\widetilde{ X_1},\widetilde{\outputmodel_0},\widetilde{\outputmodel_1})=\textsc{Sample}\left(n,d-2m,k-m,0,\sigma\right)$ also recall the definition of the random variable $UU^\top$ and its dependence to other random variables in \cref{def:gen-process-cov}. We use the chain rule for TV distance from \cref{lem:chain-rule-tv} to write
\[
&\TV{\norm{\widetilde{X_0} },\norm{\widetilde{\hat \mu_0} },\inner{\widetilde{X_0} }{\widetilde{\hat \mu_0} }}{\widetilde{\norm{X_1 }},\widetilde{\norm{\hat \mu_1 }},\inner{\widetilde{X_1} }{\widetilde{\hat \mu_1} }}\\
&\leq \EE\left[\TV{\inner{\widetilde{X_0}}{\widetilde{\hat \mu_0} }\Big| {\widetilde{X_0} ,UU^\top}}{\inner{\widetilde{X_1} }{\widetilde{\hat \mu_1}}\Big| {\widetilde{X_1} ,UU^\top}}\right]\\
& + \EE\left[\TV{ \norm{\widetilde{\outputmodel_0} } \Big| {\widetilde{X_0} ,UU^\top,\inner{\widetilde{X_0} }{\widetilde{\hat \mu_0} }}}{\norm{\outputmodel_1 }\Big| {\widetilde{X_1} ,UU^\top,\inner{\widetilde{X_1} }{\widetilde{\hat \mu_1} }}}\right].
\]
The intuition behind the upper bound is we make the attacker stronger by giving the vector $x$ and the noisy subspace, i.e., $UU^\top$ to the attacker. 

To gain an intuition why the TV distance is small, consider the inner product term. By definition $\EE\left[\inner{\widetilde{X_0}}{\widetilde{\outputmodel_0}}\right]=0$ and $\EE\left[\inner{\widetilde{X_1}}{\widetilde{\outputmodel_1}}\right] =  \frac{\EE\left[\norm{X_1}^2\right]}{n} = \frac{\trace(\Sigma)}{n} = \frac{\left(d - k - m + (k-m)\sigma^2\right)}{n}$. Therefore, the mean of the inner product terms under two hypotheses is different. However, as we show, the variance is large and it can \emph{hide} the difference between the means. We can show $\Var\left(\inner{\widetilde{X_0}}{\widetilde{\outputmodel_0}}\right) \approx \Var\left(\inner{\widetilde{X_1}}{\widetilde{\outputmodel_1}}\right) \approx \left(\frac{1}{n}+\rho^2\right) \cdot \norm{\Sigma}_F^2 = \left(\frac{1}{n}+\rho^2\right) \cdot \left(d-k-m + (k-m)\sigma^4\right) $ (see \cref{lem:mean-var-overview} for the proof). Assume that $(k-m)\sigma^2 \geq d$. We can make this assumption since $(k,d,\sigma)$ are parameters of our choice in the construction. Then, consider  the ratio of the mean over the the standard deviation:
 \[
\frac{\EE\left[\inner{\outputmodel}{X_1}\right]}{\sqrt{\Var\left(\inner{\outputmodel}{X_0}\right)}} &= \frac{d - k - m + (k-m)\sigma^2}{\sqrt{n + n^2 \rho^2}\left(\sqrt{d-k-m + (k-m)\sigma^4}\right)}\\
&\leq  \frac{d  + (k-m)\sigma^2}{\sqrt{n + n^2 \rho^2}\left(\sqrt{(k-m)\sigma^4}\right)} \\
&\stackrel{(a)}{\lesssim }\frac{(k-m)\sigma^2}{\sqrt{n + n^2 \rho^2}\left(\sqrt{(k-m)\sigma^4}\right)}\\
& \leq \frac{\sqrt{k-m}}{\sqrt{n + n^2 \rho^2}},
 \]
 where $(a)$ follows since $(k-m)\sigma^2 \geq d$. Assume we set the parameters such that $k-m$ is a universal constant. Then, for a sufficiently large $n$, this ratio can be arbitrarily small. This shows that the variance in the inner product hides the difference in their mean. 

 For the term related to the norms under two hypotheses, the analysis is more challenging. In particular, we need to provide an upper bound on 
\[
\TV{ \norm{\outputmodel_0 } \Big| X_0 =x,UU^\top=UU^\top,\inner{X_0 }{\hat \mu_0 }=r}{\norm{\outputmodel_1 }\Big| X_1 =x,UU^\top=UU^\top,\inner{X_1 }{\hat \mu_1 }=r}.
\]
For $i \in \{0,1\}$, $\norm{\outputmodel_i}^2=\Pi_{x}\left(\outputmodel_i\right) + \Pi_{x^\perp}\left(\outputmodel_i\right)=\inner{\frac{x}{\norm{x}}}{\outputmodel_i} + \Pi_{x^\perp}\left(\outputmodel_i\right)$, where $\Pi_{x}$ and $\Pi_{x^\perp}$ are the orthogonal projection operator on the subspace spanned by $x$ and the orthogonal complement of it, respectively. Notice that $\inner{\frac{x}{\norm{x}}}{\outputmodel_i} = r \frac{x}{\norm{x}}$. Therefore, by the invariance of TV to the one-to-one mappings,

\[
\TV{ \norm{\Pi_{x^\perp}\left(\outputmodel_0\right)} \Big| X_0 =x,UU^\top=UU^\top,\inner{X_0 }{\hat \mu_0 }=r}{\norm{\Pi_{x^\perp}\left(\outputmodel_1\right)} \Big| X_1 =x,UU^\top=UU^\top,\inner{X_1 }{\hat \mu_1 }=r}.
\]

Let $v = \Pi_{x^\perp}\left(\frac{UU^\top x}{\|UU^\top x\|}\right)$ denote the component of the normalized projection orthogonal to $x$. Since $UU^\top$ has rank $k$, we can choose an orthonormal basis $\{u_1, \ldots, u_k\}$ for $\text{Im}(UU^\top)$ such that:
\begin{itemize}
\item $u_1 = \frac{UU^\top x}{\|UU^\top x\|}$ (the normalized projection of $x$ onto the subspace)
\item $u_i \perp x$ for $i \in \{2, \ldots, k\}$ (the remaining basis vectors are orthogonal to $x$)
\end{itemize}
Also, define $Z_0 \sim \Normal(0,\id{d}),(\xi_{1,0},\xi_{1,1},\dots,\xi_k)\sim \Normal(0,1)^{k+1}$ such that $(Z_0,\xi_{1,0},\xi_{1,1},\dots,\xi_k)$ are mutually independent. Also, $\alpha_0 \triangleq \sqrt{\frac{1}{n}+\rho^2}$ and $\alpha_1 \triangleq \sqrt{\frac{n-1}{n^2}+\rho^2}$. Then, 
\[
\outputmodel _0 &\eqdist \alpha_0\left(Z_0 + \sigma\left(\frac{UU^\top x}{\norm{UU^\top x}}\xi_{1,0}+ \sum_{i=2}^k u_i \xi_i\right)\right),\\
\outputmodel _1 &\eqdist \frac{x}{n}+\alpha_1 \left(Z_0 + \sigma\left(\frac{UU^\top x}{\norm{UU^\top x}}\xi_{1,1}+ \sum_{i=2}^k u_i \xi_i\right)\right).
\]
Recall the specific basis vectors we pick for $\im{UU^\top}$ and its properties. Then, we decompose $\outputmodel _0$ and $\outputmodel _1$ on the subspace spanned by $x$ and the subspace orthogonal to $x$ as follows:
\[
\begin{cases}
\Pi_{x^\perp} \outputmodel _0 &= \alpha_0\left(\Pi_{x^\perp}\left(Z_0\right) + \sigma\left(\Pi_{x^\perp}\left(\frac{UU^\top x}{\norm{UU^\top x}}\right)\xi_{1,0}+ \sum_{i=2}^k u_i \xi_i\right)\right),\\
\Pi_{x^\perp} \outputmodel _1  &=  \alpha_1 \left(\Pi_{x^\perp}\left(Z_0\right) + \sigma\left(\Pi_{x^\perp}\left(\frac{UU^\top x}{\norm{UU^\top x}}\right)\xi_{1,1}+ \sum_{i=2}^k u_i \xi_i\right)\right).
\end{cases}
\]
\[
\begin{cases}
  \inner{x}{\outputmodel _0} &= \alpha_0\left( \inner{x}{Z_0} + \sigma \norm{UU^\top x}\xi_{1,0} \right), \\
  \inner{x}{\outputmodel _1} &= \frac{\norm{x}^2}{n} + \alpha_1\left( \inner{x}{Z_0} + \sigma \norm{UU^\top x}\xi_{1,1} \right).
\end{cases}
\]
Observe that conditioning on the inner product induces different conditioned distribution on $\xi_{1,0}$ and $\xi_{1,1}$. The main technical challenge here to show that $\sigma\sum_{i=2}^k u_i \xi_i$ can hide the difference between $\xi_{1,0}$ and $\xi_{1,1}$. To show that, we characterize the total variation distance between two random variables perturbed by chi-squared distribution: Let $(X_1,X_2,Y)$ be jointly distributed random variables such that 1)  $(X_1,X_2) \indep Y$, 2) $Y$ is distributed as $\chi^2(k)$. Then,
\[
\TV{X_1 +   Y}{ X_2 + Y} &\leq \sqrt{ \frac{\left(\EE[|X_1 - X_2|]\right)^2}{4k-16}}.
\]
Using this technical result, we show that by having large enough noisy directions, i.e., $k-m$, TV distance between norms can be arbitrarily small.
\end{proof}

\subsection{Key Steps in the Formal Proof of \cref{thm:lowerbound-cov}}

In this section, we give a detailed proof of \cref{thm:lowerbound-cov} by presenting the important steps. All the proofs are deferred to \cref{appx:pf-unknowncov}.

\subsubsection{Construction.} \label{sec:construct-unknown-cov}
 Let $\sigma>0$ be a constant and $k,d \in \Naturals$ such that $k\leq d$. We consider the following family of covariance matrices given by
\[
\mathcal{P}_{(d,k)-\text{spiked}} = \Big\{ \Normal(0,\Sigma): &\Sigma = \id{d} + \sigma^2 UU^\top~\text{where}~U \in \Reals^{d\times k}~\text{with orthonormal columns} \\
&~\text{\text{and} $UU^\top$ is the projection matrix onto a rank-$k$ subspace} \Big\}.
\]
Based on the recipe in \cref{lem:fuzzy-lecam}, 
we need to specify a prior distribution $\pi$ over the parameters $(\mathcal{D}, \mathcal{A}_n)$, where $\mathcal{D} \in \mathcal{P}_{(d,k)\text{-spiked}}$ and $\mathcal{A}_n$ is a $\rho$-accurate estimator for $\mathcal{D}$ (per \cref{def:accurate-mean-estimator}). We construct this prior $\pi$ as follows:
\begin{enumerate}
    \item We first define a prior over the distribution $\mathcal{D}$. We let $\pi_{k,d}$ be the uniform distribution over all $k$-dimensional subspaces, and we sample $UU^T \sim \pi_{k,d}$. This defines the distribution $\mathcal{D} = \Normal(0, \Sigma)$ where $\Sigma = \id{d} + \sigma^2 UU^T$.
    \item For each $\mathcal{D}$ (i.e., each $\Sigma$) sampled this way, we \textit{deterministically couple} it with a single, specific algorithm $\mathcal{A}_n$ from our class: This algorithm adds noise $\xi = \rho Z$ where $Z \sim \Normal(0, \Sigma)$, thus matching the data's covariance structure.
\end{enumerate}The sampling process is formally presented in \cref{def:gen-process-cov}.

\subsubsection{Reduction to the Zero Auxiliary Sample Case}
The first step is to establish a reduction showing that if membership inference with zero auxiliary samples is difficult, it remains difficult with $m$ auxiliary samples, provided $m$ is smaller than $n + n^2 \rho^2$. The next lemma formalizes this results. The proof quite technical and involved; for an intuition see \cref{sec:sketch-cov}.

\begin{restatable}{lemma}{redcuctionzeroaux}\label{lem:reduction-to-zero}
Let \textsc{Sample} be given by \cref{def:gen-process-cov}. Fix $d,n,m,k,\sigma$ such that $d>2m$, $k>m$, and $d\geq k+m$. Let $( \auxsample^m ,X _0,X _1,\outputmodel _0,\outputmodel _1)=\textsc{Sample}\left(n,d,k,m,\sigma\right)$ and $(\widetilde{X _0},\widetilde{ X _1},\widetilde{\outputmodel _0},\widetilde{\outputmodel _1})=\textsc{Sample}\left(n,d-2m,k-m,0,\sigma\right)$. Then, we have
\[
\TV{ \auxsample^m ,X _0,\outputmodel _0}{ \auxsample^m  ,X _1,\outputmodel _1} \leq 2\sqrt{\frac{m}{n+n^2\rho^2}} + \TV{\widetilde{X _0},\widetilde{\outputmodel _0}}{\widetilde{ X _1},\widetilde{\outputmodel _1}}.
\]
\end{restatable}

\subsubsection{Zero Auxiliary Sample Case}

After reducing the problem to the case that the number of auxiliary samples is zero, we discuss in this part how we prove that MIA is difficult with zero auxiliary samples. We first start with an important observation:
\begin{restatable}{lemma}{rotationinvar}\label{lem:rotation-invar}
Let $ \left(X_0 , \outputmodel_0 ,X_1 , \outputmodel_1 \right)=\textsc{Sample}\left(n,d,k,0,\sigma\right)$ be as defined in \cref{def:gen-process-cov}. Then, for $j  \in \{0,1\}$ and every rotation matrix $R \in \Reals^{d \times d}$, we have $\left(X _j,\hat \mu _j\right)\eqdist \left(RX _j,R\hat \mu _j\right)$.
\end{restatable}
This observation lets us  prove a structural result on the form of probability density function (PDF) of  $(X_j , \outputmodel_j)$ for $j \in \{0,1\}$. In particular, we show in \cref{lem:pdf-depend-norm} that for every pair of random variables that satisfies the property in \cref{lem:rotation-invar}, the joint PDF can only depend on the norms and their inner product. This results leads to the following result on the  sufficient statistics for the problem of MIA with zero auxiliary samples.

\begin{restatable}{lemma}{suffstat}\label{lem:tv-norm-inner}
Let $ \left(X_0 , \outputmodel_0 ,X_1 , \outputmodel_1 \right)=\textsc{Sample}\left(n,d,k,0,\sigma\right)$ be as defined in \cref{def:gen-process-cov}.  Then, we have
\[
\TV{X_0 ,\hat \mu_0 }{X_1 ,\hat \mu_1 } = \TV{\norm{X_0 },\norm{\hat \mu_0 },\inner{X_0 }{\hat \mu_0 }}{\norm{X_1 },\norm{\hat \mu_1 },\inner{X_1 }{\hat \mu_1 }}.
\]
\end{restatable}

To provide an upper bound on the RHS of \cref{lem:tv-norm-inner}, we use the following result which is an immediate consequence of the data processing inequality and the chain rule fo TV distance.

\begin{restatable}{lemma}{tvdecomposezero} \label{lem:decompose-inner-norm}
Let $ \left(X_0 , \outputmodel_0 ,X_1 , \outputmodel_1 \right)=\textsc{Sample}\left(n,d,k,0,\sigma\right)$ from \cref{def:gen-process-cov}. Also, recall the definition of $UU^\top$ from \cref{def:gen-process-cov} and its dependence on other random variables. Then,  
\[
&\TV{\norm{X_0 },\norm{\hat \mu_0 },\inner{X_0 }{\hat \mu_0 }}{\norm{X_1 },\norm{\hat \mu_1 },\inner{X_1 }{\hat \mu_1 }}\\
&\leq \EE\left[\TV{\inner{X_0}{\hat \mu_0 }\Big| {X_0 ,UU^\top}}{\inner{X_1 }{\hat \mu_1 }\Big| {X_1 ,UU^\top}}\right]\\
& + \EE\left[\TV{ \norm{\outputmodel_0 } \Big| {X_0 ,UU^\top,\inner{X_0 }{\hat \mu_0 }}}{\norm{\outputmodel_1 }\Big| {X_1 ,UU^\top,\inner{X_1 }{\hat \mu_1 }}}\right].
\]
\end{restatable}

\cref{lem:decompose-inner-norm} breaks the problem of upper bounding the total variation distance into two parts. The first step is to prove an upper bound on the inner product term.
\begin{restatable}{lemma}{zeroinnerprod} \label{lem:tv-zero-inner-term}
Let $ \left(X_0 , \outputmodel_0 ,X_1 , \outputmodel_1 \right)=\textsc{Sample}\left(n,d,k,0,\sigma\right)$ be as defined in \cref{def:gen-process-cov}. Also, recall the definition of $UU^\top$ from \cref{def:gen-process-cov} and its dependence to other random variables. Then,
\[
\EE\left[\TV{\inner{X_0 }{\hat \mu_0 }\Big| {X_0 ,UU^\top}}{\inner{X_1 }{\hat \mu_1 }\Big| {X_1 ,UU^\top}}\right] \leq \sqrt{\frac{4}{n-1 + n^2 \rho^2}\cdot\left((k+3) + \frac{(d-k+3)^2}{(1+\sigma^2)^2 (k-2)}\right)}.
\]
\end{restatable}

Finally, for the term related to the norm of the ouput under the two hypotheses, we have the following technical lemma.

\begin{restatable}{lemma}{zernormterm} \label{lem:tv-norm}
Let $n\geq 2$. Let $ \left(X_0 , \outputmodel_0 ,X_1 , \outputmodel_1\right)=\textsc{Sample}\left(n,d,k,0,\sigma\right)$ be as defined in \cref{def:gen-process-cov}. Also, recall the definition of $UU^\top$ from \cref{def:gen-process-cov} and its dependence to other random variables. For every $c>1$ such that  $n\geq 2$, $32k \leq c^2 \left(n-1+n^2\rho^2\right)$, $\sigma^2 k \geq d$, and $d >k >\max\{8c^2,9\}$, we have
\[
&\EE\left[\TV{ \norm{\outputmodel_0 } \Big|{X_0 ,UU^\top,\inner{X_0 }{\hat \mu_0 }}}{\norm{\outputmodel_1 }\Big| {X_1 ,UU^\top,\inner{X_1 }{\hat \mu_1 }}}\right] \\
&\leq \sqrt{\frac{12 c^2}{n-1+n^2\rho^2} + \frac{k}{8(n-1+n^2\rho^2)^2} + \frac{1024 c^4}{k-9} }  + \frac{3}{c^2}  \\
&+\sqrt{\frac{1}{(n-1+n^2\rho^2)^2}\frac{\sqrt{2d+d^2}}{4(1+\sigma^2)} + \frac{k}{8(n-1+n^2\rho^2)^2} + \frac{1}{(n-1+n^2\rho^2)^2}\frac{2d+d^2}{(2k-18)(1+\sigma^2)^2}}.
\]
\end{restatable}

\subsubsection{Putting the Pieces Together}
After describing the key lemmas towards proving the main theorem, now we discuss the parameter settings such that the claimed lower bound holds. Let $d^\star(n,\rho)\triangleq n+n^2\rho^2-1$ In particular, from \cref{lem:reduction-to-zero,lem:decompose-inner-norm,lem:tv-zero-inner-term,lem:tv-norm}, we have
\[
&\TV{ \auxsample^m ,X _0,\outputmodel _0}{ \auxsample^m  ,X _1,\outputmodel _1} \\
&\leq 2\sqrt{\frac{m}{d^\star(n,\rho)+1}} + \TV{\widetilde{X _0},\widetilde{\outputmodel _0}}{\widetilde{ X _1},\widetilde{\outputmodel _1}} \\
& \leq 2\sqrt{\frac{m}{d^\star(n,\rho)+1}} + \sqrt{\frac{4}{d^\star(n,\rho)-1}\cdot\left((k-m+3) + \frac{(d-k-m+3)^2}{(1+\sigma^2)^2 (k-m+2)}\right)}\\
& +\sqrt{\frac{12 c^2}{d^\star(n,\rho)} + \frac{k-m}{4(d^\star(n,\rho))^2} + \frac{525 c^4}{k-m-4} } + \frac{3}{c^2}  \\
&+ \frac{1}{d^\star(n,\rho)} \sqrt{\frac{\sqrt{2(d-2m)+(d-2m)^2}}{4(1+\sigma^2)} + \frac{k-m}{8} +\frac{2(d-2m)+(d-2m)^2}{(2(k-m)-18)(1+\sigma^2)^2}},
\]
where for \cref{lem:tv-zero-inner-term,lem:tv-norm}, we adjust the values of $d$ and $k$ since $(\widetilde{X _0},\widetilde{ X _1},\widetilde{\outputmodel _0},\widetilde{\outputmodel _1})=\textsc{Sample}\left(n,d-2m,k-m,0,\sigma\right)$. Also, recall that the last inequality holds for every $c>1$ such that  $n\geq 2$, $8(k-m) \leq c^2 \left(n-1+n^2\rho^2\right)$, $\sigma^2 (k-m) \geq d-2m$, and $d-2m >k-m >4c^2$.

We need to show that there are many parameter settings \footnote{Note that we do not make any attempts to optimize the constants. Obtaining the sharpest constants will be deferred to future work.} that leads to $\TVinline{ \auxsample^m ,X _0,\outputmodel _0}{ \auxsample^m  ,X _1,\outputmodel _1} \leq 0.02$ in order to prove that there is no MIA strategy for \cref{def:problem-def-uc}. Here, we want to show that there exists a concrete parameter setting that gives us the desired result. Let $N_{\text{dir}}=k-m$. Then,
\begin{enumerate}
\item Set $c=25$. This requires $N_{\text{dir}}\geq 2500$ and ensures that $3/c^2=0.0048$.
\item Set the number of remaining directions to $N_{\text{dir}}=10^{14}$.
\item Set $\sigma^2=d$
\item Choose $n$ large enough such that $d^\star(n,\rho)\geq 10^{19}$.
\item Set $m \leq 6 \times 10^{-6} \cdot d^\star(n,\rho)$.
\end{enumerate}
with these parameters, one can see that total variation distance becomes smaller than $0.02$.

\section{Sample-based MIA with Known Covariance} \label{sec:unknown-mean}
In this section, we consider an easier scenario where the covariance matrix of the data distribution is known to the attacker, but only the mean vector is unknown. The following result provides a lower bound on the number of auxiliary samples required by any sample-based MIA. Recall the definitions given in  \cref{def:problem-def-uc,def:dist-families}.

\begin{restatable}{theorem}{lbunknownmean}\label{thm:lb-mean-est}
Let $c$ be a universal constant. Let $d,n,m \in \mathbb{N}$ and $\rho >0$ be arbitrary constants that satisfy $n\geq 4$, $m\geq 2$, and $d < c \left(n+n^2 \rho^2 + \frac{n^2}{m}\right)$. Then, there is \underline{no} single  $m$-sample-based MIA for all pairs of $(\Dist,\Alg_n)$ such that $\Dist \in \distfamilyum_d$, $\Alg_n$ is a noise addition mechanism which is a $\rho$-accurate mean estimator for $\Dist$.
\end{restatable}

To understand the implications of this result, consider the parameter regime where $d \asymp n + n^2 \rho^2$. In this range, an informed attacker with full knowledge of the data distribution can satisfy the accuracy conditions in \cref{def:problem-def-uc} (see \cref{lem:full-info-kl}). In contrast, a sample-based MIA requires at least $\Omega\left(\min\{\frac{1}{\rho^2},n\}\right)$ auxiliary samples to achieve similar performance.  Furthermore, this result is tight: \cite{dwork2015robust} showed that when $d \gtrsim n + n^2 \rho^2 + \frac{n^2}{m}$, there exists a strategy that satisfies the requirements in \cref{def:problem-def-uc} for the distribution family $\distfamilyum_d$.
\begin{remark}
Knowledge of the covariance matrix significantly reduces the sample complexity of sample-based MIA. For $\rho > 1/\sqrt{n}$, comparing \cref{thm:informal-cov,thm:informal-mean} shows that knowing the covariance reduces the required auxiliary samples from $\Omega(n^2\rho^2)$ to $O(n)$.
\end{remark}

\begin{proof}[Proof Sketch]
For the lower bound, we follow the framework in \cref{lem:fuzzy-lecam}. We consider the distribution family $\distfamilyum_d=\{\mathcal{N}(\mu,\id{d}):\mu \in \mathbb{R}^d\}$ with prior $\pi = \mathcal{N}(0,I_d)$ over $\mu$. The key insight is that conditioning on auxiliary samples yields tractable posterior distributions, allowing us to compute the total variation distance explicitly.

\textbf{Setup.} Let $(\mu,Z_0, Z_1)\sim \mathcal{N}(0,I_d)^{\otimes 3}$ be independent, and define $\alpha = \sqrt{\frac{1}{n}+\rho^2}$ and $\beta = \sqrt{\frac{n-1}{n^2}+\rho^2}$. The data point and algorithm output under the two hypotheses are:
\begin{align*}
\text{OUT: } &X_0 \eqdist \mu + Z_0, \quad \hat{\mu}_0 \eqdist \mu + \alpha Z_1\\
\text{IN: } &X_1 \eqdist \mu + Z_0, \quad \hat{\mu}_1 \eqdist \mu + \frac{Z_0}{n} + \beta Z_1
\end{align*}

\textbf{Key Step.} Notice that under the two hypotheses, the distribution of auxiliary samples are the same. Therefore, by the chain rule for KL divergence \citep{polyanskiy2025information},
\begin{equation*}
\KL{\auxsample^m,X_1,\outputmodel_1}{\auxsample^m,X_0,\outputmodel_0}= \EE\left[\KL{X _1,\hat \mu _1 \vert  \auxsample^m }{X _0,\hat \mu _0 \vert  \auxsample^m }\right].
\end{equation*}

Conditioning on auxiliary samples $\auxsample^m = \mathbf{y}^m$ only affects the distribution of $\mu$. By \cref{lem:gauss-posterior}, the posterior is
\[
\mu | \auxsample^m = \mathbf{y}^m \sim \mathcal{N}\left(\frac{m}{m+1}\bar{\mathbf{y}}^m, \frac{1}{m+1}\id{d}\right)
\]
where $\bar{\mathbf{y}}^m = \frac{1}{m}\sum_{i=1}^m y_i$.

Let $\tilde{Z} \sim \mathcal{N}(0,I_d)$ be independent of $(Z_0,Z_1)$. We can express the conditional distributions as
\begin{align*}
X_i | \auxsample^m = \mathbf{y}^m &\eqdist \frac{m}{m+1}\bar{\mathbf{y}}^m + \frac{1}{\sqrt{m+1}}\tilde{Z} + Z_0 \quad \text{(same for both $i=0,1$)}\\
\hat{\mu}_0 | \auxsample^m = \mathbf{y}^m &\eqdist \frac{m}{m+1}\bar{\mathbf{y}}^m + \frac{1}{\sqrt{m+1}}\tilde{Z} + \alpha Z_1\\
\hat{\mu}_1 | \auxsample^m = \mathbf{y}^m &\eqdist \frac{m}{m+1}\bar{\mathbf{y}}^m + \frac{1}{\sqrt{m+1}}\tilde{Z} + \frac{Z_0}{n} + \beta Z_1
\end{align*}

Since these conditional distributions are Gaussian with explicit covariance structures, we can compute their KL divergence in closed form. The key observation is that the KL divergence scales as $O(d/(n + n^2\rho^2+n^2/m))$. Applying Pinsker's inequality \citep{polyanskiy2025information} then yields the desired bound on the total variation distance, completing the proof.
\end{proof}

\section*{Acknowledgments}
The authors would like to thank Clément Canonne, Jeffery Negrea, and Ankit Pensia for helpful discussions, and the anonymous reviewers for pointing out relevant prior work and several (fixable) mistakes.

\addcontentsline{toc}{section}{References}
\printbibliography
\newpage
\appendix
\section{MIA on the Noisy Empirical Mean}
In the next definition, we introduce a particular mean estimator that plays an important role in the sequel. 
\begin{definition} \label{def:noisy-emp-mean}
Fix $d,n \in \Naturals$ and $\rho>0$. Let $\Dist \in \distfamilyuc_{d}$ be a Gaussian distribution with mean $\mu$ and covariance $\Sigma$. Let $\algnoisy$ be an estimator that given $(X_1,\dots,X_n)\sim \Dist^{\otimes n}$ outputs $\hat \mu =\frac{1}{n}\sum_{i=1}^n X_i + \rho Z$ where $Z \sim \Normal(0,\Sigma)$ such that $Z \indep (X_1,\dots,X_n)$. This algorithm satisfies $\EE\left[\norm{\hat \mu - \mu}^2_{\Sigma}\right]\leq d\left(\rho^2 + \frac{1}{n}\right)$ where $\rho>0$ and $\norm{\hat \mu - \mu}_{\Sigma} = \|\Sigma^{-\frac{1}{2}}\left(\hat \mu - \mu\right)\|_2$.
\end{definition}
Let us consider MIA on $\algnoisy$. To establish a performance benchmark, consider an attacker with access to infinite number of auxiliary samples -- equivalently, complete knowledge of the data distribution.  The following proposition characterizes the performance of such an idealized attacker. We skip the proof as the first part is from \citep{dwork2015robust}, and the second part follows by simple calculations.

\begin{proposition}\label{lem:full-info-kl}
Let $c$ be a universal constant. Fix $ d,n \in \Naturals$ and $\rho >0$ Let $\Dist \in \distfamilyuc_d$. Let $(X_0,X_1,\dots,X_n)\sim \Dist^{\otimes (n+1)}$ and $\hat \mu=\algnoisy(X_1,\dots,X_n)$. Then, 
\begin{enumerate}
\item Given $d \geq c\left(n + n^2 \rho^2\right)$, there exists an attacker with the full knowledge of the data distribution, i.e., $m \to \infty$ in \cref{def:problem-def-uc}, that can succeed in the sense of \cref{def:problem-def-uc}.
\item We have 
$\displaystyle
 \TVinline{X_1,\hat \mu}{X_0, \hat \mu} \leq  \sqrt{\frac{d}{n + n^2 \rho^2}}.
$
\end{enumerate}

\end{proposition}

\section{Sample Complexity of  Mahalanobis Distance Estimation} \label{sec:approx-mahalanobis-Distance}
\begin{definition} \label{def:estimator}
Fix $d,n \in \Naturals$ and $\alpha \in (0,1]$. We say a (possibly randomized) algorithm $\Phi: \Reals^d \times (\Reals^d)^n  \to \Reals$ is an $\alpha$-accurate Mahalanobis norm estimator using $n$ samples if the following holds for every unit-norm $y \in \Reals^d$ and positive definite $\Sigma \in \Reals^{d \times d}$: let $\mathbf{X}^n =(X_1,\dots,X_n)\sim \Normal(0,\Sigma)^{\otimes (n)}$,
\[
\Pr\left(  (1 - \alpha)  y^\top \Sigma^{-1} y \leq \Phi(y,\mathbf{X}^n)\leq (1 + \alpha)  y^\top \Sigma^{-1} y\right) \geq 0.95.
\]
\end{definition}

\begin{theorem} \label{thm:mahalanbis-lb}
Let $0<c_1<1$ and $c_2>1$ be universal constants. Let $d \in \Naturals$ such that $d>7$ and $y \in \Reals^d$ be a unit-norm vector. Then, for every $\frac{1}{2}$-accurate Mahalanobis norm estimator for $y$ with $n$ samples, we require $n \geq c_1 d$. To achieve this bound using $n>c_2 d$ samples, a $\frac{1}{2}$-accurate Mahalanobis norm estimator for $y$ is given by $y^\top \hat{\Sigma}^{-1} y$ where $\hat{\Sigma}^{-1} $ is the inverse of the empirical covariance  using $n$ samples.
 \end{theorem}

\begin{proof}
In \cref{lem:reduction}, we show that any Mahalanobis norm estimator   can be converted into a  tester for the following binary hypothesis testing problem:

\begin{definition} \label{def:testing}
Fix $d \in \Naturals$. Define the following  binary hypothesis testing problem:
\[
&H_0: P_0 = \Normal(0,\Sigma_0)~\text{where}~ \Sigma_0 = \id{d},\\
&H_1: \text{sample $v \sim \unif{\mathcal{S}_{d-1}}$ then}~  P_v = \Normal(0,\Sigma_v)~\text{where}~ \Sigma_v = \id{d} -vv^\top +\frac{1}{31d} vv^\top.
\]
\end{definition}
Let $\pi = \unif{\mathcal{S}_{d-1}}$. We say a tester $T :(\Reals^d)^n \to \{0,1\}$ with sample size $n$ has accuracy $0.55$, if 
\[
\Pr_{(X_1,\dots,X_n)\sim P_0^{\otimes n}}\left(T(X_1,\dots,X_n)=1\right) \leq 0.45,\\
\EE_{v \sim \pi}\left[\Pr_{(X_1,\dots,X_n)\sim P_v^{\otimes n}}\left(T(X_1,\dots,X_n)=0\right)\right] \leq 0.45.
\]

Then, in \cref{lem:lb-test}, we show that any tester requires at least $c_1 d$ samples to achieve a non-trivial accuracy. This completes the proof.

The upper bound follows from the well-known results on the estimation of the covariance matrix in the operator norm \citep{vershynin2018high}.
\end{proof}

\subsection{Supporting Lemmas}
\label{sec:supporting-lemmas}

\begin{lemma} \label{lem:lb-norm}
Let $d \geq 7$. Fix a unit-norm vector $y \in \Reals^d$. Let $v \sim \unif{\mathcal{S}_{d-1}}$. Then, with probability at least $\frac{3}{5}$, we have $y^\top \Sigma_v^{-1} y \geq 4$ where $\Sigma_v = \id{d} -vv^\top +\frac{1}{31d} vv^\top$.
\end{lemma}
\begin{proof}
Notice $\Sigma_v^{-1} = \id{d} + (31d-1)vv^\top$. Therefore, 
$
y^\top \Sigma_v^{-1}y = \norm{y}^2 +(31d-1)\left(\inner{v}{y}\right)^2.
$
Next, we provide a probabilistic lowerbound on $\left(\inner{v}{y}\right)^2$. By rotational invariance of $v$, we can assume $y = (1,0,\dots,0)$. Therefore, $\left(\inner{v}{y}\right)^2 \eqdist v_1^2$. Then, we use that $v \eqdist \frac{g}{\norm{g}}$ where $g \sim \Normal(0,\id{d})$, to write
\[
\Pr\left(\left(\inner{v}{y}\right)^2 \geq \frac{1}{10d}\right) &= \Pr\left( \frac{g_1^2}{g_1^2 + \sum_{i=2}^d g_i^2} \geq \frac{1}{10d}\right)\\
&\geq \Pr\left(g_1^2 \geq \frac{1}{4} \wedge \sum_{i=2}^d g_i^2 \leq \frac{9}{4}d \right)\\
&= \Pr\left(g_1^2 \geq \frac{1}{4}\right) \cdot \Pr\left( \sum_{i=2}^d g_i^2 \leq \frac{9}{4}d \right)\\
&\geq 0.60,
\]
where the last step follows from numerical evaluations. Therefore, with probability at least $0.6$, we have
\[
y^\top \Sigma_v^{-1}y = \norm{y}^2 +(31d-1)\left(\inner{v}{y}\right)^2 \geq 1 + (31d-1) \frac{1}{10d}\geq 4.
\]
\end{proof}

\begin{lemma} \label{lem:reduction}
Fix a $\frac{1}{2}$-accurate Mahalanobis norm estimator denoted by $\Phi$ with sample size $n$ (see \cref{def:estimator}.). Then, by postprocessing the output of $\Phi$, we can convert $\Phi$ to a hypothesis tester that distinguishes $H_0$ from $H_1$ with accuracy at least $0.55$
\end{lemma}
\begin{proof}
The tester we propose is defined as follows. The tester first fixes a unit-norm vector $y \in \Reals^d$. Then, its decision rule is given by
\[
T(\dataset^n) = \begin{cases}
  0 & \text{if } \Phi(y,\dataset^n) \leq 1.75 \\
  1  & \text{otherwise}.
  \end{cases}
\]
First consider the problem under $H_0$. Let $\dataset^n \sim P_0^{\otimes n}$. Under $H_0$, we have $\Sigma_0 = \id{d}$. Therefore, $y^\top \Sigma_0^{-1}y=\norm{y}^2=1$. Based on \cref{def:estimator}, we have $0.5\leq \Phi(y,\dataset^n)\leq 1.5$ with probability at least $0.95$. Therefore,
\[
\Pr_{\dataset^n \sim P_0^{\otimes n}}\left(T(\dataset^n)=1\right)&= \Pr_{\dataset^n \sim P_0^{\otimes n}}\left(\Phi(y,\dataset^n) > 1.75\right)\\
&\leq 0.05\\
&\leq 0.45.
\]
Thus, it satisfies the accuracy condition under $H_0$. Next, consider the accuracy under $H_1$. 

\[
\EE_{v \sim \pi}\left[\Pr_{\dataset^n \sim P_v^{\otimes n}}\left(T(\dataset^n)=0\right)\right] &= \EE_{v \sim \pi}\left[\Pr_{\dataset^n \sim P_v^{\otimes n}}\left(\Phi(y,\dataset^n) \leq 1.75\right)\right] \\
&=\EE_{v \sim \pi}\left[\Pr_{\dataset^n \sim P_v^{\otimes n}}\left(\Phi(y,\dataset^n) \leq 1.75\right)\big| y^\top \Sigma_v^{-1}y \geq 4\right] \Pr_{v \sim \pi}\left(y^\top \Sigma_v^{-1}y \geq 4\right)  \\
&+ \EE_{v \sim \pi}\left[\Pr_{\dataset^n \sim P_v^{\otimes n}}\left(\Phi(y,\dataset^n) \leq 1.75\right)\big| y^\top \Sigma_v^{-1}y <  4\right] \Pr_{v \sim \pi}\left(y^\top \Sigma_v^{-1}y < 4\right)\\
&\leq  \EE_{v \sim \pi}\left[\Pr_{\dataset^n \sim P_v^{\otimes n}}\left(\Phi(y,\dataset^n) \leq 1.75\right)\big| y^\top \Sigma_v^{-1}y \geq 4\right]  + \Pr_{v \sim \pi}\left(y^\top \Sigma_v^{-1}y < 4\right).
\]
From \cref{lem:lb-norm}, we have $\Pr_{v \sim \pi}\left(y^\top \Sigma_v^{-1}y \geq 4\right) \geq \frac{3}{5}$. Under the event that $\{ y^\top \Sigma_v^{-1}y \geq 4\}$, based on \cref{def:estimator}, we have that $ 2 \leq \Phi(y,\dataset^n)$ with probability at least $0.95$. Therefore,
\[
 &\EE_{v \sim \pi}\left[\Pr_{\dataset^n \sim P_v^{\otimes n}}\left(\Phi(y,\dataset^n) \leq 1.75\right)\big| y^\top \Sigma_v^{-1}y \geq 4\right]  + \Pr_{v \sim \pi}\left(y^\top \Sigma_v^{-1}y < 4\right)\\
 &\leq \frac{1}{20}  + \frac{2}{5}\\
 &= 0.45,
\]
as was to be shown.
\end{proof}

\begin{lemma} \label{lem:lb-test}
Let $c$ be a universal constant and $d \geq 7$ be an integer. Recall the testing problem  in \cref{def:testing}. Then, any hypothesis tester for distinguishing $H_0$ from $H_1$ using $n$ samples with an accuracy of $0.55$ requires at least $n \geq c d$ samples.
\end{lemma}
\begin{proof}
In this proof, we use a proof technique attributed to Ingster and Suslina \citep{ingster2012nonparametric}. Before delving into details of the proof, we review some properties of chi-squared divergence.  For proofs and more discussions see \citep{polyanskiy2025information,duchi2023lecture}. For two distributions $P$ and $Q$ such that $P \ll Q$ defined over $\Reals^d$ with probability density functions (wrt to Lebesgue measure) $p$ and $q$, the chi-squared divergence is given by
\[
\chisqdist{P}{Q} = \int_{x \in \Reals^d} \frac{p^2(x)}{q(x)}\text{d}x - 1.
\]

Let $P_0 = \Normal(0,\id{d})$ and recall the definition of $ \EE_{v\sim \pi}\left[P_v \right]$ from  \cref{def:testing}. To show the hardness of testing, our goal is to characterize the total variation distance using \citealp[Prop.~2.9]{duchi2023lecture} which states
\begin{equation}
\label{eq:tv-dist}
\TV{\EE_{v\sim \pi}\left[P_v^{\otimes n} \right]}{P_0^{\otimes n}} \leq \sqrt{\frac{1}{2} \log\left( \chisqdist{\EE_{v\sim \pi}\left[P_v^{\otimes n} \right]}{P_0^{\otimes n}} + 1\right)}.
\end{equation}
We also need \citealp[Lem.~13.2.3] {duchi2023lecture} which states
\begin{equation}
\label{eq:tensor-pf}
1 + \chisqdist{\EE_{v\sim \pi}\left[P_v^{\otimes n} \right]}{P_0^{\otimes n}} = \EE_{(v_1,v_2)\sim \pi^{\otimes 2}}\left[ \left(\EE_{X \sim P_0} \left[\frac{p_{v_1}(X)p_{v_2}(X)}{p_0^2(X)}\right]\right)^n\right] .
\end{equation}

In the first step we invoke \cref{lem:chi-square}, 
\begin{equation}\label{eq:chi-sq-term}
\begin{aligned} 
\EE_{X \sim P_0} \left[\frac{p_{v_1}(X)p_{v_2}(X)}{p_0^2(X)}\right] &= \int_{x\in \Reals^d}   \frac{p_{v_1}(x)p_{v_2}(x)}{p_0(x)} \text{d}x \\
&=(\det(\Sigma_{v_1}) \det(\Sigma_{v_2}))^{-1/2} (\det(\Sigma_{v_1}^{-1}+\Sigma_{v_2}^{-1}-I))^{-1/2},
\end{aligned}
\end{equation}
Notice that $\det(\Sigma_{v_1}) = \det(\Sigma_{v_2}) = \frac{1}{31d}$. Then, we have
\begin{equation}\label{eq:det-term1}
\begin{aligned}
\Sigma_{v_1}^{-1} + \Sigma_{v_2}^{-1} - \id{d} &= \id{d} + (31d-1)v_1v_1^\top + \id{d} + (31d-1)v_2v_2^\top - \id{d}\\
& = \id{d} + (31d-1)v_1v_1^\top + (31d-1)v_2v_2^\top.
\end{aligned}
\end{equation}

For notational convenience  let $\gamma = 31d-1$. From \cref{lem:determ-rank2}, we have
\begin{equation} \label{eq:det-term2}
\det\left(\id{d} +  \gamma v_1v_1^\top +  \gamma v_2v_2^\top\right) =  1+2\gamma + \gamma^2\left(1 - \left(\inner{v_1}{v_2}\right)^2\right).
\end{equation}
Thus, plugging \cref{eq:det-term1,eq:det-term2} into \cref{eq:chi-sq-term}, we obtain
\[
\EE_{X \sim p_0} \left[\frac{p_{v_1}(X)p_{v_2}(X)}{p_0^2(X)}\right] &= (\det(\Sigma_{v_1}) \det(\Sigma_{v_2}))^{-1/2} (\det(\Sigma_{v_1}^{-1}+\Sigma_{v_2}^{-1}-I))^{-1/2}\\
& = \left( \frac{1+2\gamma + \gamma^2\left(1-\left(\inner{v_1}{v_2}\right)^2\right)}{\left(1+\gamma\right)^2}\right)^{-1/2}\\
& = \left( 1 - \frac{\gamma^2}{\left(1+\gamma\right)^2} \left(\inner{v_1}{v_2}\right)^2\right)^{-1/2}.
\]
Using this result by \cref{eq:tensor-pf}, we can write
\begin{equation} \label{eq:simplifiy-chi-squared}
\begin{aligned}
1 + \chisqdist{\EE_{v\sim \pi}\left[P_v^{\otimes n} \right]}{P_0^{\otimes n}} &= \EE_{(v_1,v_2)\sim \pi^{\otimes 2}}\left[ \left(\EE_{X \sim P_0} \left[\frac{p_{v_1}(X)p_{v_2}(X)}{p_0^2(X)}\right]\right)^n\right]\\
& =  \EE_{(v_1,v_2)\sim \pi^{\otimes 2}}\left[\left(1 - \frac{\gamma^2}{\left(1+\gamma\right)^2} \left(\inner{v_1}{v_2}\right)^2\right)^{-\frac{n}{2}}\right]\\
&\leq \EE_{(v_1,v_2)\sim \pi^{\otimes 2}}\left[ \left(1 -  \left(\inner{v_1}{v_2}\right)^2\right)^{-\frac{n}{2}}\right].
\end{aligned}
\end{equation}
Consider $\inner{v_1}{v_2}$. By rotational invariance, we can assume that $v_2 = [1,0,\cdots,0]^\top$. Therefore, $\inner{v_1}{v_2} \eqdist v_1^{(1)}$, i.e., the first element of a uniformly distributed random variable on the unit-sphere in $\Reals^d$. We use the following well known fact: the distribution of $\left(v_1^{(1)}\right)^2$ follows beta distribution with parameters $\mathsf{beta}\left(\frac{1}{2},\frac{d-1}{2}\right)$ \citep{wikibetaderived}. Then, we use the closed-form expression for the PDF of the beta distribution to write 

\[
\EE_{(v_1,v_2)\sim \pi^{\otimes 2}}\left[ \left(1 -  \left(\inner{v_1}{v_2}\right)^2\right)^{-\frac{n}{2}}\right] &= \frac{1}{\mathrm{B}\left(\frac{1}{2},\frac{d-1}{2}\right)}  \int_{x\in[0,1]} \frac{1}{(1-x)^{\frac{n}{2}}} x^{-\frac{1}{2}} (1-x)^{\frac{d-3}{2}} \text{d}x\\
& =  \frac{1}{\mathrm{B}\left(\frac{1}{2},\frac{d-1}{2}\right)}  \int_{x\in[0,1]}  x^{-\frac{1}{2}} (1-x)^{\frac{d-n-3}{2}} \text{d}x\\
& = \frac{\mathrm{B}\left(\frac{1}{2},\frac{d-n-1}{2}\right)}{\mathrm{B}\left(\frac{1}{2},\frac{d-1}{2}\right)},
\]
where $\mathrm{B}$ is the beta function. We also need to have $d> n + 1$ so that integral exists. As we will see, this assumption is benign. By the equivalent definition of the beta function in terms of Gamma function,
\[
\frac{\mathrm{B}\left(\frac{1}{2},\frac{d-n-1}{2}\right)}{\mathrm{B}\left(\frac{1}{2},\frac{d-1}{2}\right)} = \frac{\Gamma\left(\frac{d-n-1}{2}\right)\Gamma\left(\frac{d}{2}\right)}{\Gamma\left(\frac{d-n}{2}\right)\Gamma\left(\frac{d-1}{2}\right)}.
\]
To control this term, we use the well-known  Gautschi's inequality \citep{gautschi1959some} which states that for every $x>1$, we have
\[
\sqrt{x-1} \leq \frac{\Gamma(x)}{\Gamma(x-\frac{1}{2})}\leq \sqrt{x},
\]
Therefore, using the previous steps, we obtain
\begin{equation} \label{eq:pf-beta}
\begin{aligned}
\EE_{(v_1,v_2)\sim \pi^{\otimes 2}}\left[ \left(1 -  \left(\inner{v_1}{v_2}\right)^2\right)^{-\frac{n}{2}}\right] &=\frac{\Gamma\left(\frac{d-n-1}{2}\right)\Gamma\left(\frac{d}{2}\right)}{\Gamma\left(\frac{d-n}{2}\right)\Gamma\left(\frac{d-1}{2}\right)}\\
&\leq \sqrt{\frac{d}{d-n-1}}\\
&= \sqrt{1 + \frac{n+1}{d-n-1}}.
\end{aligned}
\end{equation}
Combining all the previous steps, we obtain
\[
\TV{\EE_{v\sim \pi}\left[P_v^{\otimes n} \right]}{P_0^{\otimes n}} &\leq \sqrt{\frac{1}{2} \log\left( \chisqdist{\EE_{v\sim \pi}\left[P_v^{\otimes n} \right]}{P_0^{\otimes n}} + 1\right)} && \left(\text{From \cref{eq:tv-dist}}\right)\\
& \leq  \sqrt{\frac{1}{2} \log\left( \EE_{(v_1,v_2)\sim \pi^{\otimes 2}}\left[ \left(1 -  \left(\inner{v_1}{v_2}\right)^2\right)^{-\frac{n}{2}}\right]\right)} && \left(\text{From \cref{eq:simplifiy-chi-squared}}\right)\\
& \leq \sqrt{\frac{1}{2} \log\left(\sqrt{1 + \frac{n+1}{d-n-1}}\right)} && \left(\text{From \cref{eq:pf-beta}}\right)\\
& =  \sqrt{\frac{1}{4} \log\left(1 + \frac{n+1}{d-n-1}\right)} \\
& \leq \frac{1}{2} \sqrt{\frac{n+1}{d-n-1}}.&& (\log(1+x)\leq x)
\]
Therefore, given $d \geq (1+c^2) \left(n+1\right)$ for a sufficiently (and universally) large $c$, we have that 
$$\TVinline{\EE_{v\sim \pi}\left[P_v^{\otimes n} \right]}{P_0^{\otimes n}} \leq \frac{1}{2c},
$$ 
as was to be shown.
\end{proof}

\section{Proofs from \cref{sec:unknown-cov}}\label{appx:pf-unknowncov}

\subsection{Proof of the Lower Bound}

\redcuctionzeroaux*
\begin{proof}
Recall from \cref{def:problem-def-uc} that $\Sigma = \id{d} + \sigma^2 UU^\top$. Lets consider the following two sets of random variables:
\[
\auxsample_{\Vert}^m  = \{UU^\top Y_i\}_{i\in [m]}, \quad, \auxsample_{\perp}^m  = \{(\id{d}-UU^\top) Y_i\}_{i\in [m]}.
\]
Since $d>2m$ and $k > m$, we have that  $\left(\auxsample_{\Vert}^m\right) $ and $\left(\auxsample_{\perp}^m\right) $ are linearly independent vectors with probability one. Therefore  $\mathrm{span}\left(\auxsample_{\Vert}^m\right) $ and $\mathrm{span}\left(\auxsample_{\perp}^m\right) $ are two $m$ dimensional subspaces in $\Reals^d$. Let us denote this event by $\mathcal G$. In the rest of the argument, we condition on this almost sure event. 

Using data processing inequality, we have
\begin{equation}
\label{eq:tv-reduction-fist-step}
\TV{ \auxsample^m ,X _0,\outputmodel _0}{\auxsample^m ,X _1,\outputmodel _1} \leq \TV{X _0,\outputmodel _0,\auxsample_{\Vert}^m ,\auxsample_{\perp}^m }{X_1,\outputmodel _1, \auxsample_{\Vert}^m ,\auxsample_{\perp}^m }.
\end{equation}

Define the following three orthogonal projection matrices $(\Pi_\Vert,\Pi_\perp,\Pi_\text{unk})$: 
\begin{itemize}
\item $\Pi_\perp \in \Reals^{d \times d}$ is the projection matrix onto the subspace spanned by $\left(\auxsample_{\perp}^m\right) $,
\item $\Pi_\Vert \in \Reals^{d \times d}$ is the orthogonal projection matrix onto the subspace spanned by $\left(\auxsample_{\Vert}^m\right) $, 
\item  $\Pi_\text{unk} = \id{d} - \left(\Pi_\Vert + \Pi_\perp\right)$.
\end{itemize}

 Define the random variable $\mathrm{Aux} = (\auxsample_{\Vert}^m,\auxsample_{\perp}^m,\Pi_\perp,\Pi_\Vert,\Pi_\text{unk})$. Using the chain rule for TV from \cref{lem:chain-rule-tv} and the fact that for every $x \in \Reals^d$, we have $\Pi_\perp(x)+\Pi_\Vert(x)+\Pi_\text{unk}(x)=x$, we can write
\begin{align}
&\TV{X _0,\outputmodel _0,\auxsample_{\Vert}^m ,\auxsample_{\perp}^m ,\Pi_\Vert,\Pi_\perp,\Pi_\text{unk}}{X _1,\outputmodel _1,\auxsample_{\Vert}^m ,\auxsample_{\perp}^m ,\Pi_\Vert,\Pi_\perp,\Pi_\text{unk}} \nonumber\\
&\leq \underbrace{\EE\left[\TV{\Pi_\perp\left(X _0,\outputmodel _0\right)\Big|\mathrm{Aux}}{\Pi_\perp\left(X _1,\outputmodel _1\right)\Big|\mathrm{Aux}}\right]}_{\text{First Term}} \nonumber\\
&+ \underbrace{\EE\left[\TV{\Pi_\Vert\left(X _0,\outputmodel _0\right)\Big|\mathrm{Aux},\Pi_\perp\left(X _0,\outputmodel _0\right)}{\Pi_\Vert\left(X _1,\outputmodel _1\right)\Big|\mathrm{Aux},\Pi_\perp\left(X _1,\outputmodel _1\right)}\right]}_{\text{Second term}} \nonumber\\
&+ 
\underbrace{\EE\left[\TV{\Pi_\text{unk}\left(X _0,\outputmodel _0\right)\Big|\mathrm{Aux},\Pi_\Vert\left(X _0,\outputmodel _0\right),\Pi_\perp\left(X _0,\outputmodel _0\right)}{\Pi_\text{unk}\left(X _1,\outputmodel _1\right)\Big|\mathrm{Aux},\Pi_\Vert\left(X _1,\outputmodel _1\right),\Pi_\perp\left(X _1,\outputmodel _1\right)}\right]}_{\text{Third term}} \label{eq:decompos-reduc}.
\end{align}
In what follows, we provide an upper bound on each term separately.  First, recall \cref{def:problem-def-uc} and the construction in \cref{def:gen-process-cov}. Let $(Z_0,Z_1)\sim \Normal(0,\id{d})^{\otimes 2}$ and $(W_0,W_1)\sim \Normal(0,\id{k})^{\otimes 2}$ where $(W_0,W_1,Z_0,Z_1)$ are mutually independent. Also, let $\alpha = \sqrt{\frac{1}{n}+\rho^2}$ and $\beta = \sqrt{\frac{n-1}{n^2}+\rho^2}$. Then, we can represent the random variables we considered as
\begin{equation} \label{eq:unc-rep-rvs}
\begin{aligned}
X_0  &= Z_0 + \sigma U W_0, \quad \outputmodel_0  = \alpha\left( Z_1 + \sigma U W_1 \right). \\
X_1  &= Z_0 + \sigma U W_0, \quad \outputmodel_1  = \frac{Z_0 + \sigma U W_0}{n} + \beta\left( Z_1 + \sigma U W_1 \right).
\end{aligned}
\end{equation}
Also, notice that conditioning on $\mathrm{Aux}$ only changes the distribution of $UU^\top$\footnote{Notice that $U$ is not  unique. However, as we will show, the final result is invariant of the choice of $U$.}. 

\paragraph{First Term in \cref{eq:decompos-reduc}:}
First, we claim that $\Pi_{\perp}U=0$. The proof is as follows: let $UU^\top = \sum_{j=1}^k u_ju_j^\top$ where $\{u_1,\dots,u_k\}$ are orthonormal basis. Then, for every $j\in [k]$ and every $i \in [m]$, we have $\inner{u_j}{\left(\id{d}-UU^\top\right)Y_i}=\inner{u_j}{Y_i} - \inner{u_j}{UU^\top Y_i}=\inner{u_j}{Y_i} - \inner{UU^\top \left(u_j\right)}{ Y_i}=\inner{u_j}{Y_i} - \inner{u_j}{Y_i}=0$ where we used the fact that $u_j$ are in the subspace whose projection matrix is $UU^\top$. Therefore, every $u_j$ is orthogonal to $\mathrm{span}\left(\left(\auxsample_{\perp}^m\right) \right)$. Then, using this observation and \cref{eq:unc-rep-rvs}
\begin{equation} \label{eq:unc-rep-rvs-term1}
\begin{aligned}
\Pi_{\perp} X_0  &= \Pi_{\perp} Z_0, \quad \Pi_{\perp} \outputmodel_0  = \alpha \Pi_{\perp} Z_1. \\
\Pi_{\perp} X_1  &= \Pi_{\perp} Z_0 , \quad \Pi_{\perp}\outputmodel_1  = \frac{\Pi_{\perp} Z_0}{n} + \beta \Pi_{\perp} Z_1.
\end{aligned}
\end{equation} 
Consider an arbitrary eigenvalue decomposition of $\Pi_\perp = VV^\top$ where $V \in \Reals^{d \times m}$ is a matrix with orthonormal columns with probability one. Then,
\[
&\EE\left[\TV{\Pi_\perp\left(X _0,\outputmodel _0\right)\Big|\mathrm{Aux}}{\Pi_\perp\left(X _1,\outputmodel _1\right)\Big|\mathrm{Aux}}\right] \\
&= \EE\left[\TV{\Pi_{\perp} Z_0, \alpha \Pi_{\perp} Z_1\Big|\mathrm{Aux}}{\Pi_{\perp} Z_0,\frac{ \Pi_{\perp} Z_0}{n} +\beta \Pi_{\perp} Z_1\Big|\mathrm{Aux}}\right]\\
&= \EE\left[\TV{VV^\top Z_0, \alpha VV^\top Z_1\Big|\mathrm{Aux}}{VV^\top Z_0,\frac{ VV^\top Z_0}{n} +\beta VV^\top Z_1\Big|\mathrm{Aux}}\right]\\
&\stackrel{(a)}{=} \EE\left[\TV{V^\top Z_0, \alpha  V^\top Z_1\Big|\mathrm{Aux}}{V^\top Z_0,\frac{V^\top Z_0}{n} + \beta  V^\top Z_1\Big|\mathrm{Aux}}\right]\\
&\leq c \sqrt{\frac{m}{n + n^2 \rho^2}},
\]
where $(a)$ follows from invariance of the total variation distance to one-to-one mappings. Then, notice that $V^\top Z_0 \sim \Normal(0,\id{m})$, $V^\top Z_1 \sim \Normal(0,\id{m})$, and $V^\top Z_1 \indep V^\top Z_0$. The last step follows since the total variation in Step $(a)$ corresponds to MIA for the full-information setting in dimension $m$. (See \cref{lem:full-info-kl}.) Finally, notice that the final bound is invariant to the choice of $V$ and $U$.

\paragraph{Second Term in \cref{eq:decompos-reduc}:}
First, from \cref{eq:unc-rep-rvs}, we can write
\begin{equation} \label{eq:unc-rep-rvs-term2}
\begin{aligned}
\Pi_\Vert X_0  &= \Pi_\Vert Z_0 + \sigma \Pi_\Vert U W_0, \quad \Pi_\Vert \outputmodel_0  = \alpha\left( \Pi_\Vert Z_1 + \sigma \Pi_\Vert U W_1 \right). \\
\Pi_\Vert X_1  &= \Pi_\Vert Z_0 + \sigma \Pi_\Vert U W_0, \quad \Pi_\Vert \outputmodel_1  = \frac{\Pi_\Vert Z_0 + \sigma \Pi_\Vert U W_0}{n} + \beta\left( \Pi_\Vert Z_1 + \sigma \Pi_\Vert U W_1 \right).
\end{aligned}
\end{equation}
Note from \cref{eq:unc-rep-rvs-term1}, we have the following Markov chain $UU^\top \leftrightarrow \mathrm{Aux} \leftrightarrow\Pi_\perp\left(X _j,\outputmodel _j\right)$ for $j \in \{0,1\}$. The key observation is that since $\Pi_\Vert$ and $\Pi_\perp$ are the projection matrices onto the orthogonal subspaces, by the properties of Gaussian random variables, we have $\Pi_\Vert Z_j \indep \Pi_\perp Z_j$ for $j \in \{0,1\}$.  Therefore, from \cref{eq:unc-rep-rvs-term1} and \cref{eq:unc-rep-rvs-term2}, 
\[
&\EE\left[\TV{\Pi_\Vert\left(X _0,\outputmodel _0\right)\Big|\mathrm{Aux},\Pi_\perp\left(X _0,\outputmodel _0\right)}{\Pi_\Vert\left(X _1,\outputmodel _1\right)\Big|\mathrm{Aux},\Pi_\perp\left(X _1,\outputmodel _1\right)}\right] \\
&= \EE\left[\TV{\Pi_\Vert\left(X _0,\outputmodel _0\right)\Big|\mathrm{Aux}}{\Pi_\Vert\left(X _1,\outputmodel _1\right)\Big|\mathrm{Aux}}\right].
\]
The second key observation is the following: for a \underline{fixed} $UU^\top$ and a \underline{fixed} $\Pi_{\Vert}$, consider the covariance matrix of the Gaussian random variable $ \Pi_\Vert Z_j + \sigma \Pi_\Vert U W_j$ for $j \in \{0,1\}$. The covariance matrix is given by $\Pi_\Vert +\sigma^2 \Pi_{\Vert}UU^\top \Pi_{\Vert}$. We claim that $\Pi_\Vert +\sigma^2 \Pi_{\Vert}UU^\top \Pi_{\Vert} = (1+\sigma^2)\Pi_{\Vert}$. The proof is as follows: by construction $\im{\Pi_{\Vert}} \subseteq \im{UU^\top}$ and by definition for every $y \in \im{UU^\top}$, we have $UU^\top y =y$. Thus, for an arbitrary vector $r$, we have $UU^\top \Pi_{\Vert}r = \Pi_{\Vert}r$. Then, because $\Pi_{\Vert}$ is a projection matrix, we have $\Pi_{\Vert}^2=\Pi_{\Vert}$. Therefore, $\left(\Pi_\Vert +\sigma^2 \Pi_{\Vert}UU^\top \Pi_{\Vert} \right)r = \Pi_\Vert r + \sigma^2 \Pi_\Vert^2r = \left(1+\sigma^2\right)\Pi_\Vert r$. This observation lets us to show the following: Let $B \subseteq \Reals^d$. Then,
\[
\Pr\left( \Pi_\Vert Z_j + \sigma \Pi_\Vert U W_j \in B \vert \text{Aux}\right) &= \EE_{UU^\top \vert \text{Aux}}\left[\Pr\left( \Pi_\Vert Z_j + \sigma \Pi_\Vert U W_j \in B \vert \text{Aux},UU^\top\right)\right]\\
&=\Pr_{Z_j \sim \Normal(0,\id{d})} \left(\sqrt{1+\sigma^2}\Pi_{\Vert}Z_j \in B\right).
\]
Based on this result, we can simplify \cref{eq:unc-rep-rvs-term2}: Let $(Z_0,Z_1)\sim \Normal(0,\id{d})^{\otimes 2}$. Then,
\begin{equation} \label{eq:unc-rep-rvs-term2-v2}
\begin{aligned}
\Pi_\Vert X_0  &\eqdist \sqrt{1+\sigma^2}\Pi_\Vert Z_0, \quad \Pi_\Vert \outputmodel_0  \eqdist \alpha\sqrt{1+\sigma^2}\Pi_\Vert Z_1. \\
\Pi_\Vert X_1  &\eqdist \sqrt{1+\sigma^2}\Pi_\Vert Z_0, \quad \Pi_\Vert \outputmodel_1  \eqdist \frac{\sqrt{1+\sigma^2}\Pi_\Vert Z_0}{n} + \beta\left( \sqrt{1+\sigma^2}\Pi_\Vert Z_1 \right).
\end{aligned}
\end{equation}
By the invariance of the total variation distance to one-to-one mappings, we can divide all the random variables by $\sqrt{1+\sigma^2}$ to write
\[
\EE\left[\TV{\Pi_\Vert\left(X _0,\outputmodel _0\right)\Big|\mathrm{Aux}}{\Pi_\Vert\left(X _1,\outputmodel _1\right)\Big|\mathrm{Aux}}\right]  =\EE\left[\TV{\Pi_\Vert Z_0,\alpha \Pi_\Vert Z_1\Big|\mathrm{Aux}}{\Pi_\Vert Z_0,\frac{\Pi_\Vert Z_0}{n} + \beta\left(\Pi_\Vert Z_1 \right)\Big|\mathrm{Aux}}\right].
\]
Then, consider an arbitrary eigenvalue decomposition of $\Pi_\Vert = VV^\top$ where $V \in \Reals^{d \times m}$ is a matrix with orthonormal columns with probability one. Then, notice that $\Pi_\Vert Z_j = VV^\top Z_j$ and $V^\top Z_j \sim \Normal(0,\id{m})$. 
\[
&\EE\left[\TV{\Pi_\Vert Z_0,\alpha \Pi_\Vert Z_1\Big|\mathrm{Aux}}{\Pi_\Vert Z_0,\frac{\Pi_\Vert Z_0}{n} + \beta\left(\Pi_\Vert Z_1 \right)\Big|\mathrm{Aux}}\right] \\
&\stackrel{(a)}{=} \EE\left[\TV{V^\top Z_0, \alpha  V^\top Z_1\Big|\mathrm{Aux}}{V^\top Z_0,\frac{V^\top Z_0}{n} + \beta  V^\top Z_1\Big|\mathrm{Aux}}\right]\\
&\leq c \sqrt{\frac{m}{n + n^2 \rho^2}},
\]
where the last step follows since the total variation in Step $(a)$ corresponds to MIA in the full-information setting in dimension $m$ (See \cref{lem:full-info-kl}.) Also, notice that the final bound is invariant to the choice of $V$. 

\paragraph{Third Term in \cref{eq:decompos-reduc}:}
In the first step, for $i \in \{0,1\}$, we claim that the following Markov chain holds:
\[
\left(\Pi_\Vert\left(X _i,\outputmodel _i\right),\Pi_\perp\left(X _i,\outputmodel _i\right)\right)\leftrightarrow \mathrm{Aux} \leftrightarrow \Pi_\text{unk}\left(X _i,\outputmodel _i\right)
\]
To prove this claim we use \cref{lem:cond-indep}. In particular, using the notation of \cref{lem:cond-indep}, let $A = \left(\Pi_\Vert\left(X _i,\outputmodel _i\right),\Pi_\perp\left(X _i,\outputmodel _i\right)\right)$, $B = \Pi_\text{unk}\left(X _i,\outputmodel _i\right)$, $C = \mathrm{Aux}$, and $M = UU^\top$. First of all as we showed in the previous two steps, $\Pi_\Vert\left(X _i,\outputmodel _i\right),\Pi_\perp\left(X _i,\outputmodel _i\right)$ are independent of $UU^\top$ conditioned on $\mathrm{Aux}$. Then, note that  conditioned on $UU^\top$ and $\mathrm{Aux}$, $A$ and $B$ are jointly Gaussian random variables and for Gaussian random variables, independence is equivalent to zero correlation. Then, using these facts it is straightforward to show that $A\indep B \big| \mathrm{Aux},UU^\top$. Therefore, by \cref{lem:cond-indep}, we have the claimed Markov chain.  Using this step, 
\begin{align}
&\EE\left[\TV{\Pi_\text{unk}\left(X_0,\outputmodel _0\right)\Big|\mathrm{Aux},\Pi_\Vert\left(X_0,\outputmodel _0\right),\Pi_\perp\left(X_0,\outputmodel_0\right)}{\Pi_\text{unk}\left(X_1,\outputmodel _1\right)\Big|\mathrm{Aux},\Pi_\Vert\left(X_1,\outputmodel _1\right),\Pi_\perp\left(X_1,\outputmodel_1\right)}\right] \nonumber\\
&=\EE\left[\TV{\Pi_\text{unk}\left(X _0,\outputmodel _0\right)\Big|\mathrm{Aux}}{\Pi_\text{unk}\left(X_1,\outputmodel _1\right)\Big|\mathrm{Aux}}\right]. \label{eq:term3-condition}
\end{align}

In the next step, we prove a structural result regarding the conditional distribution of $UU^\top$ on $\mathrm{Aux}$. Let $V_1 V_1^\top = \Pi_{\Vert}$ and $V_2 V_2^\top = \Pi_{\text{unk}}$ where $V_1 \in \Reals^{d \times m}$ and $V_2 \in \Reals^{d \times (d-2m)}$ are matrices with orthonormal columns with probability one. Also, let $\tilde{U}\tilde{U}^\top$ be a uniformly random subspace of dimension $k-m$ in $\Reals^{d-2m}$ which is independent of all other random variables, therefore, $\tilde{U} \in \Reals^{(d-2m)\times (k-m)}$. Then, combining \cref{lem:description-proj} and \cref{lem:posterior-random-dir}, we have
\begin{equation} \label{eq:cond-dist-proj}
UU^\top \big| \mathrm{Aux} \eqdist V_1 V_1^\top + V_2 \tilde{U}\tilde{U}^\top V_2^\top.
\end{equation}

The proof of this result is as follows: In \cref{lem:description-proj}, we give a general description of all the projection matrices $\Pi$ that satisfies the following two conditions: given a fixed set of vectors $(y_1,\dots,y_m)\in (\Reals^d)^m$ and $(x_i=\Pi(y_i))_{i \in [m]}$. Then,  in \cref{lem:posterior-random-dir}, we show that since we choose the prior as uniform, the conditional distribution is also uniform over all the projection matrices characterized by \cref{lem:description-proj}.

Using this key result, we represent the random variables in \cref{eq:term3-condition} as follows: let $(\tilde{Z}_0,\tilde{Z}_1)\sim \Normal(0,\id{d-2m})^{\otimes 2}$ and $(\tilde{W}_0,\tilde{W}_1)\sim \Normal(0,\id{k-m})^{\otimes 2}$ where $(\tilde{W}_0,\tilde{W}_1,\tilde{Z}_0,\tilde{Z}_1)$ are mutually independent. Also, $\alpha = \sqrt{\frac{1}{n}+\rho^2}$ and $\beta = \sqrt{\frac{n-1}{n^2}+\rho^2}$. Then,
\begin{equation} \label{eq:unc-rep-rvs-term4}
\begin{aligned}
&\Pi_\text{unk} X_0 \big| \mathrm{Aux} \eqdist V_2 \tilde{Z}_0 + \sigma V_2 \tilde{U} \tilde{W}_0\\
& \Pi_\text{unk} \outputmodel_0 \big| \mathrm{Aux}  \eqdist \alpha\left( V_2 \tilde{Z}_1 + \sigma V_2 \tilde{U} \tilde{W}_1 \right). \\
&\Pi_\text{unk} X_1 \big| \mathrm{Aux}  \eqdist V_2 \tilde{Z}_0 + \sigma V_2 \tilde{U} \tilde{W}_0\\
& \Pi_\text{unk} \outputmodel_1 \big| \mathrm{Aux} \eqdist \frac{ V_2 \tilde{Z}_0 + \sigma V_2 \tilde{U} \tilde{W}_0}{n} + \beta\left(\tilde{Z}_1 + \sigma V_2 \tilde{U} \tilde{W}_1 \right),
\end{aligned}
\end{equation}

For notational convenience, let $A = \tilde{Z}_0 + \sigma\tilde{U}\tilde{W}_0$ and $B = \tilde{Z}_1 + \sigma\tilde{U}\tilde{W}_1$. Then, we can simplify \cref{eq:term3-condition} as follows  
\[
&\EE\left[\TV{\Pi_\text{unk}\left(X _0,\outputmodel _0\right)\Big|\mathrm{Aux}}{\Pi_\text{unk}\left(X_1,\outputmodel _1\right)\Big|\mathrm{Aux}}\right]\\
& = \EE\left[\TV{ \left( V_2 A,\alpha V_2 B\right) \Big|\mathrm{Aux}}{ \left( V_2 A,\frac{ V_2 A}{n} + \beta\left(V_2 B \right)\right) \Big|\mathrm{Aux}}\right]\\
& \stackrel{(a)}{=} \EE\left[\TV{ \left( A,\alpha B  \right) \Big|\mathrm{Aux}}{ \left(  A,\frac{  A}{n} + \beta  B\right) \Big|\mathrm{Aux}}\right]\\
& =\TV{ \left( \tilde{Z}_0 + \sigma \tilde{U} \tilde{W}_0 ,\alpha\left( \tilde{Z}_1 + \sigma \tilde{U} \tilde{W}_1 \right)\right) }{ \left(  \tilde{Z}_0 + \sigma  \tilde{U} \tilde{W}_0,\frac{  \tilde{Z}_0 + \sigma \tilde{U} \tilde{W}_0}{n} + \beta\left(V_2 \tilde{Z}_1 + \sigma \tilde{U} \tilde{W}_1 \right)\right)},
\]
where  Step $(a)$ follows from the invariance of TV to one-to-one mappings. Notice that  the last step can be written as follows: let $(\widetilde{X_0},\widetilde{ X_1},\widetilde{\outputmodel_0},\widetilde{\outputmodel_1})=\textsc{Sample}\left(n,d-2m,k-m,0,\sigma\right)$, then,

\[
&\TV{ \left( \tilde{Z}_0 + \sigma \tilde{U} \tilde{W}_0 ,\alpha\left( \tilde{Z}_1 + \sigma \tilde{U} \tilde{W}_1 \right)\right) }{ \left(  \tilde{Z}_0 + \sigma  \tilde{U} \tilde{W}_0,\frac{  \tilde{Z}_0 + \sigma \tilde{U} \tilde{W}_0}{n} + \beta\left(V_2 \tilde{Z}_1 + \sigma \tilde{U} \tilde{W}_1 \right)\right)}\\
& = \TV{ \widetilde{X_0},\widetilde{\outputmodel_0} }{  \widetilde{X_1},\widetilde{\outputmodel_1}}.
\]

Thus, combining all the steps concludes the proof.
 
\end{proof}

\rotationinvar*
\begin{proof}
We prove the statement for the case $j=0$, the case of $j=1$ follows similarly. By definition, we can write
\[
X_0  &\eqdist Z_1 + \sigma U Z_2,\\
\outputmodel_0  &\eqdist \alpha\left(Z_3 + \sigma U Z_4\right).
\]
\mhtext{Here $(Z_1,Z_3)\sim  \Normal(0,\id{d})^{\otimes 2},(Z_2,Z_4)\sim \Normal(0,\id{k})^{\otimes 2}$} and $\alpha = \sqrt{\frac{1}{n}+\rho^2}$. The main observation is that for every rotation matrix $R$, we have $RZ_1\eqdist Z_1$, $RZ_3\eqdist Z_3$, and $RU\eqdist U$. Therefore, the claim follows as was to be shown.
\end{proof}

\begin{lemma}\label{lem:pdf-depend-norm}
Let $(X,Y)$ be a pair of random variables in $\Reals^{d} \times \Reals^d$ with a joint density given by $f:\Reals^d \times \reals^d \to \reals$. Assume that for every rotation matrix $R \in \Reals^{d \times d}$, we have $\left(X,Y\right)\eqdist (RX,RY)$. Then, there exists a function $g: \Reals \times \Reals \times \Reals \to \Reals$ such that for every $x,y \in (\Reals^d)^2$, 
\[
f\left(x,y\right) = g\left( \norm{x},\norm{y},\inner{x}{y}\right).
\]
\end{lemma}
\begin{proof}
In the first step, we show that for every rotation matrix $R$ and $(x,y)\in (\Reals^d)^2$ we have 
$f(x,y) = f(Rx,Ry)$. The proof is as follows. Let $(A,B)\subseteq \Reals^d \times \Reals^d$. Then,
\[
\Pr\left((X,Y)\in \left(A,B\right)\right) &= \int_{(A,B)} f(x,y) \text{d}x \text{d}y\\
&\stackrel{(a)}{=} \Pr\left((RX,RY)\in \left(A,B\right)\right)\\
&= \Pr\left((X,Y)\in \left(R^{-1}A,R^{-1}B\right)\right)\\
&= \int_{(R^{-1}A,R^{-1}B)} f\left(x,y\right) \text{d}x \text{d}y \\ 
&\stackrel{(b)}{=} \int_{(A,B)} f\left(R^{-1}x,R^{-1}y\right) \text{d}x \text{d}y,
\]
where $(a)$ follows from the assumption and $(b)$ follows from the change of variable and the fact that $\det(R)=1$. Since it holds for every measurable set $(A,B)$, we have $f(Rx,Ry) = f\left(x,y\right)$.

Let $(n_1,n_2,r)\in \Reals^3$. Consider the following set $C=\{(x,y): (\norm{x},\norm{y},\inner{x}{y})=(n_1,n_2,r)\}$. In the final step, we show that $f$ is constant over the set $C$. Consider two pairs of points $(x_1,y_1)\in C$ and $(x_2,y_2)\in C$. It is immediate to see there exists a rotation matrix $R$ such that $(Rx_1,Ry_1)=(x_2,y_2)$. Therefore, $f(x_1,y_1)=f(Rx_1,Ry_1)=f(x_2,y_2)$. Thus, $f$ is constant on $C$ and as a result for every $(x,y)\in C$, we have $f(x,y)=g(n_1,n_2,r)$, as was to be shown.
\end{proof}

\suffstat*
\begin{proof}
Let  $T: \Reals^d \times \Reals^d \to (\Reals\times \Reals\times \Reals)$ where $T(x,y)=(\norm{x},\norm{y},\inner{x}{y})$. By the data processing inequality, 
\[
\TV{X_0 ,\hat \mu_0 }{X_1 ,\hat \mu_1 } &= \TV{X_0 ,\hat \mu_0 ,T(X_0 ,\hat \mu_0 )}{X_1 ,\hat \mu_1 ,T(X_1 ,\hat \mu_1 )}.
\]

In the next step, using the chain rule for TV from \cref{lem:chain-rule-tv},  
\begin{align}
&\TV{X_0 ,\hat \mu_0 ,T(X_0 ,\hat \mu_0 )}{X_1 ,\hat \mu_1 ,T(X_1 ,\hat \mu_1 )} \leq \TV{T(X_0 ,\hat \mu_0 )}{T(X_1 ,\hat \mu_1 )} \nonumber\\
&+ \EE\left[\TV{(X_0 ,\hat \mu_0 ) \Big| T\left(X_0 ,\hat \mu_0 \right)}{(X_1 ,\hat \mu_1 ) \Big| T\left(X_1 ,\hat \mu_1 \right)}\right] \nonumber.
\end{align}
We claim the last term is zero. For $j\in \{0,1\}$, conditional on the event that $T(X_j ,\hat \mu_j )=t$, the distribution of $(X_j ,\hat \mu_j )$ is uniform over the set $T^{-1}(t)$ by \cref{lem:rotation-invar} and \cref{lem:pdf-depend-norm}. The reason is that as we show in \cref{lem:pdf-depend-norm}, for $i \in \{0,1\}$, the density function of $X_i,\hat \mu_i$ only depends on the triple of $\left(\norm{X_i},\norm{\hat \mu_i},\inner{X_i}{\hat \mu_i}\right)$. 
\end{proof}

\tvdecomposezero*

\begin{proof}
We have
\[
&\TV{\norm{X_0 },\norm{\hat \mu_0 },\inner{X_0 }{\hat \mu_0 }}{\norm{X_1 },\norm{\hat \mu_1 },\inner{X_1 }{\hat \mu_1 }} \\
&\stackrel{(a)}{\leq} \TV{\norm{X_0 },\norm{\hat \mu_0 },\inner{X_0 }{\hat \mu_0 },X_0 ,UU^\top}{\norm{X_1 },\norm{\hat \mu_1 },\inner{X_1 }{\hat \mu_1 },X_1 ,UU^\top}\\
&\stackrel{(b)}{=} \TV{\norm{\hat \mu_0 },\inner{X_0 }{\hat \mu_0 },X_0 ,UU^\top}{\norm{\hat \mu_1 },\inner{X_1 }{\hat \mu_1 },X_1 ,UU^\top},
\]
where $(a)$ follows by the data processing inequality, and $(b)$ follows because the norm of a random variable is a deterministic function of it. Then,

\[
&\TV{\norm{\hat \mu_0 },\inner{X_0 }{\hat \mu_0 },X_0 ,UU^\top}{\norm{\hat \mu_1 },\inner{X_1 }{\hat \mu_1 },X_1 ,UU^\top}\\
&\leq \TV{X_0 ,UU^\top}{X_1 ,UU^\top} \\
&+ \EE\left[\TV{\norm{\hat \mu_0 },\inner{X_0 }{\hat \mu_0 }\Big| X_0 ,UU^\top}{\norm{\hat \mu_1 },\inner{X_1 }{\hat \mu_1 }\Big| X_1 ,UU^\top}\right]\\
& = \EE\left[\TV{\norm{\hat \mu_0 },\inner{X_0 }{\hat \mu_0 }\Big| X_0 ,UU^\top}{\norm{\hat \mu_1 },\inner{X_1 }{\hat \mu_1 }\Big| X_1 ,UU^\top}\right]\\
&\leq \EE\left[\TV{\inner{X_0 }{\hat \mu_0 }\Big| X_0 ,UU^\top}{\inner{X_1 }{\hat \mu_1 }\Big| X_1 ,UU^\top}\right]\\
& + \EE\left[\TV{ \norm{\outputmodel_0 } \Big| X_0 ,UU^\top,\inner{X_0 }{\hat \mu_0 }}{\norm{\outputmodel_1 }\Big| X_1 ,UU^\top,\inner{X_1 }{\hat \mu_1 }}\right],
\]
where the last line follows because $\TVinline{X_0 ,UU^\top}{X_1 ,UU^\top}=0$.
\end{proof}

\zeroinnerprod*

\begin{proof}
In the first step, we use Pinsker's inequality \citep{polyanskiy2025information} and Jensen's inequality,
\[
\EE\left[\TV{\inner{X_0 }{\hat \mu_0 }\Big| X_0 ,UU^\top}{\inner{X_1 }{\hat \mu_1 }\Big| X_1 ,UU^\top}\right] &\leq \EE\left[\sqrt{\frac{1}{2}\KL{\inner{X_0 }{\hat \mu_0 }\Big| X_0 ,UU^\top}{\inner{X_1 }{\hat \mu_1 }\Big| X_1 ,UU^\top}}\right]\\
&\leq \sqrt{\frac{1}{2}\EE\left[ \KL{\inner{X_0 }{\hat \mu_0 }\Big| X_0 ,UU^\top}{\inner{X_1 }{\hat \mu_1 }\Big| X_1 ,UU^\top} \right]}.
\]

Consider the following \emph{conditional} KL divergence, where with an abuse of notation we use $UU^\top$ as the realization of $UU^\top$.
\[
\KL{\inner{X_0 }{\hat \mu_0 }\Big| X_0 =x,UU^\top=UU^\top}{\inner{X_1 }{\hat \mu_1 }\Big| X_1 =x,UU^\top=UU^\top}.
\]
Conditional on $UU^\top$ and $x$, by the construction in \cref{def:gen-process-cov}, we can represent $\outputmodel _0$ and $\outputmodel _1$ as follows: let the subspace spanned by $UU^\top$ is given by $\mathrm{span}\{\frac{UU^\top x}{\norm{UU^\top x}},u_2,\dots,u_k\}$, $Z_0 \sim \Normal(0,\id{d}),(\xi_1,\dots,\xi_k)\sim \Normal(0,1)^{\otimes k}$ such that $(Z_0,\xi_1,\dots,\xi_k)$ are mutually independent. Then,
\[
\outputmodel _0 &\eqdist \alpha\left(Z_0 + \sigma\left(\frac{UU^\top x}{\norm{UU^\top x}}\xi_1+ \sum_{i=2}^k u_i \xi_i\right)\right),\\
\outputmodel _1 &\eqdist \frac{x}{n}+\beta \left(Z_0 + \sigma\left(\frac{UU^\top x}{\norm{UU^\top x}}\xi_1+ \sum_{i=2}^k u_i \xi_i\right)\right),
\]
We have $\inner{x}{u_i}=0$ for $i \in \{2,\dots,k\}$. Also, $\inner{x}{UU^\top x}=x^\top U U^\top UU^\top x = \norm{UU^\top x}^2$. Therefore,
\[
&\inner{x}{\outputmodel _0} \eqdist \alpha\left(\inner{x}{Z_0} + \sigma \xi_1 \norm{UU^\top x}\right)\sim \Normal\left(0,\alpha^2\left(\norm{x}^2 + \sigma^2 \norm{UU^\top x}^2\right)\right),\\
&\inner{x}{\outputmodel _1} \eqdist \frac{\norm{x}^2}{n}+ \beta\left(\inner{x}{Z_0} + \sigma \xi_1 \norm{UU^\top x}\right)\sim \Normal\left(\frac{\norm{x}^2}{n},\beta^2\left(\norm{x}^2 + \sigma^2 \norm{UU^\top x}^2\right)\right).\\
\]
Let $\Sigma = \id{d} +\sigma^2 UU^\top$ and notice that $\norm{x}^2 + \sigma^2 \norm{UU^\top x}^2 = x^\top \Sigma x$. By the closed-form expression for the KL divergence between two uni-variate Gaussian distributions, we have  
\[
&\KL{\inner{x}{\outputmodel _0} \Big| X_0 =x,UU^\top=UU^\top }{\inner{x}{\outputmodel _1}  \Big| X_1 =x,UU^\top=UU^\top}\\
&\leq -\frac{1}{2(n+n^2 \rho^2)} + \frac{1}{2(n-1+n^2 \rho^2)} + \frac{\norm{x}^4}{2(n-1 + n^2 \rho^2) x^\top \Sigma x}\\
& \leq \frac{1}{2(n-1 + n^2 \rho^2)^2} + \frac{\norm{x}^4}{2(n-1 + n^2 \rho^2) x^\top \Sigma x}.
\]
Then,
\begin{align}
&\EE\left[ \KL{\inner{X_0 }{\hat \mu_0 }\Big| X_0 ,UU^\top}{\inner{X_1 }{\hat \mu_1 }\Big| X_1 ,UU^\top} \right] \nonumber \\
&\leq \frac{1}{2(n-1 + n^2 \rho^2)^2} 
+\EE_{UU^\top \sim \pi, X\sim \Normal(0,\id{d}+\sigma^2 UU^\top)}\left[\frac{\norm{X}^4}{2(n-1 + n^2 \rho^2) X^\top \Sigma X}\right]. \label{eq:kl-inner-prod-uni-simple}
\end{align}
By definition, we can represent $X = \Sigma^{1/2}Z$ where $Z\sim \Normal(0,\id{d})$. Then, by the rotational invariance of Gaussian random variables, we have
\[
\EE_{UU^\top \sim \pi, X\sim \Normal(0,\id{d}+\sigma^2 UU^\top)}\left[\frac{\norm{X}^4}{X^\top \Sigma X}\right] &= \EE_{Z\sim \Normal(0,\id{d})}\left[\frac{\left(Z^\top \Sigma Z\right)^2}{Z^\top \Sigma^2 Z}\right],
\]
where using the rotation invariance of Gaussian random variable, we can assume $$\Sigma = \text{diag}\left(\underbrace{1+\sigma^2,\dots,1+\sigma^2}_{\text{$k$ times}},\underbrace{1,\dots,1}_{\text{$d-k$ times}}\right)$$. For notational convenience, let $Z = (Z_1,Z_2)$ where $Z_1 \sim \Normal(0,\id{k})$, $Z_2 \sim \Normal(0,\id{d-k})$, and $Z_1\indep Z_2$. Using this and by $(a+b)^2\leq 2a^2 + 2b^2$ for every $a,b$, we have 
\begin{equation} \label{eq:scaled-norm-norm2}
\begin{aligned} 
\left(Z^\top \Sigma Z\right)^2 &\leq 2\left(1+\sigma^2\right)^2 \norm{Z_1}^4 + 2\norm{Z_2}^4,\\
Z^\top \Sigma^2 Z &= \left(1+\sigma^2\right)^2 \norm{Z_1}^2 + \norm{Z_2}^2.
\end{aligned}
\end{equation}

Let $\gamma = \left(1 + \sigma^2\right)^2$. Then, from \cref{eq:scaled-norm-norm2}, we have
\begin{align}
\EE_{Z\sim \Normal(0,\id{d})}\left[\frac{\left(Z^\top \Sigma Z\right)^2}{Z^\top \Sigma^2 Z}\right] &= \EE\left[ \int_{s=0}^{\infty} \left(Z^\top \Sigma Z\right)^2 \exp\left(-sZ^\top \Sigma^2 Z\right) \text{d}s\right]\nonumber\\
&\leq \int_{s=0}^\infty 2 \gamma \EE\left[\norm{Z_1}^4 \exp\left(-s\left(\gamma \norm{Z_1}^2 + \norm{Z_2}^2\right)\right)\right] \text{d}s \nonumber\\
&+ \int_{s=0}^\infty 2 \EE\left[\norm{Z_2}^4 \exp\left(-s\left(\gamma \norm{Z_1}^2 + \norm{Z_2}^2\right)\right)\right] \text{d}s\label{eq:innerprod-kl}.
\end{align}
We start with the first term in \cref{eq:innerprod-kl}.
\begin{equation}\label{eq:innerprod-kl-first-term}
\begin{aligned}
&\int_{s=0}^\infty 2 \gamma \EE\left[\norm{Z_1}^4 \exp\left(-s\left(\gamma \norm{Z_1}^2 + \norm{Z_2}^2\right)\right)\right] \text{d}s \\
&\leq \int_{s=0}^\infty 2 \gamma \EE\left[\norm{Z_1}^4 \exp\left(-s \gamma \norm{Z_1}^2 \right)\right] \text{d}s\\
&\stackrel{(a)}{=} \int_{s=0}^\infty \frac{8\gamma}{(1+2s\gamma)^{\frac{k}{2}+2}} \frac{\Gamma\left(\frac{k}{2}+2\right)}{\Gamma\left(\frac{k}{2}\right)} \text{d}s\\
&\stackrel{(b)}{\leq} \int_{s=0}^\infty \frac{8\gamma}{\left(1+2s\gamma\right)^{\frac{k}{2}+3}} \left(k+3\right)^2 \text{d}s\\
& \stackrel{(c)}{=} \frac{4 (k+3)^2}{k+4}\\
& \leq 4 (k+4),
\end{aligned}
\end{equation}
where $(a)$ follows from \cref{lem:guassian-norm-four-power}, $(b)$ follows from \cref{lem:gamma-ub}, and $(c)$ follows from \cref{lem:int-1}. For the second term in \cref{eq:innerprod-kl}, 
\begin{equation} \label{eq:innerprod-kl-second-term}
\begin{aligned}
&\int_{s=0}^\infty 2 \EE\left[\norm{Z_2}^4 \exp\left(-s\left(\gamma \norm{Z_1}^2 + \norm{Z_2}^2\right)\right)\right] \text{d}s \\
&= \int_{s=0}^\infty 2 \EE\left[\norm{Z_2}^4 \exp\left(-s \norm{Z_2}^2\right)\right] \cdot \EE\left[\exp\left(-s\gamma \norm{Z_1}^2\right)\right] \text{d}s\\
& \stackrel{(a)}{=} \int_{s=0}^\infty \frac{4 (d-k+3)^2}{(1+2s)^{(d-k)/2 + 2}}  \cdot \left(\frac{1}{(1+2s\gamma)^{k/2}}\right) \text{d}s\\
&\leq 4 \left(d-k+3\right)^2 \int_{s=0}^\infty \frac{1}{(1+2s\gamma)^{k/2}} \text{d}s\\
& \stackrel{(b)}{=}4 \left(d-k+3\right)^2 \frac{1}{\gamma\left(k-2\right)},
\end{aligned}
\end{equation}
where $(a)$ follows \cref{lem:guassian-norm-four-power,lem:gamma-ub} and the known bounds on the moment generating function of the chi-squared distribution
$\EE\left[\exp\left(-s\gamma \norm{Z_1}^2\right)\right] = \frac{1}{(1+2s\gamma)^{k/2}} $. Then, $(b)$ follows from \cref{lem:int-1}.

The statement follows by using \cref{eq:innerprod-kl-first-term} and \cref{eq:innerprod-kl-second-term} to show an upper bound on \cref{eq:innerprod-kl}. Then, by plugging this bound into \cref{eq:kl-inner-prod-uni-simple} the statement follows.
\end{proof} 

\zernormterm*

\begin{proof}
All conditional distributions below are understood as regular conditional distributions. We fix $(x,UU^\top,r)$ in a probability-one set on which these conditional distributions are defined and $\norm{x}>0$ and $\norm{UU^\top x}>0$. We use the strengthened assumption $d>k>\max\{8c^2,9\}$.

For a fixed $x$, $UU^\top$, and $r$, consider the following total variation distance
\[
\TV{ \norm{\outputmodel_0 } \Big| X_0 =x,UU^\top=UU^\top,\inner{X_0 }{\hat \mu_0 }=r}{\norm{\outputmodel_1 }\Big| X_1 =x,UU^\top=UU^\top,\inner{X_1 }{\hat \mu_1 }=r}.
\]
Define $\Pi_x$ and $\Pi_{x^\perp}$ as the projection operator on the subspace spanned by $x$ and $x^\perp$, i.e., the orthogonal subspace to $x$. For $j \in \{0,1\}$, we can write
\[
\begin{aligned}
\outputmodel_j  &= \Pi_x\left(\outputmodel_j \right) + \Pi_{x^\perp}\left(\outputmodel_j \right)\\
&= \inner{\outputmodel_j }{\frac{x}{\norm{x}}}\frac{x}{\norm{x}} + \Pi_{x^\perp}\left(\outputmodel_j \right)\\
&= r\frac{x}{\norm{x}^2} +  \Pi_{x^\perp}\left(\outputmodel_j \right),
\end{aligned}
\]
where the last step follows because we conditioned on $\inner{x}{\hat \mu_j}=r$. For notational convenience, we use
\begin{equation} \label{eq:events01}
\begin{aligned}
E_0 &\triangleq \{X_0 =x,UU^\top=UU^\top,\inner{X_0 }{\hat \mu_0 }=r\},\\
E_1 &\triangleq \{X_1 =x,UU^\top=UU^\top,\inner{X_1 }{\hat \mu_1 }=r\}.
\end{aligned}
\end{equation}
Therefore, by invariance to one-to-one mapping  of the total variation,  we have
\begin{equation} \label{eq:tv-norm-step0}
\begin{aligned}
\TV{ \norm{\outputmodel_0 } \Big| E_0}{\norm{\outputmodel_1 }\Big| E_1}&=  \TV{ \norm{\Pi_{x^\perp}\left(\outputmodel_0 \right)}^2 + \norm{\Pi_{x}\left(\outputmodel_0 \right)}^2 \Big| E_0}{\norm{\Pi_{x^\perp}\left(\outputmodel_1 \right)}^2 + \norm{\Pi_{x}\left(\outputmodel_1 \right)}^2\Big| E_1}\\
&=\TV{ \norm{\Pi_{x^\perp}\left(\outputmodel_0 \right)}^2 \Big| E_0}{\norm{\Pi_{x^\perp}\left(\outputmodel_1 \right)}^2\Big| E_1}.
\end{aligned}
\end{equation}
Conditioned on $UU^\top$ and $x$, by the construction in \cref{def:gen-process-cov}, we can represent $\outputmodel _0$ and $\outputmodel _1$ as follows: Let the subspace $\im{UU^\top}$ be  $\mathrm{span}\left\{u_1\triangleq\frac{UU^\top x}{\norm{UU^\top x}},u_2,\dots,u_k\right\}$ where $\{u_1,\dots,u_k\}$ forms an  orthonormal basis and $u_i \perp x$ for $i \in \{2,\dots,k\}$. Also, $Z_0 \sim \Normal(0,\id{d}),(\xi_{1,0},\xi_{1,1},\dots,\xi_k)\sim \Normal(0,1)^{k+1}$ such that $(Z_0,\xi_{1,0},\xi_{1,1},\dots,\xi_k)$ are mutually independent. Also, $\alpha_0 \triangleq \sqrt{\frac{1}{n}+\rho^2}$ and $\alpha_1 \triangleq \sqrt{\frac{n-1}{n^2}+\rho^2}$. Then, we can represent $\hat \mu_0$ and $\hat \mu_1$ as follows
\[
\begin{aligned}
\outputmodel _0 &\eqdist \alpha_0\left(Z_0 + \sigma\left(\frac{UU^\top x}{\norm{UU^\top x}}\xi_{1,0}+ \sum_{i=2}^k u_i \xi_i\right)\right),\\
\outputmodel _1 &\eqdist \frac{x}{n}+\alpha_1 \left(Z_0 + \sigma\left(\frac{UU^\top x}{\norm{UU^\top x}}\xi_{1,1}+ \sum_{i=2}^k u_i \xi_i\right)\right).
\end{aligned}
\]

Recall the specific basis vectors we pick for $\im{UU^\top}$ and its properties. Then, we decompose $\outputmodel _0$ and $\outputmodel _1$ on the subspace spanned by $x$ and the subspace orthogonal to $x$ as follows:

\[
\begin{cases}
\Pi_{x^\perp} \outputmodel _0 &= \alpha_0\left(\Pi_{x^\perp}\left(Z_0\right) + \sigma\left(\Pi_{x^\perp}\left(\frac{UU^\top x}{\norm{UU^\top x}}\right)\xi_{1,0}+ \sum_{i=2}^k u_i \xi_i\right)\right),\\
\Pi_{x^\perp} \outputmodel _1  &=  \alpha_1 \left(\Pi_{x^\perp}\left(Z_0\right) + \sigma\left(\Pi_{x^\perp}\left(\frac{UU^\top x}{\norm{UU^\top x}}\right)\xi_{1,1}+ \sum_{i=2}^k u_i \xi_i\right)\right).
\end{cases}
\]
\begin{equation} \label{eq:inner-prod-representation}
\begin{cases}
  \inner{x}{\outputmodel _0} &= \alpha_0\left( \inner{x}{Z_0} + \sigma \norm{UU^\top x}\xi_{1,0} \right), \\
  \inner{x}{\outputmodel _1} &= \frac{\norm{x}^2}{n} + \alpha_1\left( \inner{x}{Z_0} + \sigma \norm{UU^\top x}\xi_{1,1} \right).
\end{cases}
\end{equation}

We make the following observation. For $j\in\{0,1\}$, the random vector $(\Pi_{x^\perp}(Z_0),\xi_2,\dots,\xi_k)$ is independent of $(\inner{x}{Z_0},\xi_{1,j})$. Therefore, conditioning on $\inner{x}{\outputmodel_j}=r$ leaves the joint distribution of $(\Pi_{x^\perp}(Z_0),\xi_2,\dots,\xi_k)$ unchanged and only changes the distribution of $\xi_{1,j}$ among the variables appearing in the projected representation.

For notational convenience, let 
$v \triangleq \Pi_{x^\perp}\left(\frac{UU^\top x}{\norm{UU^\top x}}\right)$ and observe that $\norm{v}\leq 1$. Recall we assume that the subspace $\im{UU^\top}$ is given by $\mathrm{span}\left\{u_1\triangleq\frac{UU^\top x}{\norm{UU^\top x}},u_2,\dots,u_k\right\}$ where $\{u_1,\dots,u_k\}$ forms an orthonormal basis and $u_i \perp x$ for $i \in \{2,\dots,k\}$. Notice that $\inner{v}{u_j}$ for $j \in \{2,\cdots,k\}$ is zero. It is because $\inner{v}{u_j} = \inner{\Pi_{x^\perp} u_1}{u_j} = \inner{ u_1}{\Pi_{x^\perp}u_j} = \inner{ u_1}{u_j} =0$ where we used that 1) $\Pi_{x^\perp}$ is a symmetric matrix, 2) $\Pi_{x^\perp}u_j = u_j$ for $j \in \{2,\dots,k\}$.

Let $\bar v=v/\norm{v}$ when $\norm{v}>0$. When $\norm{v}=0$, let $\bar v$ be any unit vector in $x^\perp$ orthogonal to $u_2,\dots,u_k$. In either case, $v=\norm{v}\bar v$ and $\{\bar v,u_2,\dots,u_k\}$ is an orthonormal set.

Using this observation, we can represent $\Pi_{x^\perp}=HH^\top$, where $H\in\Reals^{d\times(d-1)}$ is given by
\begin{equation*}
H = \begin{bmatrix}
\bar v & u_2 & \cdots & u_k & \tilde{u}_{k+1} & \cdots & \tilde{u}_{d-1}
\end{bmatrix}.
\end{equation*}
Here, $\tilde{u}_{k+1},\dots,\tilde{u}_{d-1}$ complete $\{\bar v,u_2,\dots,u_k\}$ to an orthonormal basis for $\im{\Pi_{x^\perp}}$. Therefore,
\begin{equation}
\label{eq:change-basis-z}
\Pi_{x^\perp}\left(Z_0\right) = \bar v \inner{\bar v}{Z_0}+ \sum_{j=2}^{k} u_j \inner{u_j}{Z_0} + \sum_{j=k+1}^{d-1} \tilde{u}_j\inner{\tilde{u}_j}{Z_0}.
\end{equation}

Note that $\left( \inner{\bar v}{Z_0},\inner{u_2}{Z_0},\dots,\inner{u_k}{Z_0},\inner{\tilde{u}_{k+1}}{Z_0},\dots,\inner{\tilde{u}_{d-1}}{Z_0}\right)\sim \Normal(0,1)^{\otimes (d-1)}$ since $\{\bar v,u_2,\dots,u_k,\tilde u_{k+1},\dots,\tilde u_{d-1}\}$ is an orthonormal basis for $x^\perp$.

The observation in \cref{eq:change-basis-z} implies the following. The conditioning event depends only on $(\inner{x}{Z_0},\xi_{1,j})$, and all the other Gaussian coordinates in the orthogonal projection retain their original joint distribution. Thus, by rotational invariance, for $j\in\{0,1\}$,
\[
\norm{\Pi_{x^\perp} \outputmodel _j}^2 \eqdist \alpha_j^2 \left((Z + \sigma \norm{v}\xi_{1,j})^2 + (1+\sigma^2) Y + W\right),
\]
where $Z \sim \Normal(0,1)$, $Y \sim \chi^2(k-1)$, $W\sim \chi^2(d-k-1)$, and $(Z,\xi_{1,0},\xi_{1,1},Y,W)$ are mutually independent. 

Then,  we can write
\begin{align}
&\TV{\norm{\Pi_{x^\perp} \outputmodel _0}^2 \Big| E_0}{\norm{\Pi_{x^\perp} \outputmodel _1}^2 \Big| E_1} \nonumber\\
&= \TV{\alpha_0^2 \left((Z + \sigma \norm{v}\xi_{1,0})^2 + (1+\sigma^2) Y + W\right) \Big| E_0}{\alpha_1^2 \left((Z + \sigma \norm{v}\xi_{1,1})^2 + (1+\sigma^2) Y + W\right)\Big| E_1} \nonumber\\
&= \TV{\alpha_0^2 \left((Z + \sigma \norm{v}\xi_{1,0})^2 + (1+\sigma^2) Y_1 + (1+\sigma^2) Y_2  + W\right) \Big| E_0}{\alpha_1^2 \left((Z + \sigma \norm{v}\xi_{1,1})^2 + (1+\sigma^2) Y_1 + (1+\sigma^2) Y_2 + W\right)\Big| E_1}, \nonumber
\end{align}

where in the last line we decompose $Y=Y_1+Y_2$, where $Y_1\sim\chi^2((k-1)/2)$ and $Y_2\sim\chi^2((k-1)/2)$ are independent. Here, $\chi^2(\nu)$ denotes $\operatorname{Gamma}(\nu/2,2)$ for every real $\nu>0$. Since $k>9$, we have $(k-1)/2>4$, as required by \cref{lem:tv-chi-square}. Then,
\begin{align}
&\TV{\norm{\Pi_{x^\perp} \outputmodel _0}^2 \Big| E_0}{\norm{\Pi_{x^\perp} \outputmodel _1}^2 \Big| E_1}\nonumber\\
&\stackrel{(a)}{\leq} \TV{\alpha_0^2\left( (Z + \sigma \norm{v}\xi_{1,0})^2 + (1+\sigma^2) Y_1\right) \Big| E_0}{\alpha_1^2 \left((Z + \sigma \norm{v}\xi_{1,1})^2 + (1+\sigma^2) Y_1\right) \Big| E_1} \nonumber\\
&+ \TV{\alpha_0^2 \left((1+\sigma^2)Y_2+ W\right)}{\alpha_1^2 \left((1+\sigma^2)Y_2+ W\right)}, \label{eq:tv-norm-proj-decompose}
\end{align}

where Step $(a)$ follows from tensorization of total variation for product measures, the triangle inequality, and the data processing inequality under addition. Also, notice that we can drop the conditioning in the second term as $(Y_2,W)$ is independent of the events $E_0$ and $E_1$.

Next, we provide an upper bound
for \cref{eq:tv-norm-proj-decompose} as follows. Let $\alpha_0^2/\alpha_1^2 = 1 + \frac{1}{n-1+n^2\rho^2}\triangleq 1+\gamma$. Then, by the invariance of the TV distance to one-to-one mappings, 
\begin{equation*} 
\begin{aligned} 
\TV{\alpha_0^2 \left((1+\sigma^2)Y_2+ W\right)}{\alpha_1^2 \left((1+\sigma^2)Y_2+ W\right)} &= \TV{\left(1+\gamma\right) Y_2 + \left(1+\gamma\right)\frac{W}{1+\sigma^2}}{Y_2 + \frac{W}{1+\sigma^2}} 
\end{aligned}
\end{equation*}

Now, we are in the position to use \cref{lem:tv-chi-square} to provide an upper bound on this term. First of all notice that we have
\[
\begin{aligned}
\EE\left[ \left(\frac{(1+\gamma)W}{1+\sigma^2} - \frac{W}{1+\sigma^2}\right)^2 \right] &= \frac{\gamma^2}{(1+\sigma^2)^2}\EE[W^2]\\
&= \frac{\gamma^2}{(1+\sigma^2)^2}\left(2(d-k-1)+(d-k-1)^2\right)\\
&\leq \frac{\gamma^2}{(1+\sigma^2)^2}\left(2d+d^2\right)
\end{aligned}
\]
where in the last line we used $W\sim \chi^2(d-k-1)$. Thus, from \cref{lem:tv-chi-square}, we have
\begin{align}
&\TV{\left(1+\gamma\right) Y_2 + \left(1+\gamma\right)\frac{W}{1+\sigma^2}}{Y_2 + \frac{W}{1+\sigma^2}} \nonumber\\
&\leq \sqrt{\frac{\gamma^2}{1+\sigma^2}\frac{\sqrt{2d+d^2}}{4} + \frac{\gamma^2 k}{8} + \frac{\gamma^2}{(1+\sigma^2)^2}\frac{2d+d^2}{2k-18}} \nonumber\\
&\leq \sqrt{\frac{1}{(n-1+n^2\rho^2)^2}\frac{\sqrt{2d+d^2}}{4(1+\sigma^2)} + \frac{k}{8(n-1+n^2\rho^2)^2} + \frac{1}{(n-1+n^2\rho^2)^2}\frac{2d+d^2}{(2k-18)(1+\sigma^2)^2}}.
\label{eq:tv-norm-t2-final}
\end{align}

In the next step, we provide an upper bound on the first term in \cref{eq:tv-norm-proj-decompose}. Consider \cref{eq:events01}. We can represent $E_0$ and $E_1$ using the representation in  \cref{eq:inner-prod-representation} as follows:

\[
\begin{aligned}
E_0 &= \{X_0 =x,UU^\top=UU^\top,\inner{X_0 }{\hat \mu_0 }=r\}\\
 &= \{X_0 =x,UU^\top=UU^\top,\inner{x}{Z_0}+\sigma\norm{UU^\top x}\xi_{1,0}=\frac{r}{\alpha_0}\},\\
E_1 &= \{X_1 =x,UU^\top=UU^\top,\inner{X_1 }{\hat \mu_1 }=r\}\\
 &= \{X_1 =x,UU^\top=UU^\top,\inner{x}{Z_0}+\sigma\norm{UU^\top x}\xi_{1,1}= \frac{r -\frac{\norm{x}^2}{n}}{\alpha_1}\}.
\end{aligned}
\]

From \cref{lem:cond-inner}, we can represent the conditional distribution of $\xi_{1,0}$ and $\xi_{1,1}$ on $E_0$ and $E_1$ as follows: Let $\tilde{Z}\sim \Normal(0,1)$ such that it is independent of other random variables. Then,
\[
\begin{aligned}
\xi_{1,0} \Big| E_0  \eqdist \mu_0 + h \tilde{Z},\\
\xi_{1,1} \Big| E_1    \eqdist \mu_1 + h \tilde{Z},
\end{aligned}
\]
We use the same $\tilde Z$ to couple these two conditional distributions, where $\tilde Z$ is independent of $Z,Y_1,Y_2$, and $W$.
where 
\begin{equation} \label{eq:cond-dist-params}
\begin{aligned}
\mu_0 &\triangleq \frac{r}{\alpha_0} \left(\frac{\sigma \norm{UU^\top x}}{\sigma^2 \norm{UU^\top x}^2 + \norm{x}^2}\right),\\
\mu_1 &\triangleq\frac{r - \frac{\norm{x}^2}{n}}{\alpha_1} \left(\frac{\sigma \norm{UU^\top x}}{\sigma^2 \norm{UU^\top x}^2 + \norm{x}^2}\right),\\
h &\triangleq \sqrt{\frac{\norm{x}^2}{\norm{x}^2+\sigma^2 \norm{UU^\top x}^2}}\leq 1.
\end{aligned}
\end{equation}

Using the closed-form expression for the conditional distribution $\xi_{1,0}$ and $\xi_{1,1}$ on $E_0$ and $E_1$, we can write
\begin{align}
&\TV{\alpha_0^2\left( \left(Z + \sigma \norm{v}\xi_{1,0}\right)^2 + (1+\sigma^2) Y_1\right) \Big| E_0}{\alpha_1^2 \left((Z + \sigma \norm{v}\xi_{1,1})^2 + (1+\sigma^2) Y_1\right) \Big| E_1}= \nonumber\\
& \TV{\alpha_0^2\left( \left(Z + \sigma \norm{v} \left(\mu_0 + h \tilde{Z}\right)\right)^2 + (1+\sigma^2) Y_1\right) \Big| E_0}{\alpha_1^2 \left(\left(Z + \sigma \norm{v}\left(\mu_1 + h  
\tilde{Z}\right)\right)^2 + (1+\sigma^2) Y_1\right) \Big| E_1}. \label{eq:tv-difficult-term}
\end{align}

Let $c>1$ be a constant that will be determined later. Define the following \emph{good} event:
\begin{equation} \label{eq:good-event}
\begin{aligned}
\mathcal{G}= \Big\{(x,UU^\top,r):\;& |r| \leq c \alpha_0\sqrt{\sigma^2 \norm{UU^\top x}^2 + \norm{x}^2}\\
&\cap \big| \norm{UU^\top x}^2 - (1+\sigma^2)k\big|\leq c(1+\sigma^2)\sqrt{2k} \\
&\cap \norm{(\id{d}-UU^\top) x}^2 \leq (d-k)+c\sqrt{2(d-k)}
\Big\}.
\end{aligned}
\end{equation}
Under the event $\mathcal{G}$, we have the following deterministic bound on $\mu_0$ from \cref{eq:cond-dist-params}:
\begin{equation} \label{eq:mu1-bound}
\begin{aligned}
|\mu_0| &= \frac{|r|}{\alpha_0} \left(\frac{\sigma \norm{UU^\top x}}{\sigma^2 \norm{UU^\top x}^2 + \norm{x}^2}\right) \\
&\stackrel{(a)}{\leq } c \left(\frac{\sigma \norm{UU^\top x}}{\sqrt{\sigma^2 \norm{UU^\top x}^2 + \norm{x}^2}}\right) \\
&\stackrel{(b)}{\leq } c,
\end{aligned}
\end{equation}
where $(a)$ follows from the assumption that under the event $\mathcal{G}$, we have $|r|\leq c \alpha_0\sqrt{\sigma^2 \norm{UU^\top x}^2 + \norm{x}^2}$. $(b)$ follows by the fact that the term inside the parenthesis is less than one.

Under the event $\mathcal{G}$, we have the following deterministic bound on $\mu_1$ from \cref{eq:cond-dist-params}:
\[
\begin{aligned}
|\mu_1| &\leq \frac{|r|}{\alpha_1} \left(\frac{\sigma \norm{UU^\top x}}{\sigma^2 \norm{UU^\top x}^2 + \norm{x}^2}\right) + \frac{\norm{x}^2}{n\alpha_1}  \left(\frac{\sigma \norm{UU^\top x}}{\sigma^2 \norm{UU^\top x}^2 + \norm{x}^2}\right) \\
&\stackrel{(a)}{\leq } \frac{c\alpha_0}{\alpha_1}  \left(\frac{\sigma \norm{UU^\top x}}{\sqrt{\sigma^2 \norm{UU^\top x}^2 + \norm{x}^2}}\right) + \frac{\norm{x}^2}{n \alpha_1}  \left(\frac{\sigma \norm{UU^\top x}}{\sigma^2 \norm{UU^\top x}^2 + \norm{x}^2}\right)\\
&\leq \frac{c\alpha_0}{\alpha_1}   + \frac{\norm{x}^2}{n \alpha_1}  \left(\frac{\sigma \norm{UU^\top x}}{\sigma^2 \norm{UU^\top x}^2 + \norm{x}^2}\right)\\
&\stackrel{(b)}{\leq } 2c + \frac{1}{n \alpha_1} \frac{\norm{x}^2}{\sigma \norm{UU^\top x}} 
\end{aligned}
\]
where $(a)$ follows from the assumption that under the event $\mathcal{G}$, we have $|r|\leq c \alpha_0\sqrt{\sigma^2 \norm{UU^\top x}^2 + \norm{x}^2}$, $(b)$ follows from $\alpha_0/\alpha_1\leq 2$ for $n\geq 2$ and simple algebraic manipulations. Since $k>8c^2$, the event $\mathcal{G}$ implies

\begin{equation}\label{eq:mu2-bound}
\begin{aligned}
|\mu_1|
&\stackrel{(c)}{\leq}2c+\frac{1}{n\alpha_1}
\frac{2\sigma^2k+2d}{\sigma^2\sqrt{k/2}}\\
&=2c+\frac{2\sqrt{2}}{\sqrt{n-1+n^2\rho^2}}
\left(\sqrt{k}+\frac{d}{\sigma^2\sqrt{k}}\right)\\
&\stackrel{(d)}{\leq}2c+\frac{\sqrt{32k}}{\sqrt{n-1+n^2\rho^2}}\\
&\stackrel{(e)}{\leq}3c.
\end{aligned}
\end{equation}

where Step $(c)$ follows by using the assumption that under the event $\mathcal{G}$, we have a deterministic upper and lower bounds on the norms, and $(d)$ follows from given that $d>k\geq 4c^2$, we have $k+c\sqrt{k}\leq 2k$, $d+c\sqrt{d}\leq 2d$, and $k-c\sqrt{k}\geq k/2$. Step $(e)$ follows from assuming $d\leq \sigma^2 k$, and Step $(f)$ by $32k \leq c^2 \left(n-1+n^2 \rho^2\right)$.

To control the total variation term in \cref{eq:tv-difficult-term}, we aim to use \cref{lem:tv-chi-square}. Consider,
\[
\begin{aligned}
&\TV{\alpha_0^2\left( \left(Z + \sigma \norm{v} \left(\mu_0 + h \tilde{Z}\right)\right)^2 + (1+\sigma^2) Y_1\right) \Bigg| E_0}{\alpha_1^2 \left(\left(Z + \sigma \norm{v}\left(\mu_1 + h  
\tilde{Z}\right)\right)^2 + (1+\sigma^2) Y_1\right) \Bigg| E_1}\\
&=\TV{\left( \underbrace{\frac{\alpha_0^2}{\alpha_1^2}\frac{\left(Z + \sigma \norm{v} \left(\mu_0 + h \tilde{Z}\right)\right)^2}{(1+\sigma^2)}}_{R_0} +  \frac{\alpha_0^2}{\alpha_1^2}Y_1\right) \Bigg| E_0}{\left(\underbrace{\frac{\left(Z + \sigma \norm{v}\left(\mu_1 + h  
\tilde{Z}\right)\right)^2}{(1+\sigma^2)}}_{R_1} +  Y_1\right) \Big| E_1}\\
&=\TV{  R_0 +  \frac{\alpha_0^2}{\alpha_1^2}Y_1  \Big| E_0}{ R_1 +  Y_1  \Bigg| E_1}\\
&=\TV{  R_0 +  (1+\gamma) Y_1  \Big| E_0}{ R_1 +  Y_1  \Big| E_1},
\end{aligned}
\]
where  $\alpha_0^2/\alpha_1^2 = 1 + \frac{1}{n-1+n^2\rho^2}\triangleq 1+\gamma$.  Also, notice that $Y_1$ is independent of $E_0$ (and $E_1$).

If $(x,UU^\top,r)\in \mathcal{G}$, where $\mathcal{G}$ is defined in \cref{eq:good-event}, we have an upper bound on $|\mu_0|$ and $|\mu_1|$ from \cref{eq:mu1-bound} and \cref{eq:mu2-bound}. 
Assume $(x,UU^\top,r)\in \mathcal{G}$. Then,
\begin{equation} \label{eq:r0-expectaion}
\begin{aligned}
\EE\left[R^2_1 \Big| E_1 \right] &= \EE\left[\left(\frac{\left(Z + \sigma \norm{v}\left(\mu_1 + h  
\tilde{Z}\right)\right)^2}{(1+\sigma^2)}\right)^2 \Bigg| X_1 =x,UU^\top=UU^\top,\inner{X_1 }{\hat \mu_1 }=r \right]\\
&= \EE\left[\left(\frac{Z + \sigma \norm{v}\left(\mu_1 + h  
\tilde{Z}\right)}{\sqrt{1+\sigma^2}}\right)^4 \Bigg| X_1 =x,UU^\top=UU^\top,\inner{X_1 }{\hat \mu_1 }=r \right]\\
&\stackrel{(a)}{\leq}
\frac{24+8\sigma^4\norm{v}^4\EE\left[\left(\mu_1+h\tilde Z\right)^4\right]}
{(1+\sigma^2)^2}\\
&\stackrel{(b)}{\leq}
\frac{24+8\sigma^4\left(81c^4+54c^2+3\right)}
{(1+\sigma^2)^2}\\
&\leq1128c^4.
\end{aligned}
\end{equation}
Here, Step $(a)$ follows from $(v+w)^4\leq 8v^4+8w^4$, and Step $(b)$ follows from $\norm{v}\leq1$, $h\leq1$, $|\mu_1|\leq3c$, and the exact identity $\EE[(\mu+h\tilde Z)^4]=\mu^4+6\mu^2h^2+3h^4$. The last step uses $c>1$ and $\sigma^4/(1+\sigma^2)^2\leq1$.

Similarly, assume $(x,UU^\top,r)\in \mathcal{G}$. Then,
\begin{equation} \label{eq:r1-expectaion}
\begin{aligned}
\EE\left[R^2_0 \Big| E_0 \right] &=\EE\left[\left(\frac{\alpha_0^2}{\alpha_1^2}\frac{\left(Z + \sigma \norm{v} \left(\mu_0 + h \tilde{Z}\right)\right)^2}{(1+\sigma^2)}\right)^2 \Bigg| X_0 =x,UU^\top=UU^\top,\inner{X_0 }{\hat \mu_0 }=r\right] \\
&\stackrel{(a)}{\leq} 4 \EE\left[\left(\frac{\left(Z + \sigma \norm{v} \left(\mu_0 + h \tilde{Z}\right)\right)}{\sqrt{1+\sigma^2}}\right)^4 \Bigg| X_0 =x,UU^\top=UU^\top,\inner{X_0 }{\hat \mu_0 }=r\right]\\
&\stackrel{(b)}{\leq} 4 \EE\left[\frac{\left(8Z^4 + 8\sigma^4 \norm{v}^4 \left(8\mu^4_0 + 8h^4 \tilde{Z}^4\right)\right)}{\left(\sqrt{1+\sigma^2}\right)^4} \Bigg| X_0 =x,UU^\top=UU^\top,\inner{X_0 }{\hat \mu_0 }=r\right]\\
&\stackrel{(c)}{\leq} 4 \EE\left[\frac{\left(24 + 8\sigma^4  \left(8c^4 + 24\right)\right)}{1+\sigma^4} \Bigg| X_0 =x,UU^\top=UU^\top,\inner{X_0 }{\hat \mu_0 }=r\right)\\
&\leq 256\left(c^4+3\right).
\end{aligned}
\end{equation}
Here, Step $(a)$ follows because $\alpha_0^2/\alpha_1^2\leq2$, Step $(b)$ follows from $(v+w)^4\leq8v^4+8w^4$, and Step $(c)$ follows from $\norm{v}\leq1$, $h\leq1$, $|\mu_0|\leq c$, and $\EE[(\mu+h\tilde Z)^4]=\mu^4+6\mu^2h^2+3h^4$. The last step uses $c>1$ and $\sigma^4/(1+\sigma^2)^2\leq1$.

With a slight abuse of notation, let $(\tilde R_0,\tilde R_1)$ be a coupling of the conditional laws $R_0\mid E_0$ and $R_1\mid E_1$ obtained by using the same independent $Z$ and $\tilde Z$ in the two displayed representations. Since $\tilde R_0$ and $\tilde R_1$ are nonnegative, \cref{eq:r0-expectaion,eq:r1-expectaion} gives, for $(x,UU^\top,r)\in\mathcal G$,
\[
\begin{aligned}
\EE\left[|\tilde R_0-\tilde R_1|^2\right]
&\leq\EE\left[\tilde R_0^2\right]+\EE\left[\tilde R_1^2\right]\\
&\leq1544c^4\leq2048c^4.
\end{aligned}
\]

Let $N\triangleq n-1+n^2\rho^2$, $\gamma=1/N$, and
$\nu\triangleq(k-1)/2$. Also, let
\[
\delta_R^2
\triangleq
\EE\left[|\tilde R_0-\tilde R_1|^2\right]
\leq 2048c^4.
\]
We apply \cref{lem:tv-chi-square} with
$X_1=\tilde R_0$, $X_2=\tilde R_1$, and
$Y=Y_1\sim\chi^2(\nu)$. Since $Y_1$ is independent of
$(\tilde R_0,\tilde R_1)$, we obtain
\begin{align}
&\TV{R_0+(1+\gamma)Y_1\mid E_0}
{R_1+Y_1\mid E_1}
\nonumber\\
&=
\TV{\tilde R_0+(1+\gamma)Y_1}
{\tilde R_1+Y_1}
\nonumber\\
&\stackrel{(a)}{\leq}
\sqrt{
\frac{\gamma\delta_R}{4}
+\frac{\gamma^2\nu}{4}
+\frac{\delta_R^2}{4\nu-16}
}
\nonumber\\
&\stackrel{(b)}{\leq}
\sqrt{
\frac{12c^2}{n-1+n^2\rho^2}
+\frac{k}{8(n-1+n^2\rho^2)^2}
+\frac{1024c^4}{k-9}
}.
\label{eq:tv-norm-t1-final}
\end{align}

Here, Step $(a)$ follows from \cref{lem:tv-chi-square}. For Step $(b)$,
using $\delta_R^2\leq2048c^4$, $\gamma=1/N$, and
$\nu=(k-1)/2$, we have
\[
\frac{\gamma\delta_R}{4}
\leq
\frac{\sqrt{2048}\,c^2}{4N}
\leq
\frac{12c^2}{N},
\qquad
\frac{\gamma^2\nu}{4}
=
\frac{k-1}{8N^2}
\leq
\frac{k}{8N^2}, \qquad 
\frac{\delta_R^2}{4\nu-16}
\leq
\frac{2048c^4}{2k-18}
=
\frac{1024c^4}{k-9}.
\]

To control the probability of the complement of $\mathcal{G}$, first notice that, conditioned on $UU^\top$,
\[
\frac{\norm{UU^\top X_0}^2}{1+\sigma^2}
\sim \chi^2(k)
\qquad\text{and}\qquad
\norm{(\id{d}-UU^\top)X_0}^2
\sim \chi^2(d-k).
\]
Moreover, conditioned on $(X_0,UU^\top)$,
$
\frac{\inner{X_0}{\outputmodel_0}}
{\alpha_0
\sqrt{\sigma^2\norm{UU^\top X_0}^2+\norm{X_0}^2}}
\sim \Normal(0,1).
$
Therefore, by a union bound,
\begin{align}
&\Pr\left(
\left(X_0,UU^\top,\inner{X_0}{\outputmodel_0}\right)
\in\mathcal{G}^c
\right)
\nonumber\\
&\leq
\Pr\left(
\left|\inner{X_0}{\outputmodel_0}\right|
>
c\alpha_0
\sqrt{\sigma^2\norm{UU^\top X_0}^2+\norm{X_0}^2}
\right)
\nonumber\\
&\quad+
\Pr\left(
\left|
\norm{UU^\top X_0}^2-(1+\sigma^2)k
\right|
>
c(1+\sigma^2)\sqrt{2k}
\right)
\nonumber\\
&\quad+
\Pr\left(
\norm{(\id{d}-UU^\top)X_0}^2
>
(d-k)+c\sqrt{2(d-k)}
\right)
\nonumber\\
&\leq \frac{3}{c^2},
\label{eq:good-event-prob}
\end{align}
where the last inequality follows from Chebyshev's inequality, using
$\Var(\chi^2(k))=2k$ and $\Var(\chi^2(d-k))=2(d-k)$. We also use that a standard normal random variable has variance one.

Putting everything together, we obtain the following.
Here and below, the expectation and probability are with respect to the OUT mixing law of $(X_0,UU^\top,\inner{X_0}{\outputmodel_0})$.
\begin{align*}
&\EE\left[\TV{\norm{\Pi_{x^\perp} \outputmodel _0}^2 \Big| E_0}{\norm{\Pi_{x^\perp} \outputmodel _1}^2 \Big| E_1}\right] \\
&\stackrel{(a)}{\leq} \EE\left[\TV{\alpha_0^2\left( (Z + \sigma \norm{v}\xi_{1,0})^2 + (1+\sigma^2) Y_1\right) \Big| E_0}{\alpha_1^2 \left((Z + \sigma \norm{v}\xi_{1,1})^2 + (1+\sigma^2) Y_1\right) \Big| E_1}\right] \\
&+ \TV{\alpha_0^2 \left((1+\sigma^2)Y_2+ W\right)}{\alpha_1^2 \left((1+\sigma^2)Y_2+ W\right)}\\
&\stackrel{(b)}{\leq}  \EE\left[\TV{\alpha_0^2\left( (Z + \sigma \norm{v}\xi_{1,0})^2 + (1+\sigma^2) Y_1\right) \Big| E_0}{\alpha_1^2 \left((Z + \sigma \norm{v}\xi_{1,1})^2 + (1+\sigma^2) Y_1\right) \Big| E_1} \indic{\mathcal{G}}\right] \\
& +  \Pr\left(\mathcal{G}^c\right) + \TV{\alpha_0^2 \left((1+\sigma^2)Y_2+ W\right)}{\alpha_1^2 \left((1+\sigma^2)Y_2+ W\right)}\\
&\stackrel{(c)}{\leq}  \sqrt{\frac{12 c^2}{n-1+n^2\rho^2} + \frac{k}{8(n-1+n^2\rho^2)^2} + \frac{1024 c^4}{k-9} }  + \frac{3}{c^2}  \\
&+\sqrt{\frac{1}{(n-1+n^2\rho^2)^2}\frac{\sqrt{2d+d^2}}{4(1+\sigma^2)} + \frac{k}{8(n-1+n^2\rho^2)^2} + \frac{1}{(n-1+n^2\rho^2)^2}\frac{2d+d^2}{(2k-18)(1+\sigma^2)^2}} .
\end{align*}
Here $(a)$ follows from \cref{eq:tv-norm-proj-decompose}, $(b)$ follows from the fact that total variation is at most one, and $(c)$ follows from \cref{eq:tv-norm-t1-final,eq:tv-norm-t2-final,eq:good-event-prob}. The claim now follows from \cref{eq:tv-norm-step0}.
\end{proof}

\subsection{Proof of the Upper Bound}
\ubcov*

\mhtext{
\begin{proof}
We give here a sketch of the proof, as the proof is similar to \citet{dwork2015robust}. We show that by having $m=\Omega(d)$ samples, the attacker can estimate the inverse covariance sufficiently well to construct a successful sample-based MIA.

Let the unknown data distribution be $\Normal(\mu,\Sigma)$. Fix $c>1$ to be a sufficiently large universal constant. Let $m=2m_1+m_2+1$, where $m_1=\lceil cd\rceil$ and $m_2=\lceil c\min{n,1/\rho^2}\rceil$. Since $d\geq c(n+n^2\rho^2)$, we have $m=\Theta(d)$. We divide the first $2m_1$ auxiliary samples into independent pairs $(Y_{i,1},Y_{i,2})$ for $i\in[m_1]$, use the next $m_2$ samples $Y_1',\ldots,Y_{m_2}'$ to estimate the mean, and denote the last held-out sample by $Y_0$. Define

$$
D_i:=\frac{Y_{i,1}-Y_{i,2}}{\sqrt{2}},
\qquad
H:=\left(\frac{1}{m_1}\sum_{i=1}^{m_1}D_iD_i^\top\right)^{-1}
=
\left(\frac{1}{2m_1}\sum_{i=1}^{m_1}(Y_{i,1}-Y_{i,2})(Y_{i,1}-Y_{i,2})^\top\right)^{-1},
\qquad
\bar\mu_0:=\frac{1}{m_2}\sum_{i=1}^{m_2}Y_i'.
$$

Notice that $D_1,\ldots,D_{m_1}$ are i.i.d.\ samples from $\Normal(0,\Sigma)$. In particular, since $m_1>d$, $H$ exists with probability one.

Consider the score function

$$
\scorefun_m(\auxsample^m,\hat\mu,X)
:=
\inner{H^{1/2}(\hat\mu-\bar\mu_0)}
{H^{1/2}(X-Y_0)}
=
(\hat\mu-\bar\mu_0)^\top H(X-Y_0).
$$

The attacker outputs $\mathsf{IN}$ if $\scorefun_m(\auxsample^m,\hat\mu,X)\geq d/(2n)$ and outputs $\mathsf{OUT}$ otherwise.

Let $M:=\Sigma^{1/2}H\Sigma^{1/2}$. By \cref{lem:operator-norm-cov}, for a sufficiently large $c$, with probability at least $0.99$,

$$
\mathcal E
:=
\left\{
\left\|M-\id{d}\right\|_{\mathrm{op}}\leq\frac14
\right\}.
$$

In particular, on $\mathcal E$, we have $\frac34\id{d}\preceq M\preceq\frac54\id{d}$.

Since $\hat\mu$ is the output of a noise-adding mechanism, we can write $\hat\mu=\frac1n\sum_{j=1}^nX_j+\xi$, where $\xi$ is mean zero and independent of $X_1,\ldots,X_n$. Let $Z_j:=\Sigma^{-1/2}(X_j-\mu)$, $Z_0:=\Sigma^{-1/2}(X_0-\mu)$, $U_j:=\Sigma^{-1/2}(Y_j'-\mu)$, $U_0:=\Sigma^{-1/2}(Y_0-\mu)$, and $\eta:=\Sigma^{-1/2}\xi$. Also, let $\bar Z:=\frac1n\sum_{j=1}^nZ_j$ and $\bar U:=\frac1{m_2}\sum_{j=1}^{m_2}U_j$. By \cref{def:accurate-mean-estimator},

$$
\mathbb{E}\left[\|\eta\|^2\right]
=
\mathbb{E}\left[
\left\|\Sigma^{-1/2}(\hat\mu-\mu)\right\|^2
\right]
-\frac dn
\leq \rho^2d.
$$

Conditioned on $H$, the score in the $\mathsf{OUT}$ case and the score in the $\mathsf{IN}$ case for $i\in[n]$ can be written as

$$
T_0=(\bar Z+\eta-\bar U)^\top M(Z_0-U_0),
\qquad
T_i=(\bar Z+\eta-\bar U)^\top M(Z_i-U_0).
$$

A direct calculation using the independence of the random variables shows that, on $\mathcal E$,

$$
\mathbb{E}[T_0\mid H]=0,
\qquad
\mathbb{E}[T_i\mid H]=\frac{\operatorname{Tr}(M)}{n}\geq\frac{3d}{4n},
$$

and

$$
\max\left\{\Var(T_0\mid H),\Var(T_i\mid H)\right\}
\leq
C d\left(\rho^2+\frac1n+\frac1{m_2}\right)
$$

for a universal constant $C$. Finally, the definition of $m_2$ implies $
n^2\left(\rho^2+\frac1n+\frac1{m_2}\right)
\leq
C\left(n+n^2\rho^2\right).
$
Therefore, by Chebyshev's inequality and the assumption $d\geq c(n+n^2\rho^2)$, for a sufficiently large $c$,

$$
\Pr\left(T_0\geq\frac{d}{2n}\,\middle|\,\mathcal E\right)\leq0.48,
\qquad
\Pr\left(T_i<\frac{d}{2n}\,\middle|\,\mathcal E\right)\leq0.48.
$$

Since $\Pr(\mathcal E^c)\leq0.01$, the false-positive probability is at most $0.49$, while the true-positive probability is at least $0.51$ for every $i\in[n]$. Thus, the proposed strategy is an $m$-sample-based MIA. Since $m=\Theta(d)$, the result follows.
\end{proof}
}

\section{Proofs from \Cref{sec:unknown-mean}} 

\lbunknownmean*
\begin{proof}
Recall the definition of the family of the data distributions in  $\distfamilyum_d$. Fix $\mu \in \Reals^d$ and define $Z_0, Z_1 \sim \Normal(0,\id{d})$ where $Z_0 \indep Z_1$, $\alpha = \sqrt{\frac{1}{n}+\rho^2}$, and $\beta = \sqrt{\frac{n-1}{n^2}+\rho^2}$. Then, consider the following (correlated) random variables
\[
X_0 = \mu + Z_0, \quad \hat \mu_0 = \mu + \alpha Z_1,
\]
\[
X_1 = \mu + Z_0, \quad\hat \mu_1 = \mu + \frac{Z_0}{n} +  \beta Z_1.
\]
Define the distribution $\pi = \Normal(0,\id{d})$. We aim to use \cref{lem:fuzzy-lecam} to prove a lower bound and in particular we assume that $\mu \sim \pi$. Using \cref{lem:fuzzy-lecam} and the chain rule for the KL divergence, we have
\begin{equation}
\label{eq:kl-umeans-step1}
\KL{\auxsample^m,X_1,\outputmodel_1}{\auxsample^m,X_0,\outputmodel_0}= \EE\left[\KL{X _1,\hat \mu _1 \vert  \auxsample^m }{X _0,\hat \mu _0 \vert  \auxsample^m }\right].
\end{equation}

Consider $\KL{X _1,\hat \mu _1 \vert  \auxsample^m =\bm{y}^m}{X _0,\hat \mu _0 \vert  \auxsample^m =\bm{y}^m}$. Notice that conditioning on $ \auxsample^m =\bm{y}^m$ (auxiliary samples) only changes the distribution of $\mu$. In particular, using \cref{lem:gauss-posterior}, we have

\[
\mu \vert  \auxsample^m =\bm{y}^m \sim \Normal\left( \frac{m}{m + 1} \bar{\bm{y}}^m , \frac{1}{m + 1} \id{d}\right),
\]
where 
\[
\bar{\bm{y}}^m = \frac{1}{m}\sum_{i=1}^{m} y_i.
\]
 Let $(\tilde{Z},Z_0, Z_1) \sim \Normal(0,\id{d})^{\otimes 3}$. Using this observation, we can represent the random variables that we compare their KL divergence in  \cref{eq:kl-umeans-step1} as follows:

\[
X _0 = \frac{m}{m + 1} \bar{\bm{y}}^m +\frac{1}{\sqrt{1+m}}\tilde{Z} + Z_0, \quad \hat \mu _0 = \frac{m}{m + 1} \bar{\bm{y}}^m +  \frac{1}{\sqrt{1+m}}\tilde{Z} + \alpha Z_1,
\]

\[
X _1 = \frac{m}{m + 1} \bar{\bm{y}}^m + \frac{1}{\sqrt{1+m}}\tilde{Z} + Z_0, \quad\hat \mu _1 = \frac{m}{m + 1} \bar{\bm{y}}^m + \frac{1}{\sqrt{1+m}}\tilde{Z} + \frac{Z_0}{n} +  \beta Z_1,
\]
By the invariance property of the KL divergence to the translation of the random variables, we can subtract all the random variables by $\frac{m}{m + 1} \bar{\bm{y}}^m$. Therefore, \cref{eq:kl-umeans-step1} is given by
\begin{equation}
\label{eq:kl-umeans-step2}
\EE\left[\KL{X _1,\hat \mu _1 \vert  \auxsample^m }{X _0,\hat \mu _0 \vert  \auxsample^m }\right] = \KL{\Normal\left(0,\Sigma_1\right)}{\Normal\left(0,\Sigma_2\right)},
\end{equation}
where $\Sigma_1,\Sigma_2\in \Reals^{d\times d}$ are 
\begin{equation}
\nonumber
\Sigma_2 \;=\;
\begin{pmatrix}
 \displaystyle \frac{m+2}{m+1} \id{d} &  \displaystyle \frac{1}{m+1} \id{d} \\[6pt]
 \displaystyle \frac{1}{m+1} \id{d} &  \displaystyle \left(\frac{1}{m+1} + \alpha^{2}\right) \id{d}
\end{pmatrix},
\qquad
\Sigma_1 \;=\;
\begin{pmatrix}
\displaystyle \frac{m+2}{m+1} \id{d} & \displaystyle \left(\frac{1}{m+1} + \frac{1}{n}\right)\id{d} \\[6pt]
\displaystyle \left(\frac{1}{m+1} + \frac{1}{n}\right)\id{d} & \displaystyle \left(\frac{1}{m+1} + \alpha^{2}\right) \id{d}
\end{pmatrix}.
\end{equation}
In the next step, we use the closed-form expression for the KL divergence between two Guassians.
\[
\KL{\Normal\left(0,\Sigma_1\right)}{\Normal\left(0,\Sigma_2\right)} &= \frac{1}{2}\left[\log\frac{\det(\Sigma_2)}{\det(\Sigma_1)} - 2d + \trace\left(\Sigma_2^{-1}\Sigma_1\right)\right].
\]
By simple algebraic manipulations, 
\[
\log(\det(\Sigma_2)) &= d\log\left(\frac{1}{m+1} + \alpha^2 \frac{m+2}{m+1}\right)\\
\log(\det(\Sigma_1)) &= d \log\left(\frac{1}{m+1} + \alpha^2 \frac{m+2}{m+1} - \frac{1}{n^2} - \frac{2}{n(m+1)}\right)\\
-2 + \frac{1}{d}\cdot\trace(\Sigma_2^{-1}\Sigma_1)&= - \frac{2}{n\left(1+ \alpha^2\left(m+2\right)\right)}.
\]
Therefore,
\[
&\KL{\Normal\left(0,\Sigma_1\right)}{\Normal\left(0,\Sigma_2\right)} \\
&= \frac{d}{2}\left[ \log\left(1 + \frac{m+1+2n}{n^2 + \alpha^2 n^2(m+2) - (m+1)-2n}\right) - \frac{2}{n\left(1+\alpha^2 \left(m+2\right)\right)} \right].
\]
Using $\log(1+x)\leq x$, we can provide an upper bound on the KL divergence as follows
\[
&\log\left(1 + \frac{m+1+2n}{n^2 + \alpha^2 n^2(m+2) - (m+1)-2n}\right) - \frac{2}{n\left(1+\alpha^2 \left(m+2\right)\right)} \\
&\leq \frac{m+1+2n}{n^2 + n^2 \alpha^2 (m+2) - (m+1) -2n} - \frac{2}{n\left(1+\alpha^2\left(m+2\right)\right)}.
\]
Define $r = 1+ \alpha^2\left(m+2\right)$. Then,
\mhtext{
\begin{align*}
\frac{2n+m+1}{n^2r-(m+1)-2n}-\frac{2}{nr}&=
\frac{(2n+m+1)nr-2\left(n^2r-(m+1)-2n\right)}
{\left(n^2r-(m+1)-2n\right)(nr)}\\
&=
\frac{(m+1)nr+2(m+1)+4n}
{\left(n^2r-(m+1)-2n\right)(nr)}\\
&\leq
\frac{3(m+1)}
{n^2r-(m+1)-2n},
\end{align*}
}

where the last inequality follows since $n\geq 4$, $m\geq 2$, and
$r\geq 1$ imply
$
2(m+1)+4n\leq 2n(m+1)r.
$
\mhtext{We can provide a lower bound on the denominator as follows:
\[
\begin{aligned}
n^2 r - (m+1) -2n
&= n^2 + n(m+2) + n^2 \rho^2\left(m+2\right)-(m+1) - 2n\\
&\geq n^2 + \left(n-1\right)\left(m+2\right)
+ n^2\rho^2\left(m+2\right)-2n\\
&\geq n^2 + \left(n-2\right)m
+ n^2\rho^2\left(m+2\right),
\end{aligned}
\]
where in the last step we assume $m\geq 2$. Therefore, we can write
\[
\begin{aligned}
\frac{m+1}{n^2 r - (m+1) -2n}
&\leq
\frac{m+1}
{n^2 + \left(n-2\right)m+n^2\rho^2\left(m+2\right)}\\
&=
\frac{1}
{\frac{n^2}{m+1}
+\left(n-2\right)\frac{m}{m+1}
+n^2\rho^2\frac{m+2}{m+1}}\\
&\leq
\frac{1}
{\frac{n^2}{m+1}+\frac{n}{3}+n^2\rho^2},
\end{aligned}
\]
where in the last line we assumed $n\geq 4$. Thus,
\[
\KL{\Normal\left(0,\Sigma_1\right)}
{\Normal\left(0,\Sigma_2\right)}
\leq
\frac{3d}
{2\left(
\frac{n^2}{m+1}+\frac{n}{3}+n^2\rho^2
\right)}.
\]
Then, by Pinsker's inequality \citep{polyanskiy2025information},
\[
\begin{aligned}
\TV{\auxsample^m,X_1,\hat\mu_1}
{\auxsample^m,X_0,\hat\mu_0}
&\leq
\sqrt{
\frac{1}{2}
\KL{\auxsample^m,X_1,\hat\mu_1}
{\auxsample^m,X_0,\hat\mu_0}
}\\
&\leq
\sqrt{
\frac{3d}
{4\left(
\frac{n^2}{m+1}+\frac{n}{3}+n^2\rho^2
\right)}
}.
\end{aligned}
\]
By setting the upper bound to a sufficiently small constant, the claim follows.}
\end{proof}

\section{Helper Lemmas}
\label{sec:helpers}

\begin{lemma}[{\citealp[]{yale_chainrule}}]\label{lem:chain-rule-tv}
Let $(X_1,Y_1)$ and $(X_2,Y_2)$ be two pairs of jointly distributed random variables. Then,
\[
\TV{X_1,Y_1}{X_2,Y_2} \leq \TV{X_1}{X_2} + \EE\left[\TV{Y_1\Big| X_1}{Y_2\Big| X_2}\right].
\]
\end{lemma}

\begin{lemma}[{\citealp[]{murphy2007conjugate}}] \label{lem:gauss-posterior}
Fix $m, d \in \Naturals$ and $\sigma \in \Reals$. Consider the following sampling procedure: $\mu \sim \Normal\left(0,\sigma^2 \id{d}\right)$, $\left(Y_1,\dots,Y_m\right) \sim \Normal\left(\mu, \id{d}\right)^{\otimes m}$. Then,
\[
\mu \Big| Y_1,\dots,Y_m \eqdist \Normal\left( \frac{m\sigma^2}{m\sigma^2 + 1} \bar{Y}_m , \frac{\sigma^2}{m\sigma^2 + 1} \id{d}\right),
\]
where 
\[
\bar{Y}_m = \frac{1}{m}\sum_{i=1}^{m}Y_i.
\]
\end{lemma}

\begin{lemma}[{\citealp[Thm.~4.7.1]{vershynin2018high}}] \label{lem:operator-norm-cov}
Let $c$ be a universal constant and $n,d \in \Naturals$. Let $(Z_1,\dots,Z_n)\sim \Normal(0,\id{d})^{\otimes n}$ be $n$ \iid~Gaussian random variables. Then, for every $t>0$, we have
\[
 \Pr\left(\norm{\frac{1}{n}\sum_{i=1}^n Z_iZ_i^\top -\id{d}}_{op}  \leq c\left( \sqrt{\frac{d+t}{n}} + \frac{d+t}{n} \right)\right) \geq 1- 2\exp\left(-t\right)
\]
\end{lemma}

\begin{lemma}
\label{lem:sec-power}
Let $X \sim \Normal(0,1)$. Then, for every $s>0$ we have
\[
\EE\left[\exp\left(-sX^2\right)\right] = \frac{1}{\sqrt{1+2s}}.
\]
\end{lemma}
\begin{proof}
By the definition, 
\begin{align*}
\EE\left[\exp\left(-sX^2\right)\right] &= \frac{1}{\sqrt{2\pi}}\int_{-\infty}^{\infty} \exp\left(-sx^2\right) \exp\left(-\frac{x^2}{2}\right) \text{d}x\\
&=  \sigma \frac{1}{\sqrt{2\pi\sigma^2}} \int \exp\left(-\frac{x^2}{2 \sigma^2}\right) \text{d}x\\
&= \sigma,
\end{align*}
where $\sigma^2= \frac{1}{1+2s}$.
\end{proof}

\begin{lemma} \label{lem:int-1}
Let $\alpha >0$ and $\beta>1$. Then, 
\[
\int_{s=0}^{\infty} \frac{1}{(1+\alpha s)^\beta} \text{d}s = \frac{1}{\alpha (\beta - 1)}.
\]
\end{lemma}
\begin{proof}
We use the substitution $u = 1 + \alpha s$, so $du = \alpha ds$ and $s = (u-1)/\alpha$. When $s = 0$, we have $u = 1$, and as $s \to \infty$, we have $u \to \infty$. Thus:
\begin{align*}
\int_{s=0}^{\infty} \frac{1}{(1 + \alpha s)^{\beta}} ds &= \int_{1}^{\infty} \frac{1}{u^{\beta}} \cdot \frac{1}{\alpha} du\\
&= \frac{1}{\alpha} \int_{1}^{\infty} u^{-\beta} du\\
&= \frac{1}{\alpha} \left[ \frac{u^{-\beta+1}}{-\beta+1} \right]_{1}^{\infty}.
\end{align*}
Since $\beta > 1$, we have $-\beta + 1 < 0$, so $u^{-\beta+1} \to 0$ as $u \to \infty$. Therefore:
$$\int_{s=0}^{\infty} \frac{1}{(1 + \alpha s)^{\beta}} ds = \frac{1}{\alpha} \cdot \frac{0 - 1^{-\beta+1}}{-\beta+1} = \frac{1}{\alpha} \cdot \frac{-1}{-\beta+1} = \frac{1}{\alpha(\beta-1)}.$$
\end{proof}

\begin{lemma} \label{lem:guassian-norm-four-power}
Let $d \in \Naturals$ and $s>0$. Let $Z \sim \Normal(0,\id{d})$. Then, we have
\[
\EE\left[ \norm{Z}^4 \exp\left(-s\norm{Z}^2\right) \right]= \frac{8}{(1+2s)^{(\frac{d}{2}+3)}} \frac{\Gamma\left(\frac{d}{2}+2\right)}{\Gamma\left(\frac{d}{2}\right)} 
\]
\end{lemma}
\begin{proof}
Since $Z \sim \mathcal{N}(0, I_d)$, we have $\|Z\|^2 \sim \chi^2(d)$ with PDF
$$f_Y(y) = \frac{1}{2^{d/2}\Gamma(d/2)} y^{d/2-1} e^{-y/2}, \quad y > 0.$$

Let $Y = \|Z\|^2$. We compute
\begin{align*}
\mathbb{E}\left[\|Z\|^4 \exp(-s\|Z\|^2)\right] &= \mathbb{E}\left[Y^2 e^{-sY}\right]\\
&= \int_0^\infty y^2 e^{-sy} \cdot \frac{1}{2^{d/2}\Gamma(d/2)} y^{d/2-1} e^{-y/2} dy\\
&= \frac{1}{2^{d/2}\Gamma(d/2)} \int_0^\infty y^{(d/2+2)-1} e^{-y(s+1/2)} dy\\
&= \frac{1}{2^{d/2}\Gamma(d/2)} \cdot \frac{\Gamma(d/2+2)}{(s+1/2)^{d/2+2}}.
\end{align*}

Using $\Gamma(d/2+2) = (d/2+1)(d/2)\Gamma(d/2)$ and $s+1/2 = (1+2s)/2$:
\begin{align*}
&= \frac{(d/2+1)(d/2)}{2^{d/2}(s+1/2)^{d/2+2}}\\
&= \frac{(d/2+1)(d/2)}{2^{d/2}} \cdot \frac{2^{d/2+2}}{(1+2s)^{d/2+2}}\\
&= \frac{4(d/2+1)(d/2)}{(1+2s)^{d/2+2}}.
\end{align*}

Since $(d/2+1)(d/2) = d(d+2)/4 = 2\Gamma(d/2+2)/\Gamma(d/2)$ and $(1+2s)^{d/2+2} = (1+2s)^{(d/2+3)}$:
$$\mathbb{E}\left[\|Z\|^4 \exp(-s\|Z\|^2)\right] = \frac{8\Gamma\left(\frac{d}{2}+2\right)}{(1+2s)^{\left(\frac{d}{2}+3\right)}\Gamma\left(\frac{d}{2}\right)}.$$
\end{proof}
\begin{lemma} \label{lem:gamma-ub}
Let $m \in \Naturals$ be an integer such that $m\geq 3$. Then, we have
\[
\frac{\Gamma\left(\frac{m}{2}+2\right)}{\Gamma\left(\frac{m}{2}\right)} \leq  \frac{1}{2} \left(m+3\right)^2.
\]
\end{lemma}
\begin{proof}
We start with the fundamental recurrence relation of the Gamma function, which is $\Gamma(z+1) = z\Gamma(z)$. Let $x = \frac{m}{2}$. The LHS of the inequality can be written as $\frac{\Gamma(x+2)}{\Gamma(x)}$. We can expand the numerator using the recurrence relation twice:
\[
\Gamma(x+2) = (x+1)\Gamma(x+1) = (x+1)x\Gamma(x)
\]

Now, we can substitute this back into the LHS expression and simplify:
\[
\frac{\Gamma(x+2)}{\Gamma(x)} = \frac{(x+1)x\Gamma(x)}{\Gamma(x)} = x(x+1)
\]

Substitute $x = \frac{m}{2}$ back into the simplified expression:
\[
x(x+1) = \frac{m}{2}\left(\frac{m}{2} + 1\right) = \frac{m}{2}\left(\frac{m+2}{2}\right) = \frac{m(m+2)}{4} = \frac{m^2 + 2m}{4}
\]

Finally, it is obvious that for $m\geq 3$, $\frac{m^2 + 2m}{4} \leq \frac{1}{2}\left(m+3\right)^2$.
\end{proof}

\begin{lemma} \label{lem:cond-inner}
Let $a,b,c$ are constants such that $a^2+b^2>0$. Let $(X,Y)\sim \Normal(0,1)^{\otimes 2}$. Then, 
\[
Y \Big| aX + bY = c \sim \Normal\left( c \cdot \frac{b}{a^2+b^2},\frac{a^2}{a^2+b^2} \right).
\]
\end{lemma}
\begin{proof}
Since $(X,Y) \sim \mathcal{N}(0,1)^{\otimes 2}$, we have $X$ and $Y$ are independent standard normal random variables.

The linear combination $aX + bY$ is normally distributed with
$$\mathbb{E}[aX + bY] = a\mathbb{E}[X] + b\mathbb{E}[Y] = 0.$$

The variance is
$$\text{Var}(aX + bY) = a^2\text{Var}(X) + b^2\text{Var}(Y) = a^2 + b^2.$$

Therefore $aX + bY \sim \mathcal{N}(0, a^2 + b^2)$. Conditioning on $aX + bY = c$ gives
$$Y \mid aX + bY = c \sim \mathcal{N}\left(\frac{bc}{a^2 + b^2}, \frac{a^2}{a^2 + b^2}\right).$$ This follows from the standard formula for conditional distributions of jointly normal random variables, where
$$\mathbb{E}[Y \mid aX + bY = c] = \text{Cov}(Y, aX + bY) \cdot \frac{c}{\text{Var}(aX + bY)} = \frac{b}{a^2 + b^2} \cdot c$$
and
$$\text{Var}(Y \mid aX + bY = c) = \text{Var}(Y) - \frac{[\text{Cov}(Y, aX + bY)]^2}{\text{Var}(aX + bY)} = 1 - \frac{b^2}{a^2 + b^2} = \frac{a^2}{a^2 + b^2}.$$
\end{proof}
\begin{lemma} \label{lem:KL-chi-squared}
Let $a\geq 0$ be a constant and $r \geq 0$ and $k \in \Naturals$ such that $k>4$. Let $Y \sim \chi^2(k)$. Then,
\[
\max\left\{\KL{a + Y}{(1+r)Y},\KL{a + (1+r)Y}{Y}\right\} &\leq \frac{a^2}{2k-8} + \frac{k r^2}{4} + \frac{ar}{2}.
\]
\end{lemma}

\begin{proof}
We first start with computing the KL divergence between $\gamma Y + a$ and $Y$. The KL divergence is defined as
\[
\mathrm{KL}(\gamma Y + a \| Y) = \mathbb{E}_{X \sim \gamma Y + a}\left[\log \frac{f_{\gamma Y + a}(X)}{f_Y(X)}\right].
\]
The chi-squared density is $f_Y(y) = \frac{1}{2^{k/2}\Gamma(k/2)} y^{k/2-1} e^{-y/2}$ for $y > 0$. For $X = \gamma Y + a$, the density is obtained by change of variables:
\[
f_{\gamma Y + a}(x) = \frac{1}{\gamma} f_Y\left(\frac{x-a}{\gamma}\right) = \frac{1}{\gamma} \cdot \frac{1}{2^{k/2}\Gamma(k/2)} \left(\frac{x-a}{\gamma}\right)^{k/2-1} e^{-\frac{x-a}{2\gamma}}, \quad x > a
\]

Using the substitution $X = \gamma Y + a$, we have
\begin{align*}
\mathrm{KL}(\gamma Y + a \| Y) &= \int_a^{\infty} f_{\gamma Y + a}(x) \log \frac{f_{\gamma Y + a}(x)}{f_Y(x)} \, dx \\
&= \int_0^{\infty} f_Y(y) \log \frac{f_{\gamma Y + a}(\gamma y + a)}{f_Y(\gamma y + a)} \, dy.
\end{align*}
We have
\begin{align*}
\log \frac{f_{\gamma Y + a}(\gamma y + a)}{f_Y(\gamma y + a)} &= \log \left[\frac{\frac{1}{\gamma} \cdot y^{k/2-1} e^{-y/2}}{(\gamma y + a)^{k/2-1} e^{-(\gamma y + a)/2}}\right] \\
&= -\log \gamma + \left(\frac{k}{2}-1\right) \log \frac{y}{\gamma y + a} - \frac{y}{2} + \frac{\gamma y + a}{2} \\
&= -\log \gamma + \left(\frac{k}{2}-1\right) \log \frac{y}{\gamma y + a} + \frac{(\gamma-1)y + a}{2}
\end{align*}

Taking expectation with respect to $Y \sim \chi^2_k$:
\begin{align*}
\mathrm{KL}(\gamma Y + a \| Y) &= -\log \gamma + \left(\frac{k}{2}-1\right) \mathbb{E}\left[\log \frac{Y}{\gamma Y + a}\right] + \frac{\gamma-1}{2} \mathbb{E}[Y] + \frac{a}{2} \\
&= -\log \gamma + \left(\frac{k}{2}-1\right) \mathbb{E}\left[\log \frac{Y}{\gamma Y + a}\right] + \frac{(\gamma-1)k}{2} + \frac{a}{2}
\end{align*}
Also, we have
\[
\log \frac{Y}{\gamma Y + a} = -\log\left(\gamma + \frac{a}{Y}\right) = -\log \gamma - \log\left(1 + \frac{a}{\gamma Y}\right)
\]

Using the inequality $-\log(1 + u) \leq -u + \frac{u^2}{2}$ for $u > 0$:
\[
-\log\left(1 + \frac{a}{\gamma Y}\right) \leq -\frac{a}{\gamma Y} + \frac{a^2}{2\gamma^2 Y^2}.
\]

Therefore,
\[
\mathbb{E}\left[\log \frac{Y}{\gamma Y + a}\right] \leq -\log \gamma - \frac{a}{\gamma} \mathbb{E}\left[\frac{1}{Y}\right] + \frac{a^2}{2\gamma^2} \mathbb{E}\left[\frac{1}{Y^2}\right].
\]

For $Y \sim \chi^2_k$, we have $\mathbb{E}[1/Y] = \frac{1}{k-2}$ for $k > 2$ and $\mathbb{E}[1/Y^2] = \frac{1}{(k-2)(k-4)}$ for $k > 4$.

Setting $\gamma = 1 + r$ and using $-\log(1+r) \leq -r + \frac{r^2}{2}$ for $r\geq 0$:
\begin{align*}
&\mathrm{KL}(\gamma Y + a \| Y) \\
&\leq -\log(1+r) + \left(\frac{k}{2}-1\right)\left[-\log(1+r) - \frac{a}{(1+r)(k-2)} + \frac{a^2}{2(1+r)^2(k-2)(k-4)}\right] \\
&\quad + \frac{rk}{2} + \frac{a}{2}\\
&\leq \frac{r^2 k}{4} + \frac{a^2}{4k - 16} + \frac{ar}{2}.
\end{align*}

The proof for $\KLinline{a + Y}{(1+r)Y}$ follows similarly and we skip for brevity.
\end{proof}

\begin{lemma} \label{lem:tv-chi-square}

Let $\gamma \in (0,1)$ be a constant and $k > 4$ be an integer. Let $X_1$ and $X_2$ be two random variables with a finite second moment. Let 
$\mu_1$ and $\mu_2$ denote the (marginal) distributions of $X_1$ and $X_2$, respectively. Fix an arbitrary coupling $\pi\in\Pi(\mu_1,\mu_2)$ and draw $(X_1,X_2)\sim\pi$. Let $Y\sim\chi^2(k)$ be independent of $(X_1,X_2)$. Define
$
c^2=\EE_{(X_1,X_2)\sim\pi}\big[(X_1-X_2)^2\big].
$
Then,
\[
\TV{X_1 + (1+\gamma)Y}{X_2 +  Y} \leq \sqrt{ \frac{\gamma c}{4}+\frac{\gamma^2 k}{4} + \frac{c^2}{4k-16}}.
\]
\end{lemma}
\begin{proof}
\mhtext{
Let $\beta=1+\gamma$, and let $\mathcal{F}$ denote the set of measurable
functions $f:\Reals\to[-1,1]$. By the variational representation of total
variation \citep{polyanskiy2025information} and the tower property,
\[
\begin{aligned}
\TV{X_1+\beta Y}{X_2+Y}
&=
\frac{1}{2}\sup_{f\in\mathcal{F}}
\left|
\EE\left[
f(X_1+\beta Y)-f(X_2+Y)
\right]
\right|\\
&=
\frac{1}{2}\sup_{f\in\mathcal{F}}
\left|
\EE\left[
\EE\left[
f(X_1+\beta Y)-f(X_2+Y)
\mid X_1,X_2
\right]
\right]
\right|\\
&\leq
\EE\left[
\frac{1}{2}\sup_{f\in\mathcal{F}}
\left|
\EE\left[
f(X_1+\beta Y)-f(X_2+Y)
\mid X_1,X_2
\right]
\right|
\right]\\
&=
\EE_{(X_1,X_2)\sim\pi}
\left[
\TV{x_1+\beta Y}{x_2+Y}
\right],
\end{aligned}
\]
where the inequality follows from the triangle inequality and by
interchanging the supremum and expectation. The last equality follows
because $Y$ is independent of $(X_1,X_2)$.
}

The next step is Pinsker's inequality \citep{polyanskiy2025information}, 
\[
&\max\left\{ \TV{|x_1 - x_2| + \beta Y}{Y}, \TV{|x_1 - x_2| + Y}{\beta Y} \right\} \\
&\leq \max\left\{ \sqrt{\frac{1}{2} \KL{|x_1 - x_2| + \beta Y}{Y}}, \sqrt{\frac{1}{2} \KL{|x_1 - x_2| + Y}{\beta Y} } \right\}\\
&\leq \sqrt{\frac{|x_1 - x_2| \gamma}{4} + \frac{\gamma^2 k}{4} + \frac{|x_1 - x_2|^2}{4k-16}},
\]
where the last line follows from \cref{lem:KL-chi-squared}. Therefore,
\[
\TV{X_1 + \beta Y}{ X_2 + Y} &\leq \EE\left[\sqrt{\frac{|X_1 - X_2| \gamma}{4} + \frac{\gamma^2 k}{4} + \frac{|X_1 - X_2|^2}{4k-16}}\right]\\
&\stackrel{(a)}{\leq} \sqrt{\frac{\EE[|X_1 - X_2|] \gamma}{4} + \frac{\gamma^2 k}{4} + \frac{\EE[|X_1 - X_2|^2]}{4k-16}}\\
&\stackrel{(b)}{\leq}  \sqrt{\frac{\sqrt{\EE[|X_1 - X_2|^2]} \gamma}{4} + \frac{\gamma^2 k}{4} + \frac{\EE[|X_1 - X_2|^2]}{4k-16}},
\]
where $(a)$ follows from Jensen's inequality, and $(b)$ follows from $\EE\left[|X_1 - X_2|\right]\leq \sqrt{\EE\left[|X_1 - X_2|^2\right]}$ by another application of Jensen's inequality. 
\end{proof}

\begin{lemma}\label{lem:mean-var-overview}
Fix $d \in \Naturals$, $\rho >0$,  and $\Sigma \in \Reals^{d \times d}$. Let $\Dist = \Normal(0,\Sigma)$ be a Gaussian distribution with covariance $\Sigma$.  $(X_0,X_1,\dots,X_n,Z)\sim \Dist^{\otimes (n+2)}$. Let $\hat \mu = \frac{1}{n}\sum_{i=1}^n X_i + \rho Z$. Then, we have 
\[
\Var\left(\inner{\hat \mu}{X_0}\right) =  \left(\frac{1}{n} + \rho^2\right)  \norm{\Sigma}_F^2.
\]
\end{lemma}
\begin{proof}
Note that $\EE\left[\inner{\hat \mu}{X_0}\right]=0$ due to the independence of $\hat \mu$ and $X_0$. Therefore, 
\[
\Var\left(\inner{\hat \mu}{X_0}\right) &= \EE\left[ \inner{\hat \mu}{X_0}^2 \right]\\
&= \EE\left[ \outputmodel ^\top X_0 X_0^\top \outputmodel \right]\\
&= \EE\left[ \outputmodel^\top \Sigma \outputmodel\right]\\
&= \EE\left[ \trace\left(\Sigma \outputmodel \outputmodel^\top\right)\right]\\
&=  \left(\frac{1}{n} + \rho^2\right)  \trace\left(\Sigma^2\right)\\
&= \left(\frac{1}{n} + \rho^2\right)  \norm{\Sigma}_F^2.
\]
\end{proof}

\begin{lemma} \label{lem:description-proj}
Let $x_1, \dots, x_n$ be a set of vectors in $\mathbb{R}^d$ and let $y_1, \dots, y_n$ be another set of vectors in $\mathbb{R}^d$, where $n < d$. Let $m$ be an integer satisfying $n < m < d/2$. An orthogonal projection operator $\Pi$ onto an $m$-dimensional subspace of $\mathbb{R}^d$ satisfies $\Pi x_i = y_i$ for all $i \in \{1, \dots, n\}$ if and only if two conditions are met:

\begin{enumerate}
    \item The subspaces $\mathcal{Y} = \mathrm{span}\{y_1, \dots, y_n\}$ and $\mathcal{Z} = \mathrm{span}\{x_1 - y_1, \dots, x_n - y_n\}$ are orthogonal, i.e., $\mathcal{Y} \perp \mathcal{Z}$.
    
    \item The operator $\Pi$ can be expressed as:
    $$
    \Pi = V_1 V_1^T + V_2 U U^T V_2^T
    $$
    where:
    \begin{itemize}
        \item $r = \dim(\mathcal{Y})$ and $s = \dim(\mathcal{Z})$.
        \item $V_1 \in \mathbb{R}^{d \times r}$ is a matrix whose columns form an orthonormal basis for $\mathcal{Y}$.
        \item $V_2 \in \mathbb{R}^{d \times (d-r-s)}$ is a matrix whose columns form an orthonormal basis for the subspace $(\mathcal{Y} \oplus \mathcal{Z})^\perp$.
        \item $U \in \mathbb{R}^{(d-r-s) \times (m-r)}$ is any matrix with orthonormal columns.
    \end{itemize}
\end{enumerate}
Here $\oplus$ denotes the direct sum of subspaces. For subspaces $\mathcal{U}, \mathcal{V} \subseteq \mathbb{R}^d$, the direct sum is $
\mathcal{U} \oplus \mathcal{V} = \{u + v : u \in \mathcal{U}, v \in \mathcal{V}\},$ Also, $^\perp$ denotes the orthogonal complement. For a subspace $\mathcal{S} \subseteq \mathbb{R}^d$, the orthogonal complement is:
$
\mathcal{S}^\perp = \{v \in \mathbb{R}^d : \langle v, s \rangle = 0 \text{ for all } s \in \mathcal{S}\}
$
\end{lemma}

\begin{proof}
($\implies$) \textbf{Necessity.} Assume there exists an orthogonal projection operator $\Pi$ onto an $m$-dimensional subspace, which we denote $V = \mathrm{range}(\Pi)$, such that $\Pi x_i = y_i$ for all $i$.

First, we prove the orthogonality condition. Since $\Pi$ is a projection and $\Pi x_i = y_i$, it follows that $\Pi y_i = \Pi(\Pi x_i) = \Pi^2 x_i = \Pi x_i = y_i$. This implies that $y_i \in V$ for all $i$, and thus the entire subspace $\mathcal{Y}$ is contained in the range of $\Pi$, i.e., $\mathcal{Y} \subseteq V$.

For an orthogonal projection, the residual vector $x_i - \Pi x_i$ must be orthogonal to the range of the projection. Let $z_i = x_i - y_i = x_i - \Pi x_i$. Then $z_i \in V^\perp$ for all $i$, which implies that the entire subspace $\mathcal{Z} \subseteq V^\perp$. Since $\mathcal{Y} \subseteq V$ and $\mathcal{Z} \subseteq V^\perp$, the subspaces must be orthogonal, proving the first condition.

Now, we derive the form of $\Pi$. Since $\mathcal{Y} \subseteq V$, the space $V$ can be decomposed into the direct sum $V = \mathcal{Y} \oplus \mathcal{S}$, where $\mathcal{S}$ is the orthogonal complement of $\mathcal{Y}$ within $V$. The dimension of this complementary subspace is $\dim(\mathcal{S}) = \dim(V) - \dim(\mathcal{Y}) = m-r$.

As established, $\mathcal{Z} \perp V$, which means any vector in $\mathcal{S} \subseteq V$ must be orthogonal to any vector in $\mathcal{Z}$. Furthermore, any vector in $\mathcal{S}$ is by definition orthogonal to any vector in $\mathcal{Y}$. Therefore, $\mathcal{S}$ must be a subspace of $(\mathcal{Y} \oplus \mathcal{Z})^\perp$.

The projection operator onto $V$ can be written as the sum of projections onto the orthogonal subspaces $\mathcal{Y}$ and $\mathcal{S}$: $\Pi = \mathrm{proj}_{\mathcal{Y}} + \mathrm{proj}_{\mathcal{S}}$.
The projection onto $\mathcal{Y}$ is given by $\mathrm{proj}_{\mathcal{Y}} = V_1 V_1^T$.
Since $\mathcal{S}$ is an $(m-r)$-dimensional subspace of $(\mathcal{Y} \oplus \mathcal{Z})^\perp$, we can construct an orthonormal basis for $\mathcal{S}$ by selecting $m-r$ basis vectors from the space spanned by the columns of $V_2$. This selection is precisely what the matrix $U$ accomplishes. The columns of the matrix $V_2 U$ form an orthonormal basis for $\mathcal{S}$. The projection onto $\mathcal{S}$ is therefore $\mathrm{proj}_{\mathcal{S}} = (V_2 U)(V_2 U)^T = V_2 U U^T V_2^T$.

Combining these gives the required form: $\Pi = V_1 V_1^T + V_2 U U^T V_2^T$.

\vspace{1em}
($\Longleftarrow$) \textbf{Sufficiency.} Assume $\mathcal{Y} \perp \mathcal{Z}$ and let $\Pi$ be an operator of the given form for some matrix $U$ with orthonormal columns. We must show that $\Pi$ is an $m$-dimensional orthogonal projection satisfying $\Pi x_i = y_i$.

Let $P_1 = V_1 V_1^T$ and $P_2 = V_2 U U^T V_2^T$. The ranges of these operators are $\mathcal{Y}$ and a subspace $\mathcal{S} \subseteq (\mathcal{Y} \oplus \mathcal{Z})^\perp$, respectively. Since these ranges are orthogonal, the projectors are also orthogonal ($P_1 P_2 = P_2 P_1 = 0$).

\begin{enumerate}
    \item \textbf{Projection Property:} $\Pi^2 = (P_1 + P_2)^2 = P_1^2 + P_2^2 + P_1 P_2 + P_2 P_1 = P_1 + P_2 = \Pi$. Thus, $\Pi$ is a projection. It is orthogonal because its matrix representation is symmetric.
    
    \item \textbf{Dimensionality:} The rank of $\Pi$ is the dimension of its range, $\mathcal{Y} \oplus \mathcal{S}$.
    $\mathrm{rank}(\Pi) = \dim(\mathcal{Y}) + \dim(\mathcal{S}) = r + (m-r) = m$.
    
    \item \textbf{Constraint Satisfaction:} For any $i$, we verify $\Pi x_i = y_i$. Note that $x_i = y_i + z_i$.
    $\Pi x_i = (P_1 + P_2)(y_i + z_i) = P_1 y_i + P_1 z_i + P_2 y_i + P_2 z_i$.
    \begin{itemize}
        \item $P_1 y_i = y_i$, since $y_i \in \mathrm{range}(P_1) = \mathcal{Y}$.
        \item $P_1 z_i = 0$, since $z_i \in \mathcal{Z}$ and $\mathcal{Z} \perp \mathcal{Y}$.
        \item $P_2 y_i = 0$, since $y_i \in \mathcal{Y}$ and $\mathrm{range}(P_2) \perp \mathcal{Y}$.
        \item $P_2 z_i = 0$, since $z_i \in \mathcal{Z}$ and $\mathrm{range}(P_2) \perp \mathcal{Z}$.
    \end{itemize}
    Combining these results yields $\Pi x_i = y_i$.
\end{enumerate}
This confirms that any operator of the specified form is a valid projection, completing the proof.
\end{proof}

\begin{lemma}\label{lem:posterior-random-dir}
Let $d,m,n\in\Naturals$ with $0<m<d$ and $n\geq 1$. Let
$\mathcal{G}_{d,m}$ denote the set of all rank-$m$ orthogonal
projection matrices in $\Reals^d$. Let $\nu$ denote the uniform
probability measure on $\mathcal{G}_{d,m}$, and let $\Pi\sim\nu$.
Conditional on $\Pi$, let
\[
X_1,\ldots,X_n\overset{\iid}{\sim}
\Normal(0,\id{d}+\sigma^2\Pi)
\]
for some $\sigma>0$, and define $Y_i=\Pi X_i$ for every $i\in[n]$.

For almost every realization
$(x_1,\ldots,x_n,y_1,\ldots,y_n)$, define
\[
\mathcal{Y}=\operatorname{span}\{y_1,\ldots,y_n\},
\qquad
\mathcal{Z}=\operatorname{span}\{x_1-y_1,\ldots,x_n-y_n\},
\]
and let $r=\dim(\mathcal{Y})$ and $s=\dim(\mathcal{Z})$. Let
$V_1\in\Reals^{d\times r}$ have orthonormal columns spanning
$\mathcal{Y}$, and let
$V_2\in\Reals^{d\times(d-r-s)}$ have orthonormal columns spanning
$(\mathcal{Y}\oplus\mathcal{Z})^\perp$. Then, a version of the regular
conditional distribution of $\Pi$ satisfies
\[
\left.\Pi\right|
(X_i,Y_i)_{i\in[n]}=(x_i,y_i)_{i\in[n]}
\eqdist
V_1V_1^\top+V_2QV_2^\top,
\]
where $Q$ is uniformly distributed on
$\mathcal{G}_{d-r-s,m-r}$. Here, we use the conventions
$\mathcal{G}_{\ell,0}=\{0\}$ and
$\mathcal{G}_{\ell,\ell}=\{\id{\ell}\}$.
\end{lemma}

\begin{proof}
Let $\Sigma=\id{d}+\sigma^2\Pi$. Since $\Pi$ is an orthogonal
projection of rank $m$, it has eigenvalues $1$ with multiplicity $m$
and $0$ with multiplicity $d-m$. Therefore,
\[
\det(\Sigma)=(1+\sigma^2)^m.
\]
Moreover, using $\Pi^2=\Pi$, we have
\[
\Sigma^{-1}
=
\id{d}-\frac{\sigma^2}{1+\sigma^2}\Pi.
\]

The conditional density of $X_i$ given $\Pi$ is
\[
p(x_i\mid\Pi)
=
\frac{1}
{(2\pi)^{d/2}\det(\Sigma)^{1/2}}
\exp\left(-\frac{1}{2}x_i^\top\Sigma^{-1}x_i\right).
\]
Substituting the expressions for $\det(\Sigma)$ and $\Sigma^{-1}$,
we obtain
\[
\begin{aligned}
x_i^\top\Sigma^{-1}x_i
&=
x_i^\top
\left(
\id{d}-\frac{\sigma^2}{1+\sigma^2}\Pi
\right)x_i\\
&=
\norm{x_i}^2
-
\frac{\sigma^2}{1+\sigma^2}\norm{\Pi x_i}^2.
\end{aligned}
\]
Therefore, the joint likelihood of $x=(x_1,\ldots,x_n)$ is
\[
\begin{aligned}
\prod_{i=1}^n p(x_i\mid\Pi)
&=
\prod_{i=1}^n
\frac{1}{(2\pi)^{d/2}(1+\sigma^2)^{m/2}}
\exp\left(
-\frac{1}{2}\norm{x_i}^2
+
\frac{\sigma^2}{2(1+\sigma^2)}
\norm{\Pi x_i}^2
\right)\\
&=
\left(
\frac{1}{(2\pi)^{d/2}(1+\sigma^2)^{m/2}}
\right)^n
\exp\left(-\frac{1}{2}\sum_{i=1}^n\norm{x_i}^2\right)
\exp\left(
\frac{\sigma^2}{2(1+\sigma^2)}
\sum_{i=1}^n\norm{\Pi x_i}^2
\right).
\end{aligned}
\]

For fixed $x=(x_1,\ldots,x_n)$, define
\[
T_x(P)=(Px_1,\ldots,Px_n),
\]
and let $\lambda_x=(T_x)_\#\nu$. Let
$\nu_x(\cdot\mid y)$ be a regular conditional distribution of
$P\sim\nu$ given $T_x(P)=y$. Such a regular conditional distribution
exists since the spaces under consideration are standard Borel
spaces, and it satisfies
\[
\nu(dP)\delta_{T_x(P)}(dy)
=
\nu_x(dP\mid y)\lambda_x(dy).
\]
Let
\[
L_x(P)
=
C(x)\exp\left(
\frac{\sigma^2}{2(1+\sigma^2)}
\sum_{i=1}^n\norm{Px_i}^2
\right)
\]
denote the likelihood of $x$ given $P$, where $C(x)$ does not depend
on $P$. By Bayes' rule for regular conditional distributions, for
almost every $(x,y)$,
\[
\Pr(dP\mid X=x,Y=y)
=
\frac{L_x(P)\nu_x(dP\mid y)}
{\int L_x(P')\nu_x(dP'\mid y)}.
\]
The measure $\nu_x(\cdot\mid y)$ is supported on the fiber
\[
\mathcal{M}_{x,y}
=
\left\{
P\in\mathcal{G}_{d,m}:
Px_i=y_i\text{ for every }i\in[n]
\right\}.
\]
On this fiber, $\norm{Px_i}^2=\norm{y_i}^2$ for every $i\in[n]$.
Hence, $L_x(P)$ is constant on $\mathcal{M}_{x,y}$, and therefore
\[
\Pr(dP\mid X=x,Y=y)=\nu_x(dP\mid y).
\]

Notice that if $P\in\mathcal{M}_{x,y}$, then
$y_i\in\operatorname{range}(P)$ and
$x_i-y_i\in\ker(P)$ for every $i\in[n]$. Consequently,
$\mathcal{Y}\perp\mathcal{Z}$, $r\leq m$, and $s\leq d-m$. In
particular,
\[
0\leq m-r\leq d-r-s.
\]
By \cref{lem:description-proj},
\[
\mathcal{M}_{x,y}
=
\left\{
V_1V_1^\top+V_2QV_2^\top:
Q\in\mathcal{G}_{d-r-s,m-r}
\right\}.
\]

It remains to characterize the distribution of $Q$. Let
$\mathcal{H}_{x,y}$ be the subgroup of orthogonal transformations
that fix $\mathcal{Y}\oplus\mathcal{Z}$ pointwise. Every
$O\in\mathcal{H}_{x,y}$ fixes both $y_i$ and $x_i-y_i$, and hence
fixes $x_i$, for every $i\in[n]$. By equivariant disintegration of
the orthogonally invariant measure $\nu$, the regular conditional
distribution $\nu_x(\cdot\mid y)$ may be chosen so that, for almost
every $(x,y)$, it is invariant under
\[
P\mapsto OPO^\top,
\qquad
O\in\mathcal{H}_{x,y}.
\]
In the coordinates given by $V_2$, this is the action
\[
Q\mapsto RQR^\top,
\qquad
R\in O(d-r-s).
\]
This action is transitive on $\mathcal{G}_{d-r-s,m-r}$, and its unique
invariant probability measure is the uniform measure on this
Grassmannian. Therefore, $Q$ is uniformly distributed on
$\mathcal{G}_{d-r-s,m-r}$.
\end{proof}

\begin{lemma} \label{lem:cond-indep}
Let $(A,B,C,M)$ be jointly distributed random variables such that  $A\indep B \big| C,M$ and  $A \indep M \big| C$. Then, we have $A \indep B \big| C$.
\end{lemma}
\begin{proof}
In this proof, we use elementary information-theoretic equalities from \citep{cover2012elements}. Here, $I$ denotes the mutual information.   The assumptions in the statement are equivalent to
\[
I(A;B\mid C,M)=0
\qquad\text{and}\qquad
I(A;M\mid C)=0.
\]
Therefore, by the chain rule for conditional mutual information and $0
=
I(A;B,M\mid C)$, we have
\[
0=
I(A;M\mid C)+I(A;B\mid C,M).
\]
Applying the chain rule in the opposite order gives 
\[
0=
I(A;B,M\mid C)
=
I(A;B\mid C)+I(A;M\mid B,C).
\]
Since conditional mutual information is nonnegative, it follows that
$I(A;B\mid C)=0$, which is equivalent to $A\indep B\mid C$.
\end{proof}

\begin{lemma}\label{lem:chi-square}
Let $\Sigma_1,\Sigma_2$ be positive definite symmetric matrices such that $\Sigma_1^{-1} + \Sigma_2^{-1} -\id{d}$ be a positive definite symmetric matrix. Let $p_0(x),p_1(x),p_2(x)$ be the probability density functions for $\Normal(0,\id{d})$, $\Normal(0,\Sigma_1)$ and $\Normal(0,\Sigma_2)$. Then,
$$
\int_{x \in \Reals^d} \frac{p_{1}(x)p_{2}(x)}{p_0(x)} \text{d}x = (\det(\Sigma_1) \det(\Sigma_2))^{-1/2} (\det(\Sigma_1^{-1}+\Sigma_2^{-1}-\id{d}))^{-1/2}
$$
\end{lemma}
\begin{proof}
We  have 
\begin{align*}
 &\int \frac{p_1(x)p_2(x)}{p_0(x)} \text{d}x \\
 &= \int \frac{  (2\pi)^{-d/2} (\det(\Sigma_1))^{-1/2} \exp(-x^T \Sigma_1^{-1} x/2)(2\pi)^{-d/2} (\det(\Sigma_2))^{-1/2} \exp(-x^T \Sigma_2^{-1} x/2)}{(2\pi)^{-d/2} \exp(-x^Tx/2)}\text{d}x\\
& = (\det(\Sigma_1\Sigma_2))^{-1/2} \int (2\pi)^{-d/2} \exp(-x^T (\Sigma_1^{-1}+\Sigma_2^{-1}-\id{d}) x/2)\text{d}x\\
& = (\det(\Sigma_1\Sigma_2))^{-1/2} (\det(\Sigma_1^{-1}+\Sigma_2^{-1}-\id{d}))^{-1/2},
\end{align*}
where in the last step we used the definition of Gaussian distribution. This completes the proof.
\end{proof}
\begin{lemma}
\label{lem:determ-rank2}
Fix $d \in \Naturals$. Let $v_1$ and $v_2$ be two unit norm vectors in $\Reals^d$ and $\gamma $ be a positive constant. Then,
\[
\det\left(\id{d} +  \gamma v_1v_1^\top +  \gamma v_2v_2^\top\right) = 1+2\gamma + \gamma^2\left(1 - \left(\inner{v_1}{v_2}\right)^2\right).
\]
\end{lemma}

\begin{proof}
Consider $d \times 2$ matrix $V = [v_1, v_2]$. The argument of the determinant is $\id{d} + \gamma V V^\top$. By Sylvester's determinant identity, we can reduce the problem's dimension:
\begin{align*}
    \det\left(\id{d} + \gamma v_1v_1^\top + \gamma v_2v_2^\top\right)
    &= \det\left(\id{d} + \gamma V V^\top \right) \\
    &= \det\left(\id{2} + \gamma V^\top V \right) && \text{(by Sylvester's identity)} \\
    &= \det\left( \begin{pmatrix} 1 & 0 \\ 0 & 1 \end{pmatrix} + \gamma \begin{pmatrix} 1 & \left\langle v_1, v_2 \right\rangle \\ \left\langle v_1, v_2 \right\rangle & 1 \end{pmatrix} \right) && \text{(since $\|v_1\|=\|v_2\|=1$)} \\
    &= \det \begin{pmatrix} 1+\gamma & \gamma\left\langle v_1, v_2 \right\rangle \\ \gamma\left\langle v_1, v_2 \right\rangle & 1+\gamma \end{pmatrix} \\
    &= (1+\gamma)^2 - \gamma^2 \left\langle v_1, v_2 \right\rangle^2 \\
    &= 1 + 2\gamma + \gamma^2\left( 1 - \left\langle v_1, v_2 \right\rangle^2 \right),
\end{align*}
as was to be shown.
\end{proof}

\end{document}